\documentclass[lettersize,journal]{IEEEtran}
\usepackage{amsmath,amsfonts}
\usepackage{algorithmic}
\usepackage{algorithm}
\usepackage{array}
\usepackage{textcomp}
\usepackage{stfloats}
\usepackage{url}
\usepackage{verbatim}
\usepackage{graphicx}
\usepackage{cite}
\usepackage{multirow}
\usepackage{booktabs}       
\usepackage[table]{xcolor}
\usepackage{xcolor}         
\usepackage{tabularx}
\newcolumntype{Y}{>{\centering\arraybackslash}X}
\usepackage{subfigure}
\usepackage{subcaption}
\usepackage{graphicx}
\usepackage{hyperref}

\newtheorem{theorem}{\textbf{Theorem}}
\newtheorem{discussion}{\textbf{Discussion}}

\newtheorem{assumption}{\textbf{Assumption}}
\newtheorem{challenge}{\textbf{Challenge}}
\newtheorem{proof}{\textbf{Proof}}
\newtheorem{proposition}{\textbf{Proposition}}
\newtheorem{corollary}{Corollary}

\begin{document}

\title{FreKoo++: Learning Continuous Spectral Dynamics for Temporal Domain Generalization}

\author{En Yu, Xiaoyu Yang, Wei Duan, Guangquan Zhang, Jie Lu~\IEEEmembership{IEEE Fellow}

\thanks{En Yu, Xiaoyu Yang, Wei Duan, Guangquan Zhang and Jie Lu are with the Australian Artificial Intelligence Institute (AAII), Faculty of Engineering and Information Technology, University of Technology Sydney (UTS), Australia (e-mail: en.yu-1@uts.edu.au;  xiaoyu.yang-3@student.uts.edu.au; wei.duan@uts.edu.au; guangquan.zhang@uts.edu.au; jie.lu@uts.edu.au.)}

}

\markboth{Journal of \LaTeX\ Class Files,~Vol.~14, No.~8, August~2021}%
{Shell \MakeLowercase{\textit{et al.}}: A Sample Article Using IEEEtran.cls for IEEE Journals}


\maketitle

\begin{abstract}
Temporal Domain Generalization (TDG) aims to learn from historical domains and generalize to unseen future distributions under concept drift. Nevertheless, prevailing TDG methods struggle with complex real-world streaming scenarios involving both multi-scale drift patterns (e.g., long-term periodicity intertwined with short-term incremental changes) and local uncertainties, especially in continuous settings where observations arrive irregularly. To address this limitation, we propose FreKoo++, a novel continuous spectral-dynamical framework that pioneers the unification of continuous Koopman modal dynamics with adaptive spectral disentanglement. Specifically, FreKoo++ maps source-domain parameters into a compact latent space, modeling their evolution as a superposition of learnable continuous modes where complex eigenvalues jointly encode oscillatory frequency and temporal growth or decay. This formulation naturally accommodates irregular timestamps and supports arbitrary horizon extrapolation without rigid discrete stepping. Furthermore, we propose a new adaptive soft spectral weighting mechanism backed by stability and spectral regularization, which automatically isolates persistent dominant dynamics from transient noise without relying on manual frequency thresholds. We derive modal approximation and generalization bounds that characterize how amplitude and eigenvalue estimation errors propagate with the prediction horizon. Extensive experiments on both discrete and continuous TDG benchmarks demonstrate that FreKoo++ achieves state-of-the-art performance under complex multi-scale drifts and irregular sampling.
\end{abstract}


\section{Introduction}
\label{sec:introduction}
\IEEEPARstart{M}{odern} machine learning models face significant challenges in dynamic environments where data distributions evolve over time~\cite{han2024model,11359344, yu2026drift}. Unlike static settings that assume an IID relationship between training and test data, real-world scenarios often involve continuously generated data streams (e.g., user activity logs, sensor readings, financial transactions) exhibiting temporal dependencies and non-stationarity due to factors like shifting user behavior, environmental variations, or system changes~\cite{wang2023koopman,yang2026turning,giaretta2025supervised}. These temporal dynamics lead to distributional shifts over time, known as \emph{concept drift}, which invalidates the alignment between historical and future out-of-distribution (OOD) data~\cite{lu2018learning,xu2025drift2matrix,wang2024distributionally}. Consequently, models trained on past data frequently fail to generalize to future instances, leading to performance degradation and reduced reliability. This has spurred increasing interest in \emph{Temporal Domain Generalization (TDG)}, which seeks to learn from chronologically ordered source domains and generalize to unseen future distributions under evolving temporal shifts~\cite{nasery2021training}.

\begin{figure}[tbp]
    \centering
    \includegraphics[width=\linewidth]{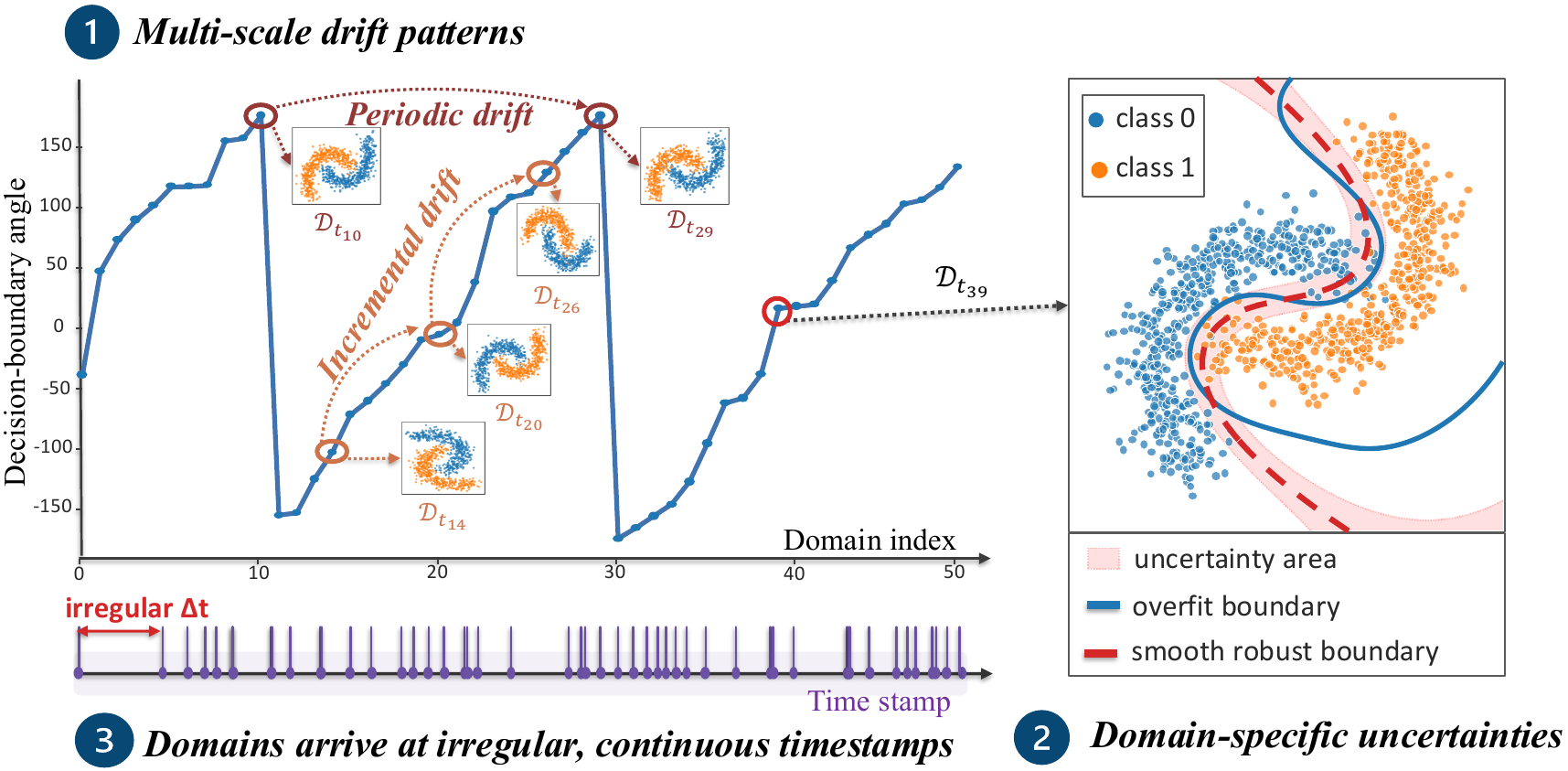} 
    \caption{Illustration of the three TDG challenges on a 2-Moons stream. The underlying distribution exhibits multi-scale periodicity, while empirical parameter estimates are corrupted by domain-specific uncertainties. Additionally, the irregularly sampled observations violate the time-invariant discrete transition assumed by standard sequence models.}
    \label{figure:challenges}
\end{figure}

Recent progress in TDG generally follows two temporal modeling paradigms: \textit{discrete} and \textit{continuous} approaches. \textit{Discrete TDG (DTDG)} methods operate under the rigid assumption of uniformly spaced temporal grids, enforcing step-wise transitions or generating synthetic OOD samples to model temporal evolution~\cite{nasery2021training,zeng2023foresee,bai2023temporal,chang2023coda}. 
However, in real-world applications such as social media streams and clinical healthcare, observations unfold asynchronously and unpredictably~\cite{yang2025walking}. This causes temporal domains to be irregularly distributed over continuous time~\cite{cai2024continuous}. Binding dynamics to fixed discrete steps prevents discrete methods from characterizing the continuous dynamics of irregular time-distributed domains. 
To solve this, \textit{continuous TDG (CTDG)} approaches have recently emerged, leveraging continuous domain indices or neural ODEs to enable prediction at arbitrary timestamps~\cite{cai2024continuous,ortiz2019cdot,qin2023evolving}. 

Nevertheless, most methods implicitly assume that concept drift is primarily incremental or focus on ensuring local smoothness between adjacent domains, which causes two challenges as visualized in Fig.~\ref{figure:challenges}: \emph{1) Difficulty in multi-scale drift patterns}: real-world concept drift frequently exhibits long-term periodicity (e.g., seasonality, weekly/daily user patterns, economic cycles), not just smoothly incremental changes~\cite{yu2024online,jiao2022dynamic}. Current approaches struggle to capture these recurring patterns spanning longer time horizons~\cite{bai2023temporal}. Their effectiveness often diminishes, particularly near phase transitions of periodic cycles where the inability to anticipate pattern recurrence hinders generalization. \emph{2) Vulnerability to Overfitting Domain-Specific Uncertainties}: current TDG methods typically segment continuous data streams into discrete temporal domains. However, the complex drift patterns inherent in real-world streams make it difficult to guarantee that data within each domain remains IID. Consequently, localized uncertainties and non-IID characteristics within domains can complicate optimization and increase the risk of overfitting to domain-specific artifacts, ultimately undermining stable cross-temporal generalization~\cite{niu2023towards}.

These challenges underscore the necessity for principled frameworks that can capture persistent multi-scale dynamics while filtering local uncertainties. Since model parameters optimized on successive temporal domains provide a compact proxy for distributional evolution, we model parameter trajectories from a spectral perspective~\cite{yu2025learning,11300955,cai2024continuous}. In the spectral domain, long-term trends and periodic components are typically concentrated in low-frequency structures, whereas transient domain-specific uncertainties often appear as high-frequency variations~\cite{ye2024frequency,wang2025long,liu2023koopa,liu2023adaptive}. This separation offers a useful inductive bias for isolating predictable temporal dynamics from disruptive local fluctuations. Building upon this intuition, our preliminary work \textbf{FreKoo}~\cite{yu2025learning} integrated Discrete Fourier Transform (DFT) with Koopman operator theory for discrete TDG. It disentangled parameter trajectories into distinct spectral bands and leveraged Koopman dynamics to extrapolate low-frequency dominant patterns while regularizing high-frequency uncertainties. However, FreKoo is strictly bound to uniform temporal grids and relies on manually pre-defined, non-adaptive frequency thresholds. This rigid design prevents dynamic spectral disentanglement and renders it fundamentally incompatible with continuous-time data streams where observations arrive asynchronously.

In this paper, we propose \textbf{FreKoo++}, a unified framework for continuous temporal domain generalization. FreKoo++ is motivated by the observation that the solution of a continuous-time linear dynamical system admits an exponential form that can be interpreted as a spectral decomposition in the Laplace domain. Specifically, each eigenvalue of the Koopman operator encodes both the growth/decay rate and oscillatory frequency of a latent dynamical mode. This connection enables us to directly parameterize temporal evolution through learnable Koopman eigenvalues, bypassing the need for discrete Fourier transforms and naturally supporting irregularly sampled observations. Concretely, FreKoo++ is implemented as a continuous-time parameterized module that learns Koopman eigenvalues from irregularly sampled parameter trajectories. This formulation supports prediction at arbitrary future timestamps without relying on discrete time stepping. Furthermore, leveraging the continuous spectral representation, we introduce a learnable soft spectral weighting mechanism together with a spectral energy regularization strategy, which adaptively emphasizes dominant low-frequency structures while suppressing high-frequency noise. This design removes the need for manually defined frequency thresholds and allows end-to-end optimization. To provide theoretical grounding, we analyze the approximation behavior of the proposed continuous-time formulation, bounding its modal approximation error and deriving a generalization bound whose growth is controlled by the learned spectrum.

Our main contributions are summarized as follows:
\begin{itemize}
    \item We propose \textbf{FreKoo++}, a pioneering continuous-time TDG framework that shifts the paradigm of parameter evolution from discrete-step transitions to learnable Koopman-inspired modal dynamics. By encoding frequency and temporal growth/decay into continuous eigenvalues, FreKoo++ naturally accommodates irregularly sampled source domains and arbitrary-horizon prediction.
    
    \item We introduce an adaptive spectral disentanglement mechanism with spectral and stability regularization, which automatically disentangles persistent dynamics from transient fluctuations without requiring manual frequency selection.

    \item We provide comprehensive theoretical and empirical validation for FreKoo++. Our theoretical analysis establishes rigorous approximation and generalization guarantees, while extensive experiments across diverse TDG benchmarks demonstrate its superior performance and robust extrapolation capabilities.
\end{itemize}

\section{Related Work}
\label{sec:relatedworks}
\subsection{Temporal Domain Generalization}

\subsubsection{Discrete Temporal Domain Generalization (DTDG)}
DTDG focuses on learning from chronologically ordered source domains to generalize to unseen future distributions~\cite{nasery2021training}. The core objective is to effectively leverage temporal evolutionary cues without accessing target domain data. Existing approaches fall broadly into \emph{data driven} and \emph{model centric} paradigms. \emph{Data driven} methods enhance temporal robustness by synthesizing temporally shifted distributions or enforcing invariant features. For instance, CODA constructs concept drift simulators to generate future-like samples~\cite{chang2023coda}, Foresee synthesizes out-of-distribution instances for non-stationary environments~\cite{zeng2023foresee}, and recent methods model instance-level evolution or learn temporally invariant representations to improve robustness to future shifts~\cite{jin2024temporal,zeng2024generalizing,liang2026deepboots}. 
\emph{Model centric} approaches instead integrate temporal dynamics directly into optimization objectives or model parameters. Gradient interpolation regularization encourages decision boundaries to evolve smoothly over time~\cite{nasery2021training}, DRAIN forecasts future neural network parameters via drift-aware recurrent architectures~\cite{bai2023temporal}, and sequential autoencoders capture evolving latent domain structures~\cite{qin2022generalizing,xu2025drift2matrix}. Despite these advances, most DTDG methods implicitly rely on uniformly spaced temporal grids and assume smooth incremental shifts. 

\subsubsection{Continuous Temporal Domain Generalization (CTDG)}
Continuous temporal learning relaxes the restrictive assumption of rigid, uniformly spaced temporal grids. Early continuous domain adaptation frameworks treat time as a continuous domain index to learn smoothly transferable representations across evolving distributions~\cite{ortiz2019cdot,wang2020continuously,he2025learning}. More recently, Continuous Temporal Domain Generalization (CTDG) has emerged to address future-domain extrapolation under continuous, irregular observation timestamps~\cite{cai2024continuous,cai2025continuous}. By supporting asynchronous observation arrivals and arbitrary-horizon forecasting, CTDG represents a critical paradigm shift toward realistic streaming environments.

However, existing CTDG methods still exhibit two fundamental limitations. First, prevailing formulations rely heavily on monotonic or locally smooth evolution assumptions, rendering them inadequate for complex, multi-scale, or recurring concept drifts. Second, methods driven by standard continuous-time dynamics—such as Neural ODEs or continuous-time recurrent units—typically model parameter or feature transitions directly in the time domain~\cite{yang2025walking,luo2025theoretical,hoover2026koopman}. As a result, they lack an explicit mechanism to disentangle persistent macro-trends from transient domain-specific fluctuations. This research overcomes these bottlenecks by establishing a continuous Koopman spectral representation, where complex eigenvalues explicitly quantify frequency and temporal persistence, enabling adaptive spectral disentanglement under irregular temporal sampling.

\subsection{Concept Drift}
Concept drift signifies a change in the underlying data distribution over time, formally occurring between time $t$ and $t+1$ if the joint distribution $P_{t+1}(X, y) \neq P_t(X, y)$. This non-stationarity poses a fundamental challenge, requiring models to adapt dynamically to maintain predictive performance and reliability~\cite{lu2018learning,li2024distribution,yu2024online,lu2025early}. Prior research treats it as unpredictable and often employing detect-then-adapt strategies~\cite{yang2025adapting,yu2026drift}. In contrast, the Temporal Domain Generalization (TDG) setting frequently focuses on leveraging more predictable evolutionary dynamics for proactive generalization~\cite{bai2023temporal,yu2026generalized}. However, existing TDG methods often oversimplify these dynamics, primarily modeling smooth, incremental drifts while struggling to capture complex long-range temporal structures like periodicity which are prevalent in many real-world data streams yet demand different approaches than simple monotonic progression.

Furthermore, current TDG involves segmenting the continuous data stream into temporal domains $\{D_1, D_2, \dots, D_T\}$, assuming concept drift occurs primarily between these sequential domains~\cite{nasery2021training}. This often struggles to perfectly align with real-world data streams, potentially violating the implicit assumption that data within each domain $D_t$ is Independent and Identically Distributed (IID). Consequently, individual domains may harbor internal non-IID structures or localized noise patterns~\cite{peng2025distributional,wu2025scot}. The presence of intra-domain complexities leads to sensitivity to domain-specific noise and artifacts, hindering robust generalization, particularly for TDG models designed predominantly to address shifts occurring only at the domain boundaries. Effectively handling both the diverse inter-domain evolutionary dynamics (including periodicity) and these intra-domain data characteristics is therefore crucial for robust temporal generalization.

\subsection{Frequency Learning and Koopman Dynamics}
Frequency-domain learning has become an effective tool for modeling non-stationary temporal data. By transforming signals from the time domain to the spectral domain, frequency-aware methods can capture periodicity, seasonality, and multi-scale temporal dependencies that may be difficult to identify through local time-domain statistics. For example, frequency-enhanced architectures have been used in long-term time-series forecasting~\cite{zhou2022fedformer,yi2024filternet}, adaptive time-frequency modeling~\cite{ye2024atfnet}, and frequency normalization for non-stationary forecasting~\cite{ye2024frequency}. Classical frequency analysis has also been widely used to detect periodic structures in temporal signals~\cite{cutler2000robust,hurley2022radio,wang2021enhanced}. These studies suggest that spectral representations provide a useful inductive bias for separating stable temporal patterns from high-frequency disturbances.

Koopman theory provides another principled view of temporal dynamics by representing nonlinear evolution through a linear operator in an appropriate lifted space~\cite{koopman1931hamiltonian,brunton2021modern}. Recent machine learning methods have combined Koopman operators with neural representations for forecasting dynamical systems and non-stationary time series~\cite{wang2023koopman,liu2023koopa}. In temporal domain generalization, Koopman-based methods have also been explored to model evolving domains~\cite{zeng2024generalizing}. Our preliminary FreKoo framework first applies DFT-based
decomposition to uniformly sampled parameter trajectories and then extrapolates selected components through discrete Koopman dynamics~\cite{yu2025learning}. FreKoo++ instead directly learns a continuous modal representation, enabling spectral allocation and arbitrary-time extrapolation within a unified latent space.

\section{Preliminary}
\label{sec:preliminary}

\subsection{Temporal Domain Generalization}
Given a chronologically ordered sequence of temporal domains extracted from historical data streams, Temporal Domain Generalization (TDG) aims to train a model that generalizes to unseen future data distributions. Let $\mathcal{T} = \{t_1, t_2, \dots, t_T\}$ denote the set of historical timestamps, where $t_i \in \mathbb{R}^+$ and $t_1 < t_2 < \dots < t_T$. At each timestamp $t_i$, we observe a source domain $\mathcal{D}_{t_i} = \{(\mathbf{x}_{t_i}^{(j)}, y_{t_i}^{(j)})\}_{j=1}^{N_i}$, drawn from a time-dependent joint distribution $P_{t_i}(X,Y)$ over a common input space $\mathcal{X}$ and label space $\mathcal{Y}$. Due to \textit{concept drift}, $P_{t_i}(X,Y)$ may differ from $P_{t_j}(X,Y)$ for $i \neq j$. Depending on the temporal structure of the observed domains, TDG can be categorized into two paradigms:

\noindent
\textbf{Discrete TDG (DTDG)} assumes uniformly spaced timestamps, i.e., $\Delta t_i = t_{i+1} - t_i \equiv C$. Given historical source domains $\{\mathcal{D}_{t_i}\}_{i=1}^{T}$, the objective is to learn a model $g(\cdot;\theta): \mathcal{X} \rightarrow \mathcal{Y}$ that generalizes to the next unseen target domain $\mathcal{D}_{t_{T+1}} \sim P_{t_{T+1}}(X,Y)$ without accessing $\mathcal{D}_{t_{T+1}}$ during training~\cite{nasery2021training}.

\noindent
\textbf{Continuous TDG (CTDG)} relaxes the uniform-spacing assumption and allows $\Delta t_i$ to vary across domains. The objective is to learn a temporal prediction mechanism that generates model parameters $\theta(t_s)$ for an arbitrary future timestamp $t_s > t_T$, such that $g(\cdot;\theta(t_s))$ generalizes to the unseen target domain $\mathcal{D}_{t_s} \sim P_{t_s}(X,Y)$ without access to future-domain data during training~\cite{cai2024continuous}.

\subsection{Challenges} 
To characterize temporal distribution shifts from the perspective of model evolution, we consider the trajectory of optimal parameters induced by the time-varying distribution $P_t(X,Y)$. Specifically, let $\theta^*(t) = \arg\min_\theta \mathbb{E}_{P_t} [\ell(g(X; \theta), Y)]$ denote the optimal parameter configuration at time $t$. In practice, this underlying trajectory cannot be directly observed. Instead, at each historical timestamp $t_i \in \mathcal{T}$, we obtain an empirical estimate $\bar{\theta}_{t_i}$ by optimizing the model on the finite temporal domain $\mathcal{D}_{t_i}$. Due to finite samples, intra-domain non-stationarity, and domain-specific variations, these empirical estimates inevitably deviate from the underlying trajectory. We model this observation process as, 
\begin{equation}
    \bar{\theta}_{t_i} = \theta^*(t_i) + \epsilon(t_i),
    \label{eq:observation_model}
\end{equation}
where $\epsilon(t_i)$ summarizes estimation error and local domain-specific fluctuations. Based on this parameter-space formulation, continuous temporal domain generalization confronts three fundamental challenges:

\begin{challenge}[Multi-scale drift patterns in parameter space]
\label{challenge:1}
Prevailing TDG methods often model temporal evolution through incremental changes or local smoothness between adjacent parameter configurations. While such formulations handle gradual shifts, they fail to capture long-range recurring dynamics. In many real-world streams, the underlying parameter trajectory $\theta^*(t)$ exhibits periodic or multi-scale patterns, where similar states reappear after a temporal lag—i.e., $\exists L > 0$ such that $\theta^*(t) \approx \theta^*(t+kL)$ for $k \in \mathbb{Z}^{+}$, despite short-term variations where $\theta^*(t) \neq \theta^*(t+\Delta)$. This non-monotonic trajectory cannot be adequately captured by step-wise local smoothness alone.
\end{challenge}

\begin{challenge}[Parameter-space uncertainties and noise corruption]
\label{challenge:2}
The empirical parameter sequence $\{\bar{\theta}_{t_i}\}_{i=1}^T$ is corrupted by the perturbation term $\epsilon(t_i)$ in Eq.~\eqref{eq:observation_model}, which encapsulates finite-sample effects and domain-specific variations. Directly treating this noisy sequence as a clean dynamical signal causes models to overfit these transient fluctuations rather than tracking the true, persistent parameter trajectory $\theta^*(t)$, thereby severely degrading future-domain generalization.
\end{challenge}

\begin{challenge}[Irregular timestamps and arbitrary-time parameter extrapolation]
\label{challenge:3}
CTDG requires predicting the parameter configuration $\theta(t_s)$ at an arbitrary future timestamp $t_s > t_T$, while historical domains arrive at non-uniform intervals $\Delta t_i = t_{i+1} - t_i$. Existing discrete methods rely on time-invariant one-step parameter transitions (e.g., $\bar{\theta}_{t_{i+1}} = F(\bar{\theta}_{t_i})$) designed for uniform grids ($\Delta t_i \equiv C$). Under irregular sampling, such discrete transitions fail to scale with varying temporal gaps, leading to compounding error accumulation when recursively projected to arbitrary future times.
\end{challenge}

\begin{figure}[t]
    \centering
    \includegraphics[width=\linewidth]{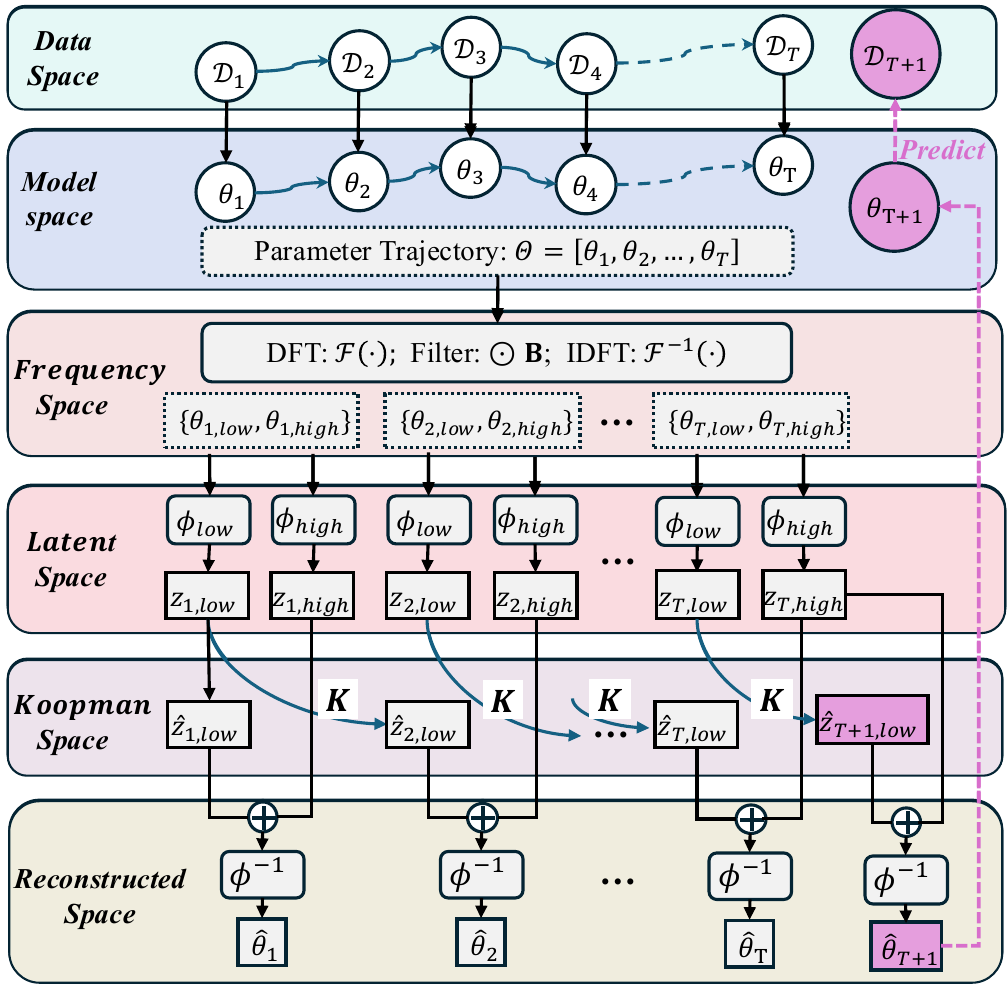}
    \caption{FreKoo framework. It decomposes model parameter trajectories into low- and high-frequency components via the Fourier transform. Koopman dynamics extrapolate the low-frequency component, while temporal-difference regularization suppresses unstable high-frequency variations.}
    \label{fig:framework}
\end{figure}
 \section{FreKoo: Frequency-Koopman Framework}
\label{sec:frekoo}
For the DTDG setting, our preliminary framework \textbf{FreKoo}~\cite{yu2025learning} models parameter trajectories through a two-stage spectral-dynamical approach (as shown in Fig.~\ref{fig:framework}), i.e., Fourier-based spectral decomposition and frequency-specific dynamic modeling. Specifically, given a time-varying parameter trajectory $\Theta = [\theta_1, \theta_2, \dots, \theta_T]^\top \in \mathbb{R}^{T \times D}$ with $T$ uniform timesteps and $D$ dimensions, FreKoo first transforms it into frequency space via Discrete Fourier Transform (DFT) along the temporal axis: $\Theta^{f} = \mathcal{F}(\Theta)$, where $\Theta^{f}[f,d] = \sum_{t=1}^{T} \Theta[t,d] \cdot e^{-j2\pi ft/T}$. To isolate dominant dynamics, it computes the average spectral magnitude across dimensions as an energy proxy, $M_f = \frac{1}{D} \sum_{d=1}^D |\Theta^{f}[f, d]|$. Based on a specified energy preservation ratio $\tau \in [0,1]$, a binary mask $\mathbf{B} \in \{0,1\}^{N_{\mathrm{freq}} \times D}$ selects the top-$Q$ frequency indices with the largest energy magnitudes. Reversing the transformation via Inverse DFT ($\mathcal{F}^{-1}$) yields a trajectory decomposition:
\begin{equation}
    \Theta_{\mathrm{low}} = \mathcal{F}^{-1}(\Theta^{f} \odot \mathbf{B}), 
    \Theta_{\mathrm{high}} = \mathcal{F}^{-1}(\Theta^{f} \odot (\mathbf{1} - \mathbf{B})),
    \label{eq:IDFT}
\end{equation}
where $\Theta = \Theta_{\mathrm{low}} + \Theta_{\mathrm{high}}$. Here, $\Theta_{\mathrm{low}}$ encapsulates dominant trends and long-term periodic patterns (\emph{Challenge 1}), whereas $\Theta_{\mathrm{high}}$ captures transient domain-specific uncertainties (\emph{Challenge 2}).
To model the persistent low-frequency dynamics $\Theta_{\mathrm{low}} = \{\theta_{1,\mathrm{low}}, \dots, \theta_{T,\mathrm{low}}\}$, FreKoo maps them into a higher-dimensional Hilbert space via an encoder $\phi_{\mathrm{low}}: \mathbb{R}^D \to \mathbb{R}^m$, assuming near-linear evolution governed by a learnable Koopman operator $K \in \mathbb{R}^{m \times m}$. The low-frequency transition is expressed as $z_{t,\mathrm{low}} = \phi_{\mathrm{low}}(\theta_{t,\mathrm{low}})$ and $\hat{z}_{t+1,\mathrm{low}} = K z_{t,\mathrm{low}}$, supervised by a Koopman consistency loss:
\begin{equation}
    \mathcal{L}_{\mathrm{koop}} = \sum_{t=1}^{T-1} \|z_{t+1,\mathrm{low}} - \hat{z}_{t+1,\mathrm{low}}\|_2^2.
    \label{eq:loss_koop}
\end{equation}

Conversely, for the volatile high-frequency sequence $\Theta_{\mathrm{high}} = \{\theta_{1,\mathrm{high}}, \dots, \theta_{T,\mathrm{high}}\}$, forward extrapolation is avoided to prevent noise amplification. Instead, FreKoo maps it via an encoder $\phi_{\mathrm{high}}$ and imposes a temporal difference smoothness regularization:
\begin{equation}
    \mathcal{R}_{\mathrm{high}} = \sum_{t=1}^{T-1} \|z_{t+1,\mathrm{high}} - z_{t,\mathrm{high}}\|_2^2.
    \label{eq:loss_highfre}
\end{equation}

To predict parameters for the next timestamp $\hat{\theta}_{t+1}$, FreKoo fuses the extrapolated low-frequency latent with the regularized high-frequency context through a shared decoder $\phi^{-1}$:
\begin{equation}
    \hat{\theta}_{t+1} = \phi^{-1}\big(K\phi_{\mathrm{low}}(\theta_{t,\mathrm{low}}) + \phi_{\mathrm{high}}(\theta_{t,\mathrm{high}})\big).
    \label{eq:param_prediction}
\end{equation}
The overall framework is trained end-to-end under the joint objective:
\begin{equation}
    \mathcal{L}_{\mathrm{frekoo}} = \mathcal{L}_{\mathrm{task}} + \alpha \mathcal{L}_{\mathrm{rec}} + \beta \mathcal{L}_{\mathrm{koop}} + \gamma \mathcal{R}_{\mathrm{high}},
    \label{eq:loss_total}
\end{equation}
where $\mathcal{L}_{\mathrm{task}}$ is the empirical source risk, and $\mathcal{L}_{\mathrm{rec}} = \sum_{t=1}^{T-1} \|\theta_{t+1} - \hat{\theta}_{t+1}\|_2^2$ enforces overall parameter reconstruction fidelity.

\begin{discussion}[Limitations of FreKoo]
\label{discussion:limitation-frekoo}
While FreKoo effectively addresses Challenges 1 and 2 in DTDG, it exhibits two inherent limitations: 1) Its reliance on DFT requires uniformly spaced domains ($\Delta t_i \equiv C$), making it fundamentally incompatible with continuous streams where observations arrive asynchronously (\emph{Challenge 3}); and 2) It separates spectral decomposition from dynamical modeling, requiring distinct encoders ($\phi_{\mathrm{low}}, \phi_{\mathrm{high}}$) and a manually pre-defined threshold $\tau$, which prevents dynamic, end-to-end spectral disentanglement. These constraints necessitate a unified continuous time framework with adaptive spectral selection. 
\end{discussion}

\begin{figure*}[tbp]
    \centering
    \includegraphics[width=\textwidth]{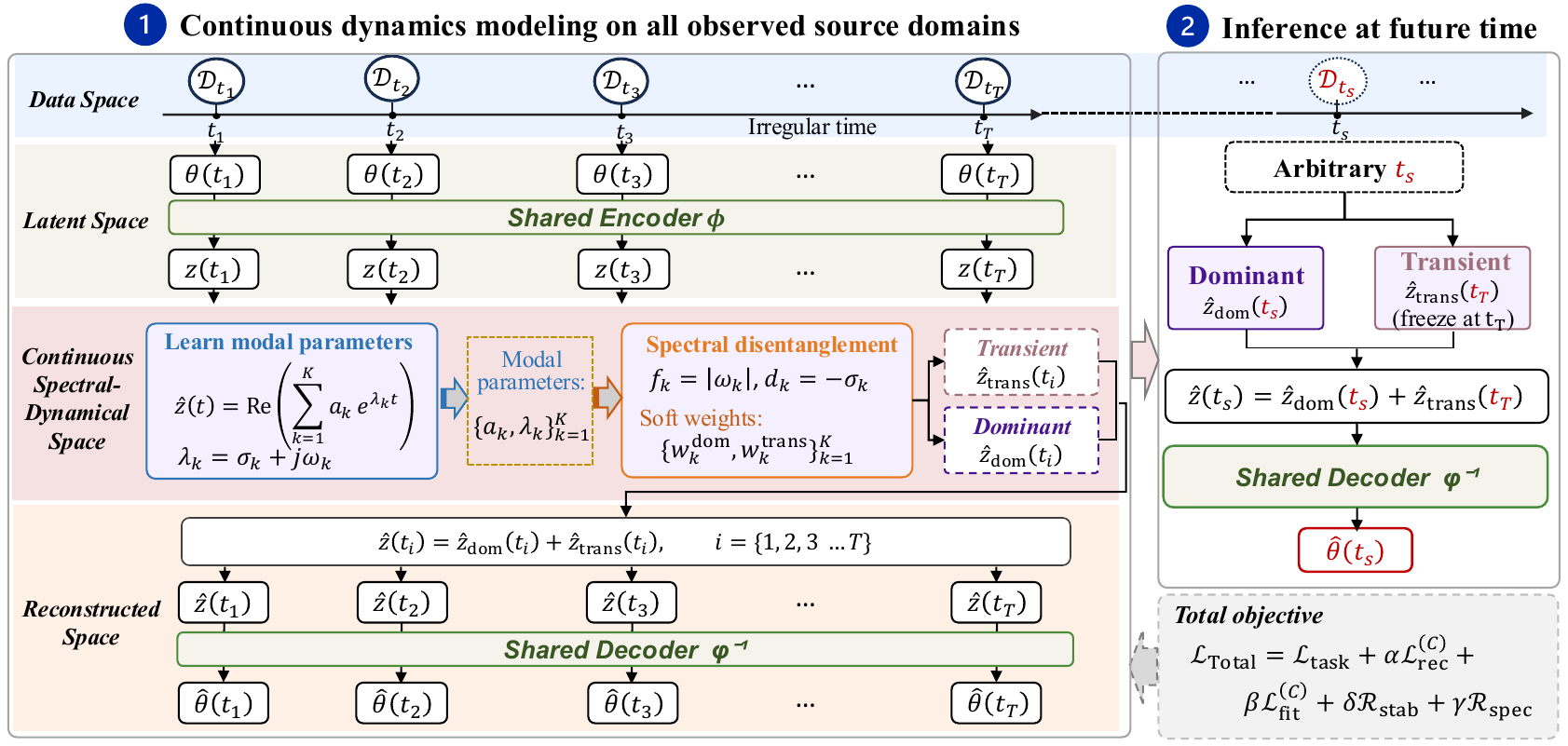}
    \caption{FreKoo++ framework. \textbf{(1) Training Phase:} Model parameters from irregular timestamps are encoded into a continuous latent trajectory and fitted by Koopman modal dynamics. Adaptive soft weighting disentangles persistent dominant dynamics from transient noise, yielding an optimized continuous spectral representation via joint optimization. \textbf{(2) Inference Phase:} For an arbitrary target timestamp $t_s > t_T$, FreKoo++ analytically extrapolates the persistent dominant component to $t_s$, freezes the transient residual at $t_T$, and decodes their superposition to obtain target parameters $\hat{\theta}(t_s)$.}
    \label{fig:FreKoo++_framework}
\end{figure*}

\section{FreKoo++: Continuous Frequency-Koopman Framework}
\label{sec:cfrekoo}
To overcome the limitations of discrete pipelines outlined in Discussion~\ref{discussion:limitation-frekoo}, we introduce FreKoo++, which establishes a unified continuous-time formulation that seamlessly bridges spectral analysis and dynamic operator modeling. As shown in Fig.~\ref{fig:FreKoo++_framework}, FreKoo++ maps temporal parameter trajectories into a compact latent space and models their evolution via learnable continuous modal dynamics. Driven by these modal characteristics, an adaptive soft spectral weighting mechanism disentangles persistent dominant dynamics from domain-specific fluctuations. By operating directly in continuous time, FreKoo++ naturally handles irregular temporal observations and enables analytical prediction at arbitrary future horizons without requiring step-wise numerical integration.

\subsection{Latent Trajectory Construction}
\label{subsec:latent_trajectory}
For each observed timestamp $t_i$, let $\theta(t_i)\in\mathbb{R}^{D}$ denote the model parameters associated with the source domain $\mathcal{D}_{t_i}$. We collect these parameters across all timestamps to form a parameter trajectory
$\Theta = [\theta(t_1),\theta(t_2),\ldots,\theta(t_T)]^\top \in \mathbb{R}^{T\times D}$,
which forms a time-stamped parameter trajectory and provides a task-dependent representation of all evolving source domains. Directly modeling temporal dynamics in the high-dimensional parameter space $\mathbb{R}^D$ is computationally expensive and may retain domain-specific parameter variations that are irrelevant to future prediction. FreKoo++ addresses this by projecting the parameter trajectory into a compact latent dynamical space $\mathbb{R}^m$ ($m \ll D$) using a shared encoder $\phi: \mathbb{R}^{D} \to \mathbb{R}^{m}$:
\begin{equation}
\label{eq:encode}
    z(t_i) = \phi\!\left( \theta(t_i) \right), \qquad i = 1,\ldots,T ,
\end{equation}
where $z(t_i) \in \mathbb{R}^{m}$ represents the latent dynamical state at time $t_i$. The sequence $\{z(t_i)\}_{i=1}^{T}$ defines the observed continuous latent trajectory across all source domains.

\subsection{Continuous Spectral-Dynamical Modeling}
\label{subsubsec:eigenvalue}
Given the observed latent trajectory $\{z(t_i)\}_{i=1}^{T}$, our primary goal is to establish a continuous-time parameterization that explicitly exposes its underlying spectral and dynamical properties. We draw theoretical motivation from continuous-time Koopman spectral theory~\cite{koopman1931hamiltonian,brunton2021modern}. When observables span a finite-dimensional Koopman-invariant subspace, the infinitesimal generator of the continuous-time Koopman operator admits exponential modal decompositions whose eigenvalues govern both frequency and temporal growth or decay.
\begin{proposition}[Koopman Spectral-Dynamical Relation]
\label{prop:equivalence}
Let $z(t) \in \mathbb{R}^{m}$ denote the continuous latent state sampled at historical timestamps $\{t_i\}_{i=1}^{T}$ according to Eq.~\eqref{eq:encode}. Suppose the continuous-time evolution of $z(t)$ is governed by a finite $K$-dimensional Koopman-invariant subspace. Let $\mathcal{K}_c \in \mathbb{C}^{K \times K}$ be the matrix representation of the infinitesimal generator of the continuous Koopman operator restricted to this subspace. If $\mathcal{K}_c$ is diagonalizable, then the continuous latent trajectory admits the exact modal expansion:
\begin{equation}
\label{eq:kmd}
    z(t) = \operatorname{Re}\!\left(
        \sum_{k=1}^{K} c_k\, v_k\, e^{\lambda_k t}
    \right),
\end{equation}
where $\{(\lambda_k, v_k)\}_{k=1}^{K}$ denote the eigenvalue-mode pairs with eigenvalues $\lambda_k \in \mathbb{C}$ and Koopman mode vectors $v_k \in \mathbb{C}^{m}$, and $c_k \in \mathbb{C}$ is the modal coefficient determined by initial conditions. Each eigenvalue $\lambda_k = \sigma_k + j\omega_k$ jointly characterizes the continuous temporal behavior: the imaginary part $\omega_k = \operatorname{Im}(\lambda_k)$ defines the oscillatory frequency, while the real part $\sigma_k = \operatorname{Re}(\lambda_k)$ dictates the temporal growth ($\sigma_k > 0$) or decay ($\sigma_k < 0$) rate.
\end{proposition}
Proposition~\ref{prop:equivalence} provides the theoretical justification for our architectural design. In practice, finding an exact Koopman-invariant subspace via neural encoders is unfeasible. Instead, FreKoo++ directly employs the closed-form modal solution in Eq.~\eqref{eq:kmd} as a \emph{parametric dynamical prior} to approximate the latent trajectory.

To eliminate scalar unidentifiability between $c_k$ and $v_k$, we absorb them into a unified complex modal amplitude vector $a_k = c_k v_k \in \mathbb{C}^m$. Collecting all amplitudes as $A = [a_1, \dots, a_K] \in \mathbb{C}^{m \times K}$, the continuous modal approximation is parameterized as:
\begin{equation}
\label{eq:kmd-amp}
    \tilde{z}(t) = \operatorname{Re}\!\left( \sum_{k=1}^{K} a_k e^{\lambda_k t} \right).
\end{equation}

Evaluating Eq.~\eqref{eq:kmd-amp} at any timestamp $t$ yields a continuous latent state without requiring step-wise transition rollouts. To bypass complex-valued optimization during backpropagation, we implement Eq.~\eqref{eq:kmd-amp} through its real-valued trigonometric equivalent:
\begin{equation}
    \tilde{z}(t) = \sum_{k=1}^{K} e^{\sigma_k t} \left[ \operatorname{Re}(a_k)\cos(\omega_k t) - \operatorname{Im}(a_k)\sin(\omega_k t) \right].
\end{equation}
This reformulation guarantees that $\tilde{z}(t) \in \mathbb{R}^m$ is evaluated natively in real arithmetic. Non-oscillatory secular trends are naturally recovered when $\omega_k = 0$.
Consequently, the learnable spectral parameters $\Lambda = \{\lambda_k\}_{k=1}^K$ and amplitudes $A \in \mathbb{C}^{m \times K}$ construct a continuous spectral coordinate system, replacing rigid DFT frequency bins with data-driven eigenvalue-indexed continuous modes.

\subsection{Adaptive Spectral Disentanglement}
\label{subsubsec:soft}
Based on the spectral interpretation above, FreKoo++ assigns each learned mode a soft dominance weight. The imaginary part $\omega_k$ reflects the oscillatory frequency, while the real part $\sigma_k$ controls the growth or decay behavior. These quantities make it possible to distinguish persistent structures from transient fluctuations, but the separation boundary should adapt to the temporal profile of each dataset.

For each mode $k \in \{1, \dots, K\}$, we define its frequency score as $f_k = |\omega_k|$ and its decay score as $d_k = -\sigma_k$ (for stable modes with $\sigma_k \le 0$). We assign a soft dominance weight $w_k^{\mathrm{dom}} \in (0, 1)$ using two learnable gating functions:
\begin{align}
\label{eq:wdom}
    w_k^{\mathrm{dom}} &= \sigma\!\big( -\kappa_f (f_k - f_0) \big) \cdot \sigma\!\big( -\kappa_d (d_k - d_0) \big), \\
\label{eq:wtrans}
    w_k^{\mathrm{trans}} &= 1 - w_k^{\mathrm{dom}},
\end{align}
where $\sigma(\cdot)$ is the Sigmoid function, $\kappa_f, \kappa_d > 0$ dictate gate sharpness, and $f_0, d_0$ are learnable thresholds optimized jointly with the network. A mode receives a high dominance weight ($w_k^{\mathrm{dom}} \approx 1$) if it is both low-frequency and persistent ($\sigma_k \approx 0$). Conversely, modes exhibiting high frequencies or rapid attenuation are routed to the transient component ($w_k^{\mathrm{trans}}$).
Applying these weights continuously decomposes the latent trajectory at any time $t$:
\begin{align}
\label{eq:zdom}
    \hat{z}_{\mathrm{dom}}(t) &= \operatorname{Re}\!\left( \sum_{k=1}^{K} w_k^{\mathrm{dom}} a_k e^{\lambda_k t} \right), \\
\label{eq:ztrans}
    \hat{z}_{\mathrm{trans}}(t) &= \operatorname{Re}\!\left( \sum_{k=1}^{K} w_k^{\mathrm{trans}} a_k e^{\lambda_k t} \right).
\end{align}

Consequently, the full reconstructed latent state at any historical observation timestamp $t_i$ ($i=1,\dots,T$) is obtained by the superposition of both components:
\begin{equation}
\label{eq:z-hat-full}
    \hat{z}(t_i) = \hat{z}_{\mathrm{dom}}(t_i) + \hat{z}_{\mathrm{trans}}(t_i).
\end{equation}

\subsection{Joint Optimization}
\label{subsubsec:optim_c}
Training FreKoo++ optimizes the source-domain parameters $\{\theta(t_i)\}_{i=1}^{T}$, the encoder $\phi$ and decoder $\phi^{-1}$, the eigenvalues $\Lambda=\{\lambda_k\}_{k=1}^{K}$, the modal amplitudes $A \in \mathbb{C}^{m \times K}$, and the spectral thresholds $f_0, d_0$ under a unified objective. 
In practice, this is implemented via alternating gradient updates between source-domain task fitting and latent modal fitting, while all components remain coupled through a unified objective. The overall loss consists of three predictive fitting terms and two spectral regularization constraints. The fitting losses optimize source-domain task predictors, align the continuous modal trajectory with encoded latent states, and preserve parameter auto-encoding fidelity. The two spectral constraints govern dominant-mode stability and adaptive spectral disentanglement, respectively.

The task loss ensures predictive accuracy across the observed source domains. Denoting the training data from domain $\mathcal{D}_{t_i}$ as $(X_{t_i}, Y_{t_i})$, the task loss is formulated as:
\begin{equation}
\label{eq:loss-task-c}
    \mathcal{L}_{\mathrm{task}} = \sum_{i=1}^{T}
    \ell\!\left( g(X_{t_i}; \theta(t_i)), Y_{t_i} \right).
\end{equation}

To align spectral disentanglement with our continuous inference protocol, we formulate modal trajectory fitting as an asymmetric pseudo-future prediction problem across chronological source domains.
Specifically, for each step $i = 2, \ldots, T$, we treat $t_{i-1}$ as the latest available observation and $t_i$ as a pseudo-future timestamp. Following the asymmetric extrapolation strategy used at test time, the persistent dominant component is propagated forward to $t_i$, whereas the volatile transient component is held at the preceding observation $t_{i-1}$:
\begin{equation}
\hat{z}(t_i \mid t_{i-1})
=
\hat{z}_{\mathrm{dom}}(t_i)
+
\hat{z}_{\mathrm{trans}}(t_{i-1}).
\label{eq:pseudo_future_state}
\end{equation}

The resulting temporal predictive fitting loss is defined as:
\begin{equation}
\mathcal{L}_{\mathrm{fit}}^{(C)}
= \sum_{i=2}^{T} \left\| z(t_i) - \hat{z}(t_i \mid t_{i-1}) \right\|_2^2,
\label{eq:fit_loss}
\end{equation}
where $z(t_i) = \phi(\theta(t_i)) \in \mathbb{R}^m$ is the encoded latent state of domain $\mathcal{D}_{t_i}$. Because the dominant and transient branches are evaluated at distinct timestamps, Eq.~\eqref{eq:fit_loss} depends explicitly on the gating weights $\{w_k^{\mathrm{dom}}, w_k^{\mathrm{trans}}\}_{k=1}^K$. It therefore directly supervises adaptive modal allocation while matching the asymmetric propagation rule executed during inference.

The reconstruction loss enforces auto-encoding consistency between the high-dimensional parameter space $\mathbb{R}^D$ and the compact latent space $\mathbb{R}^m$:
\begin{equation}
\label{eq:loss-rec-c}
    \mathcal{L}_{\mathrm{rec}}^{(C)} = \sum_{i=1}^{T}
    \left\| \theta(t_i) - \phi^{-1}\!\left( \hat{z}(t_i) \right) \right\|_2^2 .
\end{equation}

The first spectral constraint regularizes the dominant dynamics against unstable extrapolation. Dominant modes with positive growth rates ($\sigma_k > 0$) can explode exponentially when projected far beyond the historical observation window. We therefore penalize dominant modes exhibiting positive growth:
\begin{equation}
\label{eq:rstab}
    \mathcal{R}_{\mathrm{stab}} = \sum_{k=1}^{K}
    w_k^{\mathrm{dom}}\, \big( \max(0,\, \sigma_k) \big)^2 ,
\end{equation}
such that stable ($\sigma_k < 0$) or neutrally persistent ($\sigma_k = 0$) dominant modes incur no penalty. Compared with hard skew-symmetric constraints on an explicit generator matrix that force purely imaginary spectra and permit only undamped oscillations~\cite{cai2024continuous}, this soft regularization accommodates natural temporal decay while explicitly suppressing exponential divergence.

The second spectral constraint suppresses noisy transient components while preventing the learned soft gates from collapsing into degenerate uniform distributions. Because transient modes are not extrapolated at inference time, FreKoo++ penalizes their modal energy and encourages decisive spectral separation:
\begin{equation}
\label{eq:rspec}
    \mathcal{R}_{\mathrm{spec}}
    = \sum_{k=1}^{K}
    w_k^{\mathrm{trans}}\, \| a_k \|_2^2
    + \left(1 - \operatorname{Var}_{k=1}^K\!\left(w_k^{\mathrm{dom}}\right)\right).
\end{equation}

\begin{proposition}[Spectral Energy Bound]\label{prop:energy}
Let mode $k \in \{1, \dots, K\}$ have amplitude vector $a_k = c_k v_k \in \mathbb{C}^m$ and decay rate $\sigma_k = \operatorname{Re}(\lambda_k) < 0$. Its temporal energy over $[0, T_{\max}]$ is
\begin{equation}
\label{eq:energy-bound}
    E_k = \| a_k \|_2^2 \cdot \frac{1 - e^{2\sigma_k T_{\max}}}{-2\sigma_k}
    \;\leq\; \frac{\| a_k \|_2^2}{2\,|\sigma_k|} .
\end{equation}
\end{proposition}
Proposition~\ref{prop:energy} confirms that for decaying modes ($\sigma_k < 0$), temporal energy over $[0, T_{\max}]$ is strictly upper-bounded by the squared modal amplitude $\|a_k\|_2^2$ scaled by decay rates. For neutral modes ($\sigma_k = 0$), the energy scales as $T_{\max}\|a_k\|_2^2$, establishing $\|a_k\|_2^2$ as a timestamp-independent surrogate for transient modal energy in Eq.~\eqref{eq:rspec}. Meanwhile, the second term in Eq.~\eqref{eq:rspec} maximizes the variance of $\{w_k^{\mathrm{dom}}\}_{k=1}^K$, preventing gating weights from collapsing into uninformative uniform assignments and ensuring crisp spectral disentanglement. Crucially, $\mathcal{R}_{\mathrm{spec}}$ operates directly in the continuous spectral coordinate space rather than on adjacent time-domain finite differences, making it naturally invariant to non-uniform sampling intervals.

Overall, the full training objective is formulated as:
\begin{equation}
\label{eq:total-c}
    \mathcal{L}_{\text{Total}} = \mathcal{L}_{\mathrm{task}}
    + \alpha\, \mathcal{L}_{\mathrm{rec}}^{(C)}
    + \beta\, \mathcal{L}_{\mathrm{fit}}^{(C)}
    + \delta\, \mathcal{R}_{\mathrm{stab}}
    + \gamma\, \mathcal{R}_{\mathrm{spec}} ,
\end{equation}
where $\alpha, \beta, \gamma, \delta > 0$ are balancing hyperparameters. All parameters are trained end-to-end via gradient descent.

\begin{algorithm}[t]
\caption{FreKoo++ Learning and Inference}
\label{alg:cfrekoo}
\begin{algorithmic}[1]
\REQUIRE Source domains $\{\mathcal{D}_{t_1}, \ldots, \mathcal{D}_{t_T}\}$ with timestamps $\mathcal{T}$; hyperparameters: $\alpha, \beta, \gamma, \delta$; number of modes $K$; epochs $N_{\mathrm{ep}}$.
\ENSURE Predicted parameters $\hat{\theta}(t_s)$ for target time $t_s$.
\STATE Initialize base model $g(\cdot)$, encoder $\phi$, decoder $\phi^{-1}$, eigenvalues $\{\lambda_k\}_{k=1}^{K}$, amplitudes $A$, thresholds $f_0, d_0$.
\FOR{epoch $= 1$ to $N_{\mathrm{ep}}$}
    \FOR{$i = 1$ to $T$}
        \STATE Compute task loss for $g(\cdot\,;\theta(t_i))$ on $\mathcal{D}_{t_i}$.
        \STATE $z(t_i) \leftarrow \phi(\theta(t_i))$ \COMMENT{encode parameters, Eq.~\eqref{eq:encode}}
    \ENDFOR
    \STATE Compute weights $w_k^{\mathrm{dom}}, w_k^{\mathrm{trans}}$ via Eqs.~\eqref{eq:wdom}--\eqref{eq:wtrans}.
    \STATE Compute reconstructed trajectory $\hat{z}(t_i)$ for $i=\{1,\dots,T\}$ via Eq.\ref{eq:z-hat-full}.
    \STATE Compute $\mathcal{L}_{\text{Total}}$ via Eq.~\eqref{eq:total-c}.
\STATE Update $\phi, \phi^{-1}, \{\lambda_k\}, A, f_0, d_0$ and $\{\theta(t_i)\}$.
\ENDFOR
\STATE \textbf{Inference at target time $t_s > t_T$}
\STATE $\hat{z}_{\mathrm{dom}}(t_s) \leftarrow \operatorname{Re}\!\big( \sum_{k} w_k^{\mathrm{dom}}\, a_k\, e^{\lambda_k t_s} \big)$
\STATE $\hat{z}_{\mathrm{trans}}(t_T) \leftarrow \operatorname{Re}\!\big( \sum_{k} w_k^{\mathrm{trans}}\, a_k\, e^{\lambda_k t_T} \big)$
\STATE $\hat{\theta}(t_s) \leftarrow \phi^{-1}\!\big( \hat{z}_{\mathrm{dom}}(t_s) + \hat{z}_{\mathrm{trans}}(t_T) \big)$
\RETURN $\hat{\theta}(t_s)$
\end{algorithmic}
\end{algorithm}
\subsection{Future Time Inference}
\label{subsubsec:inference}
Once trained, FreKoo++ executes an analytical extrapolation protocol to predict model parameters for an unseen target domain $\mathcal{D}_{t_s}$ at an arbitrary future timestamp $t_s > t_T$. To balance long-term trends with local temporal context, the decomposed components fulfill asymmetric inference roles. Persistent dominant modes, capturing macro-trends and recurring patterns, are extrapolated directly to the future query timestamp $t_s$ via $\hat{z}_{\mathrm{dom}}(t_s)$ in Eq.~\eqref{eq:zdom}. Conversely, transient modes capture volatile local fluctuations that are inherently unpredictable far into the future. Forward-extrapolating these modes risks noise amplification; FreKoo++ thus freezes their contribution at the last observed timestamp $t_T$, retaining $\hat{z}_{\mathrm{trans}}(t_T)$ as stationary residual context. 

Consequently, the predicted latent state $\hat{z}(t_s)$ and reconstructed target parameters $\hat{\theta}(t_s)$ are formulated as:
\begin{align}
\label{eq:infer-z}
    \hat{z}(t_s) &= \hat{z}_{\mathrm{dom}}(t_s) + \hat{z}_{\mathrm{trans}}(t_T), \\
\label{eq:infer-theta}
    \hat{\theta}(t_s) &= \phi^{-1}\!\left( \hat{z}(t_s) \right).
\end{align}
Evaluating closed-form exponential modal expressions at $t_s$ and $t_T$ bypasses recursive step-wise rollouts and numerical ODE solvers, securing constant-time $\mathcal{O}(1)$ query complexity for arbitrary prediction horizons. The complete procedure is detailed in Algorithm~\ref{alg:cfrekoo}.

\begin{discussion}[FreKoo versus FreKoo++]
\label{discussion:frekoo-vs-cfrekoo}
Compared to FreKoo~\cite{yu2025learning}, FreKoo++ achieves three core advances: 1) replacing rigid DFT grids with continuous Koopman eigenvalues ($\lambda_k = \sigma_k + j\omega_k$) to natively support irregular sampling; 2) substituting manual DFT cutoffs with learnable spectral gates ($f_0, d_0$) under a single encoder; and 3) superseding discrete step-wise transitions ($K z_t$) with analytical modal extrapolation ($e^{\lambda_k t_s}$), enabling closed-form $\mathcal{O}(1)$ arbitrary-horizon forecasting.
\end{discussion}

\section{Theoretical Analysis}
\label{sec:theory}
This section provides rigorous theoretical grounding for FreKoo++ by addressing two fundamental claims: 1) the expressiveness of finite continuous Koopman modes in parameterizing latent trajectories, and 2) the explicit risk control provided by our stability ($\mathcal{R}_{\mathrm{stab}}$) and spectral ($\mathcal{R}_{\mathrm{spec}}$) regularizers. We first bound the continuous modal approximation error induced by finite spectrum truncation. We then derive a continuous-time generalization bound for unseen target domains, proving that long-horizon excess risk is strictly governed by individual modal growth rates.

\subsection{Koopman Mode Approximation Error}
\label{sec:theory_approximation}

We first bound the approximation error of the continuous modal parameterization (Eq.~\eqref{eq:kmd-amp}). For notational clarity, theoretical statements use the complex modal form; the same bounds apply to real-valued states since $\operatorname{Re}(\cdot)$ is non-expansive.

\begin{assumption}[Finite Modal Dynamics and Transient Remainder]
\label{assum:true_dynamics}
The true latent dynamics admit a finite $K^*$-order modal expansion $z(t) = \sum_{k=1}^{K^*} a_k e^{\lambda_k t}$ with $a_k \in \mathbb{C}^m$, $\lambda_k \in \mathbb{C}$. For a model retaining $K \le K^*$ modes, the first $K$ represent modeled dynamics; the remaining $K^* - K$ modes constitute a non-growing transient remainder: $\operatorname{Re}(\lambda_k) \le 0$ for all $k > K$.
\end{assumption}

\begin{theorem}[Mode Approximation Bound]
\label{thm:approximation_error}
Under Assumption~\ref{assum:true_dynamics}, let $\hat{z}(t) = \sum_{k=1}^{K} \hat{a}_k e^{\hat{\lambda}_k t}$. Then for any horizon $t_s > 0$:
\begin{equation}
\small
\begin{aligned}
\sup_{t \in [0, t_s]} \|z(t) - \hat{z}(t)\|_2 
&\le \underbrace{\sum_{k=1}^{K} \|a_k - \hat{a}_k\|_2\, e^{\hat{\rho}_k t_s}}_{\text{amplitude error}} \\
& + \underbrace{C_a \sum_{k=1}^{K} |\lambda_k - \hat{\lambda}_k|\, t_s\, e^{\varrho_k t_s}}_{\text{eigenvalue error}} \\
& + \underbrace{\sum_{k=K+1}^{K^*} \|a_k\|_2}_{\text{truncation error}},
\end{aligned}
\label{eq:theorem_approximation}
\end{equation}
where $C_a = \max_{k \le K} \|a_k\|_2$, $\hat{\rho}_k = \max(\operatorname{Re}(\hat{\lambda}_k), 0)$, $\varrho_k = \max(\operatorname{Re}(\lambda_k), \operatorname{Re}(\hat{\lambda}_k), 0)$.
\end{theorem}

\begin{proof}[Proof sketch]
Adding and subtracting $\sum_{k=1}^K a_k e^{\hat{\lambda}_k t}$ decomposes the error into amplitude, eigenvalue, and truncation terms. The amplitude term follows from the triangle inequality and $|e^{\hat{\lambda}_k t}| \le e^{\hat{\rho}_k t_s}$. For the eigenvalue term, the integral identity $e^{\lambda_k t} - e^{\hat{\lambda}_k t} = (\lambda_k - \hat{\lambda}_k)t\int_0^1 e^{[\hat{\lambda}_k + s(\lambda_k - \hat{\lambda}_k)]t}\,ds$ yields $|e^{\lambda_k t} - e^{\hat{\lambda}_k t}| \le |\lambda_k - \hat{\lambda}_k|\,t\,e^{\varrho_k t}$. For the truncation term, $\operatorname{Re}(\lambda_k) \le 0$ implies $|e^{\lambda_k t}| \le 1$. Full proof is detailed in Supplementary~\ref{app:proof_thm1}.
\end{proof}

\textit{\textbf{Insights.}} Theorem~\ref{thm:approximation_error} justifies the objective design in Eq.~\eqref{eq:total-c}: $\mathcal{L}_{\mathrm{fit}}^{(C)}$ and $\mathcal{L}_{\mathrm{rec}}^{(C)}$ directly reduce the amplitude and eigenvalue estimation terms. The eigenvalue term further reveals that spectral errors are exponentially amplified when $\varrho_k > 0$; $\mathcal{R}_{\mathrm{stab}}$ penalizes positive growth rates in dominant modes, encouraging $\hat{\rho}_k \to 0$ so that error growth scales at most linearly in $t_s$. The truncation term depends on the amplitudes of unmodeled modes; $\mathcal{R}_{\mathrm{spec}}$ serves as an empirical surrogate by suppressing amplitudes of transient-weighted modes via $w_k^{\mathrm{trans}}$, while its variance term prevents degenerate gate collapse.

\subsection{Future Time Generalization Bound}
\label{sec:theory_generalization}
We next translate the latent approximation analysis into a generalization bound on future-domain population risk. The derivation adheres to the inference protocol (Eq.~\eqref{eq:infer-z}): dominant dynamics are extrapolated to $t_s > t_T$, while transient fluctuations are frozen at $t_T$. Our bound explicitly characterizes how spectral estimation errors compound over the prediction horizon $\Delta_s = t_s - t_T$.

\begin{assumption}[Regularity, Realizability, and Modal Alignment]
\label{assum:lipschitz}
The task loss $\ell$, predictive backbone $g$, and decoder $\phi^{-1}$ are respectively $L_\ell$-, $L_g$-, and $L_{\mathrm{dec}}$-Lipschitz.
The optimal parameter trajectory is realizable in the latent space: there exists $z^\star(t)$ such that $\theta^\star(t) = \phi^{-1}(z^\star(t))$, decomposable as $z^\star(t) = z^\star_{\mathrm{dom}}(t) + z^\star_{\mathrm{trans}}(t)$, where the future transient remainder satisfies $\sup_{t \ge t_T}\|z^\star_{\mathrm{trans}}(t)\|_2 \le B_{\mathrm{trans}}$.
The dominant components admit matched modal representations anchored at $t_T$: for any $\tau \ge 0$,
\begin{equation}
\centering
\begin{aligned}
&z^\star_{\mathrm{dom}}(t_T+\tau) = \operatorname{Re}\!\Big(\textstyle\sum_{k=1}^K b_k^\star e^{\lambda_k^\star \tau}\Big), \\
&\hat{z}_{\mathrm{dom}}(t_T+\tau) = \operatorname{Re}\!\Big(\textstyle\sum_{k=1}^K \hat{b}_k e^{\hat{\lambda}_k \tau}\Big),
\end{aligned}
\end{equation}
where $\hat{b}_k = w_k^{\mathrm{dom}}\hat{a}_k e^{\hat{\lambda}_k t_T}$ and $\max_{1 \le k \le K} \|b_k^\star\|_2 \le B_{\mathrm{dom}}$. All learned modal parameters remain bounded after training convergence; in particular, the historical window factor $G_{\mathrm{trans}}(t_T) = \big(\sum_{k=1}^K w_k^{\mathrm{trans}} e^{2\hat{\rho}_k t_T}\big)^{1/2}$ is finite.
\end{assumption}

\begin{theorem}[Continuous Time Generalization Bound]
\label{thm:generalization_bound}
Let $\mathcal{R}_{t_s}(\theta) = \mathbb{E}_{P_{t_s}}[\ell(g(X;\theta),Y)]$. Under Assumptions~\ref{assum:true_dynamics}--\ref{assum:lipschitz}, the excess risk $\mathcal{E}_{t_s} := \mathcal{R}_{t_s}(\hat{\theta}(t_s)) - \mathcal{R}_{t_s}(\theta^\star(t_s))$ satisfies:
\begin{equation}
\label{eq:generalization_bound}
\mathcal{E}_{t_s} \le L_\ell L_g L_{\mathrm{dec}}\big(E_{\mathrm{dom}}(t_s) + E_{\mathrm{trans}}(t_s)\big),
\end{equation}
where $E_{\mathrm{dom}}(t_s) = \|\hat{z}_{\mathrm{dom}}(t_s) - z^\star_{\mathrm{dom}}(t_s)\|_2$ and $E_{\mathrm{trans}}(t_s) = \|\hat{z}_{\mathrm{trans}}(t_T) - z^\star_{\mathrm{trans}}(t_s)\|_2$. With $\Delta_s = t_s - t_T$, $\hat{\rho}_k = \max(\operatorname{Re}(\hat{\lambda}_k),0)$, and $\varrho_k = \max(\operatorname{Re}(\lambda_k^\star),\operatorname{Re}(\hat{\lambda}_k),0)$:
\begin{equation}
\label{eq:edom-bound}
\begin{aligned}
E_{\mathrm{dom}}(t_s) \le &\underbrace{\sum_{k=1}^K \|\hat{b}_k - b_k^\star\|_2\, e^{\hat{\rho}_k \Delta_s}}_{\text{amplitude mismatch}} \\
&+ \underbrace{B_{\mathrm{dom}}\sum_{k=1}^K |\hat{\lambda}_k - \lambda_k^\star|\,\Delta_s\, e^{\varrho_k \Delta_s}}_{\text{eigenvalue mismatch}}.
\end{aligned}
\end{equation}
The transient holding error satisfies:
\begin{equation}
\label{eq:etrans-bound}
E_{\mathrm{trans}}(t_s) \le B_{\mathrm{trans}} + G_{\mathrm{trans}}(t_T)\sqrt{\mathcal{R}_{\mathrm{spec}}}.
\end{equation}
\end{theorem}

\begin{proof}[Proof sketch]
Lipschitz continuity of $\ell$, $g$, and $\phi^{-1}$ gives $\mathcal{E}_{t_s} \le L_\ell L_g L_{\mathrm{dec}}\|\hat{z}(t_s) - z^\star(t_s)\|_2$; the triangle inequality with the inference rule yields Eq.~\eqref{eq:generalization_bound}. For the dominant branch, adding and subtracting $b_k^\star e^{\hat{\lambda}_k\Delta_s}$ and bounding the exponential difference via the integral identity $e^{\hat{\lambda}_k\Delta_s} - e^{\lambda_k^\star\Delta_s} = (\hat{\lambda}_k - \lambda_k^\star)\Delta_s\int_0^1 e^{[\lambda_k^\star + u(\hat{\lambda}_k - \lambda_k^\star)]\Delta_s}du$ gives Eq.~\eqref{eq:edom-bound}. For the transient branch, Cauchy--Schwarz yields $\|\hat{z}_{\mathrm{trans}}(t_T)\|_2 \le G_{\mathrm{trans}}(t_T)\sqrt{\sum_k w_k^{\mathrm{trans}}\|\hat{a}_k\|_2^2} \le G_{\mathrm{trans}}(t_T)\sqrt{\mathcal{R}_{\mathrm{spec}}}$, where the last step uses $1-\operatorname{Var}(w_k^{\mathrm{dom}}) \ge 0$. Full proofs are detailed in Supplementary~\ref{app:proof_thm2}.
\end{proof}

\noindent\textbf{Insights.} Theorem~\ref{thm:generalization_bound} shows that dominant extrapolation error is governed by anchor mismatch and spectral estimation error, both scaling with $e^{\hat{\rho}_k\Delta_s}$; $\mathcal{R}_{\mathrm{stab}}$ drives $\hat{\rho}_k\!\to\!0$, suppressing exponential divergence. For the transient branch, the proposed regularizers provide \emph{dual coverage} over positive modal growth: a growing mode ($\sigma_k>0$) with high $w_k^{\mathrm{dom}}$ is penalized by $\mathcal{R}_{\mathrm{stab}}$, while one routed to the transient branch ($w_k^{\mathrm{trans}}\!\approx\!1$) has its amplitude suppressed by $\mathcal{R}_{\mathrm{spec}}$. This joint effect keeps $G_{\mathrm{trans}}(t_T)$ well-conditioned in practice; when all transient-weighted modes are non-growing, $G_{\mathrm{trans}}(t_T) \le \sqrt{K}$. Crucially, freezing transients at $t_T$ renders $E_{\mathrm{trans}}$ independent of $\Delta_s$, structurally insulating future predictions from horizon-dependent noise amplification.

\begin{corollary}[Elimination of Exponential Error Amplification]
\label{cor:bounded_risk}
If $\operatorname{Re}(\lambda_k^\star)\le 0$ and $\operatorname{Re}(\hat{\lambda}_k)\le 0$ for all $k$, defining $\varepsilon_b = \sum_{k=1}^K\|\hat{b}_k - b_k^\star\|_2$ and $\varepsilon_\lambda = \sum_{k=1}^K|\hat{\lambda}_k - \lambda_k^\star|$, then:
\begin{equation}
\label{eq:stable-dom-bound}
E_{\mathrm{dom}}(t_s) \le \varepsilon_b + B_{\mathrm{dom}}(t_s - t_T)\,\varepsilon_\lambda.
\end{equation}
Stable dominant dynamics thus eliminate exponential error compounding; under exact spectral recovery ($\varepsilon_\lambda=0$), the error is bounded by $\varepsilon_b$ regardless of the horizon.
\end{corollary}

\begin{table*}[t]
\centering
\caption{Performance comparison on \textbf{CTDG datasets}. The classification tasks report error rates (\%) except for the AUC for the Twitter dataset. The regression tasks report MAE. '-' implies that the method does not support the task.}
\label{tab:ctdg_results}
\begin{tabularx}{0.9\textwidth}{lYYYYYY} 
\toprule
\multirow{2}{*}{Methods} & \multicolumn{4}{c}{Classification} & \multicolumn{2}{c}{Regression} \\
\cmidrule(lr){2-5} \cmidrule(lr){6-7}
 & 2-Moons-C $\downarrow$ & Rot-MNIST-C $\downarrow$ & Twitter $\uparrow$ & Yearbook $\downarrow$ & Cyclone  $\downarrow$ & House-C  $\downarrow$\\
\midrule
Offline       & 13.5 $\pm$ 0.3 & 6.6 $\pm$ 0.2 & 0.54 $\pm$ 0.09 & 8.6 $\pm$ 1.0 & 18.7 $\pm$ 1.4 & 19.9 $\pm$ 0.1 \\
LastDomain     & 55.7 $\pm$ 0.5 & 74.2 $\pm$ 0.9 & 0.54 $\pm$ 0.12 & 11.3 $\pm$ 1.3 & 22.3 $\pm$ 0.7 & 20.6 $\pm$ 0.7 \\
IncFinetune    & 51.9 $\pm$ 0.7 & 57.1 $\pm$ 1.4 & 0.52 $\pm$ 0.01 & 11.0 $\pm$ 0.8 & 19.9 $\pm$ 0.7 & 20.6 $\pm$ 0.2 \\
IRM~\cite{arjovsky2019invariant}            & 15.6 $\pm$ 0.2 & 8.6 $\pm$ 0.4 & 0.53 $\pm$ 0.11 & 8.3 $\pm$ 0.5 & 18.0 $\pm$ 0.8 & 19.8 $\pm$ 0.2 \\
V-REx~\cite{krueger2021out}           & 12.8 $\pm$ 0.2 & 8.6 $\pm$ 0.3 & 0.58 $\pm$ 0.05 & 8.9 $\pm$ 0.5 & 17.7 $\pm$ 0.5 & 20.2 $\pm$ 0.1 \\
CIDA~\cite{wang2020continuously}            & 18.7 $\pm$ 2.0 & 8.3 $\pm$ 0.7 & 0.63 $\pm$ 0.03 & 8.4 $\pm$ 0.8 & 17.0 $\pm$ 0.4 & 10.2 $\pm$ 1.0 \\
TKNets~\cite{zeng2024generalizing}          & 39.6 $\pm$ 1.2 & 37.7 $\pm$ 2.0 & 0.57 $\pm$ 0.04 & 8.4 $\pm$ 0.3 & - & - \\
DRAIN~\cite{bai2023temporal}           & 53.2 $\pm$ 0.9 & 59.1 $\pm$ 2.3 & 0.57 $\pm$ 0.04 & 10.5 $\pm$ 1.0 & 23.6 $\pm$ 0.5 & 9.8 $\pm$ 0.1 \\
DRAIN-$\Delta t$~\cite{bai2023temporal}   & 46.2 $\pm$ 0.8 & 57.2 $\pm$ 1.8 & 0.59 $\pm$ 0.02 & 11.0 $\pm$ 1.2 & 26.2 $\pm$ 4.6 & 9.9 $\pm$ 0.1 \\
DeepODE~\cite{chen2018neural}          & 17.8 $\pm$ 5.6 & 48.6 $\pm$ 3.2 & 0.64 $\pm$ 0.02 & 13.0 $\pm$ 2.1 & 18.5 $\pm$ 3.3 & 10.7 $\pm$ 0.4 \\
NeuralLio~\cite{cai2025continuous} & 4.5 $\pm$ 1.3 & 5.4 $\pm$ 1.1 & 0.68 $\pm$ 0.01 &  \underline{6.1 $\pm$ 1.2}  & \underline{16.2 $\pm$ 0.2}  & \underline{9.0 $\pm$ 0.4}\\
Koodos~\cite{cai2024continuous}   & \underline{2.8 $\pm$ 0.7} & \underline{4.6 $\pm$ 0.1} & \underline{0.71 $\pm$ 0.02} & {6.6 $\pm$ 1.3} & {16.4 $\pm$ 0.3} & \underline{9.0 $\pm$ 0.2} \\
\midrule
\rowcolor{gray!15} \textbf{FreKoo++} (ours)  & \textbf{1.3 $\pm$ 0.2} & \textbf{3.8 $\pm$ 0.8} & \textbf{0.72 $\pm$ 0.01} & \textbf{5.4 $\pm$ 0.4} & \textbf{16.0 $\pm$ 0.2} & \textbf{8.8 $\pm$ 0.1}  \\
\bottomrule
\end{tabularx}
\end{table*}

\begin{table*}[t]
\centering
\caption{Performance comparison on \textbf{DTDG datasets}. Classification tasks report error rates (\%), and regression tasks report MAE. ``--'' indicates that the method does not support the task.}
\label{tab:dtdg_results}
\setlength{\tabcolsep}{2.1pt}
\begin{tabularx}{0.93\textwidth}{lYYYYYYY}
\toprule
\multirow{2}{*}{Methods} & \multicolumn{5}{c}{Classification} & \multicolumn{2}{c}{Regression} \\
\cmidrule(lr){2-6} \cmidrule(lr){7-8}
& 2-Moons-D $\downarrow$ & Rot-MNIST-D $\downarrow$ & ONP $\downarrow$ & Shuttle $\downarrow$ & Elec2 $\downarrow$ & House-D $\downarrow$ & Appliance $\downarrow$ \\
\midrule
Offline & 22.4 $\pm$ 4.6 & 18.6 $\pm$ 4.0 & 33.8 $\pm$ 0.6 & 0.77 $\pm$ 0.10 & 23.0 $\pm$ 3.1 & 11.0 $\pm$ 0.36 & 10.2 $\pm$ 1.1 \\
LastDomain & 14.9 $\pm$ 0.9 & 17.2 $\pm$ 3.1 & 36.0 $\pm$ 0.2 & 0.91 $\pm$ 0.18 & 25.8 $\pm$ 0.6 & 10.3 $\pm$ 0.16 & 9.1 $\pm$ 0.7 \\
IncFinetune & 16.7 $\pm$ 3.4 & 10.1 $\pm$ 0.8 & 34.0 $\pm$ 0.3 & 0.83 $\pm$ 0.07 & 27.3 $\pm$ 4.2 & 9.7 $\pm$ 0.01 & 8.9 $\pm$ 0.5 \\
CDOT~\cite{ortiz2019cdot} & 9.3 $\pm$ 1.0 & 14.2 $\pm$ 1.0 & 34.1 $\pm$ 0.0 & 0.94 $\pm$ 0.17 & 17.8 $\pm$ 0.6 & -- & -- \\
CIDA~\cite{wang2020continuously} & 10.8 $\pm$ 1.6 & 9.3 $\pm$ 0.7 & 34.7 $\pm$ 0.6 & -- & 14.1 $\pm$ 0.2 & 9.7 $\pm$ 0.06 & 8.7 $\pm$ 0.2 \\
GI~\cite{nasery2021training} & 3.5 $\pm$ 1.4 & 7.7 $\pm$ 1.3 & 36.4 $\pm$ 0.8 & 0.29 $\pm$ 0.05 & 16.9 $\pm$ 0.7 & 9.6 $\pm$ 0.02 & 8.2 $\pm$ 0.6 \\
LSSAE~\cite{qin2022generalizing} & 9.9 $\pm$ 1.1 & 9.8 $\pm$ 3.6 & 38.8 $\pm$ 1.1 & 0.22 $\pm$ 0.01 & 16.1 $\pm$ 1.4 & -- & -- \\
DDA~\cite{zeng2023foresee} & 9.7 $\pm$ 1.5 & 7.6 $\pm$ 0.7 & 34.0 $\pm$ 0.3 & 0.21 $\pm$ 0.02 & 12.8 $\pm$ 1.1 & 9.5 $\pm$ 0.12 & 6.1 $\pm$ 0.1 \\
DRAIN~\cite{bai2023temporal} & 3.2 $\pm$ 1.2 & 7.5 $\pm$ 1.1 & 38.3 $\pm$ 1.2 & 0.26 $\pm$ 0.05 & 12.7 $\pm$ 0.8 & 9.3 $\pm$ 0.14 & 6.4 $\pm$ 0.4 \\
EvoS~\cite{xie2024evolving} & 3.0 $\pm$ 0.4 & 7.3 $\pm$ 0.6 & 35.4 $\pm$ 0.2 & 0.23 $\pm$ 0.01 & 11.8 $\pm$ 0.5 & 9.8 $\pm$ 0.10 & 7.2 $\pm$ 0.1 \\
Koodos~\cite{cai2024continuous} & 1.3 $\pm$ 0.4 & 7.0 $\pm$ 0.3 & 33.5 $\pm$ 0.4 & 0.24 $\pm$ 0.04 & -- & \textbf{8.8 $\pm$ 0.19} & 4.8 $\pm$ 0.3 \\
FreKoo~\cite{yu2025learning} & \underline{1.0 $\pm$ 0.3} & \underline{6.9 $\pm$ 0.7} & \underline{32.3 $\pm$ 0.3} & \underline{0.20 $\pm$ 0.02} & \textbf{9.2 $\pm$ 0.7} & \underline{9.0 $\pm$ 0.11} & \underline{4.0 $\pm$ 0.1} \\
\midrule
\rowcolor{gray!15} \textbf{FreKoo++} (ours) & \textbf{0.9 $\pm$ 0.8} & \textbf{5.5 $\pm$ 0.7} & \textbf{31.1 $\pm$ 0.2}  & \textbf{0.19 $\pm$ 0.01} &  \underline{10.1 $\pm$ 1.0} & 9.3 $\pm$ 0.10  & \textbf{3.9 $\pm$ 0.0} \\
\bottomrule
\end{tabularx}
\end{table*}

\section{Experiments}
\label{sec:experiments}

\subsection{Experimental Settings}
\label{sec:exp_settings}

\subsubsection{Datasets}
We evaluate FreKoo++ across both CTDG and DTDG settings. For CTDG, we benchmark on six irregularly sampled temporal datasets, including continuous variants of \emph{2-Moons} and \emph{Rot-MNIST}~\cite{deng2012mnist}, \emph{Twitter}~\cite{zhao2017feature,bai2023sign}, \emph{Yearbook}~\cite{yao2022wild}, \emph{Cyclone}~\cite{chen2018rotation}, and \emph{House}~\cite{cai2024continuous}. For DTDG, we follow the evaluation protocol of FreKoo~\cite{yu2025learning} and test on seven uniformly sampled temporal datasets, including \emph{2-Moons}, \emph{Rot-MNIST}, \emph{ONP}, \emph{Shuttle}, \emph{Elec2}, \emph{HousePrices}, and \emph{ApplianceEnergy}. Because several benchmarks share similar names across the two paradigms but differ in their construction protocols, we append the suffix `-C' to CTDG variants and `-D' to DTDG variants for clarity. Detailed dataset statistics and construction procedures are provided in Supplementary~\ref{app:datasets}.

\subsubsection{Baselines and Implementations}
For CTDG, we compare against time-agnostic baselines, including \emph{Offline}, \emph{LastDomain}, and \emph{IncFinetune}; domain generalization regularizers \emph{IRM}~\cite{arjovsky2019invariant} and \emph{V-REx}~\cite{krueger2021out}; continuous domain adaptation methods \emph{CIDA}~\cite{wang2020continuously}; and temporal modeling baselines \emph{TKNets}~\cite{zeng2024generalizing}, \emph{DRAIN}~\cite{bai2023temporal}, \emph{DRAIN-$\Delta t$}~\cite{bai2023temporal}, \emph{DeepODE}~\cite{chen2018neural}, \emph{NeuralLio}~\cite{cai2025continuous} and \emph{Koodos}~\cite{cai2024continuous}. For DTDG, we additionally compare with the full suite of baselines adopted in FreKoo~\cite{yu2025learning}. Note that FreKoo is evaluated exclusively under DTDG, as its Discrete Fourier Transform (DFT) requires uniformly spaced timestamps. Detailed descriptions are provided in Supplementary~\ref{app:baselines}.
All methods follow the same data splits, preprocessing, task backbones, and evaluation protocols within each benchmark setting. FreKoo++ employs lightweight MLPs as the encoder $\phi$ and decoder $\phi^{-1}$, with the number of Koopman modes $K$ set equal to the latent dimension $m$. Detailed implementation, hyperparameter, and training details are provided in Supplementary~\ref{app:implementation}.

\subsection{Main Results}
\label{sec:exp_main}

\subsubsection{Continuous Temporal Domain Generalization}
Table~\ref{tab:ctdg_results} presents the main performance evaluation across all CTDG benchmarks, where source domains arrive at irregular timestamps and target domains demand arbitrary-horizon forecasting. FreKoo++ consistently outperforms existing baselines in both classification and regression tasks, confirming that modeling parameter trajectories within a continuous spectral-dynamical space is more effective than either omitting temporal ordering or enforcing discrete transition rules. 

Specifically, time-agnostic approaches (e.g., Offline, LastDomain, IncFinetune) flatten historical domains into static representations, thereby obscuring persistent temporal drift amidst localized, domain-specific fluctuations. Discrete temporal baselines (e.g., DRAIN, DRAIN-$\Delta t$) incorporate parameter transitions, yet their stepwise prediction mechanisms remain bound to local iterations and lack a global spectral description of the trajectory. While continuous-time architectures (e.g., DeepODE, Koodos) naturally handle irregular observation intervals, they do not utilize the learned spectrum to separate persistent and transient parameter modes. In contrast, FreKoo++ parameterizes parameter evolution as a superposition of eigenvalue-indexed Koopman modes, where the real and imaginary components explicitly quantify temporal persistence and oscillatory frequency, respectively. This continuous spectral formulation enables the analytical extrapolation of stable dominant dynamics while insulating future predictions against volatile, non-generalizable noise. These consistent empirical gains validate that our proposed continuous-time dynamics and spectral disentanglement are exceptionally effective for continuous domains and long-horizon extrapolation.

\subsubsection{Discrete Temporal Domain Generalization}
Table~\ref{tab:dtdg_results} evaluates FreKoo++ under DTDG benchmarks with uniformly spaced temporal grids. Although designed for continuous irregularly sampled domains, FreKoo++ remains highly competitive under uniformly sampled DTDG benchmarks. It improves upon FreKoo on five of seven datasets, including 2-Moons-D, Rot-MNIST-D, ONP, Shuttle, and Appliance, while FreKoo remains stronger on Elec2 and House-D. These results demonstrate the strong backward compatibility and versatility of FreKoo++, proving that its continuous spectral-dynamical formulation remains exceptionally effective regardless of whether temporal domains are regularly or irregularly sampled.

\subsection{Ablation Studies}
\label{sec:exp_ablation}
To systematically validate the design rationale of FreKoo++, we conduct ablation studies on representative CTDG benchmarks across two complementary dimensions: \emph{method-level architectural variants} and \emph{objective-level regularization variants}. As summarized in Table~\ref{tab:ablation}, the empirical results demonstrate that every component in FreKoo++ plays an indispensable role in securing robust cross-temporal generalization.

The \emph{method-level} variants confirm that our core architectural choices provide decisive performance guarantees. First, enforcing $\sigma_k \equiv 0$ (\textit{Fixed spectrum}) restricts modes to purely unattenuated oscillations ($\lambda_k = j\omega_k$) and neglects temporal growth and decay dynamics. This causes catastrophic degradation across all tasks: classification error surges from $1.3\%$ to $9.7\%$ on 2-Moons-C, Twitter AUC drops sharply to $0.64$, and Cyclone MAE degrades to $17.2$. These drops prove that modeling growth and decay rates ($\sigma_k$) is as vital as modeling oscillatory frequencies ($\omega_k$) for tracking parameter evolution. Second, replacing adaptive soft gates with \textit{Hard gating} severely impairs generalization (e.g., Cyclone MAE worsens to $18.7$), confirming that rigid frequency cutoffs introduce severe boundary artifacts. Similarly, locking gating thresholds via \textit{Frozen gating} consistently degrades performance, validating the necessity of dynamically adapting spectral boundaries to data-specific temporal profiles.

The \emph{objective-level} variants demonstrate that all regularizers and loss terms contribute synergistically to trajectory fidelity and extrapolation stability. Omitting transient spectral regularization ($\mathcal{R}_{\mathrm{spec}}$) triggers the most severe performance drop (2-Moons-C error increases to $6.7\%$), directly validating Proposition~\ref{prop:energy} by confirming that penalizing transient amplitudes $\|a_k\|_2^2$ is essential for filtering domain-specific noise. Disabling stability regularization ($\mathcal{R}_{\mathrm{stab}}$) impairs accuracy on 2-Moons-C ($1.3\%$ vs. $2.5\%$), supporting Theorem~\ref{thm:generalization_bound} and Corollary~\ref{cor:bounded_risk} by verifying that penalizing positive growth rates ($\sigma_k > 0$) prevents dominant extrapolation modes from diverging exponentially. Finally, removing either $\mathcal{L}_{\mathrm{rec}}^{(C)}$ or $\mathcal{L}_{\mathrm{fit}}^{(C)}$ degrades performance across all datasets, confirming their complementary roles in preserving auto-encoding fidelity and continuous modal trajectory alignment.

\begin{table}[t]
\centering
\caption{Ablation studies on representative CTDG benchmarks. We report error rate (\%, $\downarrow$) for 2-Moons-C, AUC ($\uparrow$) for Twitter, and MAE ($\downarrow$) for Cyclone. The best performance in each column is in \textbf{bold}.}
\label{tab:ablation}
\setlength{\tabcolsep}{8.5pt}
\begin{tabular}{lccc}
\toprule
Variant & 2-Moons-C & Twitter & Cyclone \\
\cmidrule(lr){1-1} \cmidrule(lr){2-2} \cmidrule(lr){3-3} \cmidrule(lr){4-4}
& Err. $\downarrow$ & AUC $\uparrow$ & MAE $\downarrow$ \\
\midrule
\multicolumn{4}{l}{\textit{Method-level variants}} \\
Fixed spectrum  & 9.7 $\pm$ 1.5 & 0.64 $\pm$ 0.11 & 17.2 $\pm$ 0.7 \\
Hard gating & 6.7 $\pm$ 2.8 & 0.67 $\pm$ 0.09 & 18.7 $\pm$ 0.3 \\
Frozen gating  & 2.8 $\pm$ 0.3 & 0.70 $\pm$ 0.01 & 16.2 $\pm$ 0.2 \\
\midrule
\multicolumn{4}{l}{\textit{Objective-level variants}} \\
w/o  $\mathcal{L}_{\mathrm{rec}}^{(C)}$ & 1.6 $\pm$ 0.5 & 0.70 $\pm$ 0.02 & 16.4 $\pm$ 0.3 \\
w/o $\mathcal{L}_{\mathrm{fit}}^{(C)}$ & 2.0 $\pm$ 0.5 & 0.68 $\pm$ 0.02 & 16.2 $\pm$ 0.2 \\
w/o  $\mathcal{R}_{\mathrm{stab}}$ & 2.5 $\pm$ 1.4 & 0.70 $\pm$ 0.01 & 16.1 $\pm$ 0.1 \\
w/o $\mathcal{R}_{\mathrm{spec}}$   & 6.7 $\pm$ 3.3 & 0.69 $\pm$ 0.01 & 16.8 $\pm$ 0.4 \\
\midrule
\textbf{FreKoo++} & \textbf{1.3 $\pm$ 0.2} & \textbf{0.72 $\pm$ 0.01} & \textbf{16.0 $\pm$ 0.2} \\
\bottomrule
\end{tabular}
\end{table}

\begin{figure*}[!t]
\centering
\subfigure[Dominant--transient decomposition]{
    \includegraphics[width=0.31\linewidth]{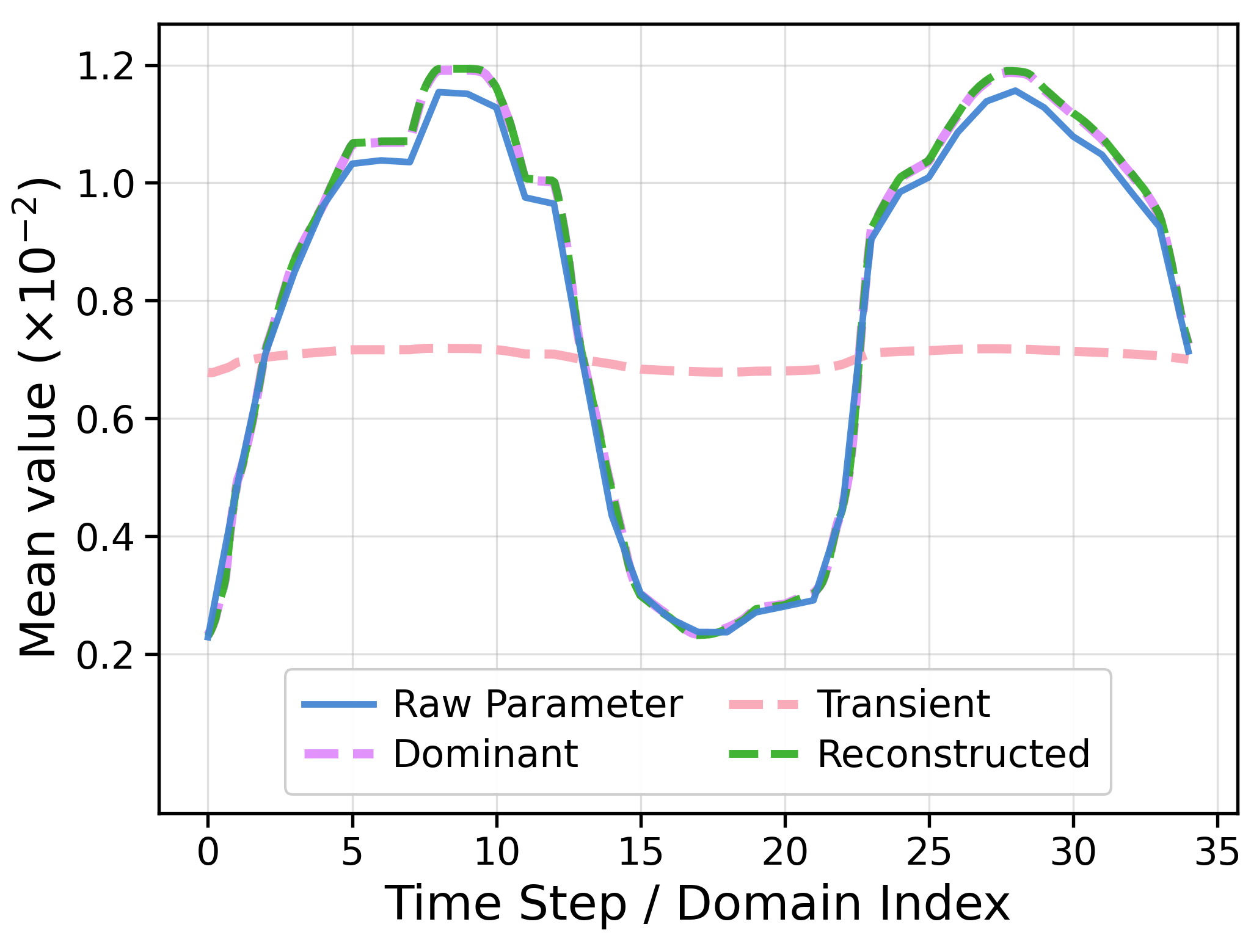}
    \label{subfig:moons_param_trajectory}
    }
\subfigure[Adaptive spectral gating]{
    \includegraphics[width=0.31\linewidth]{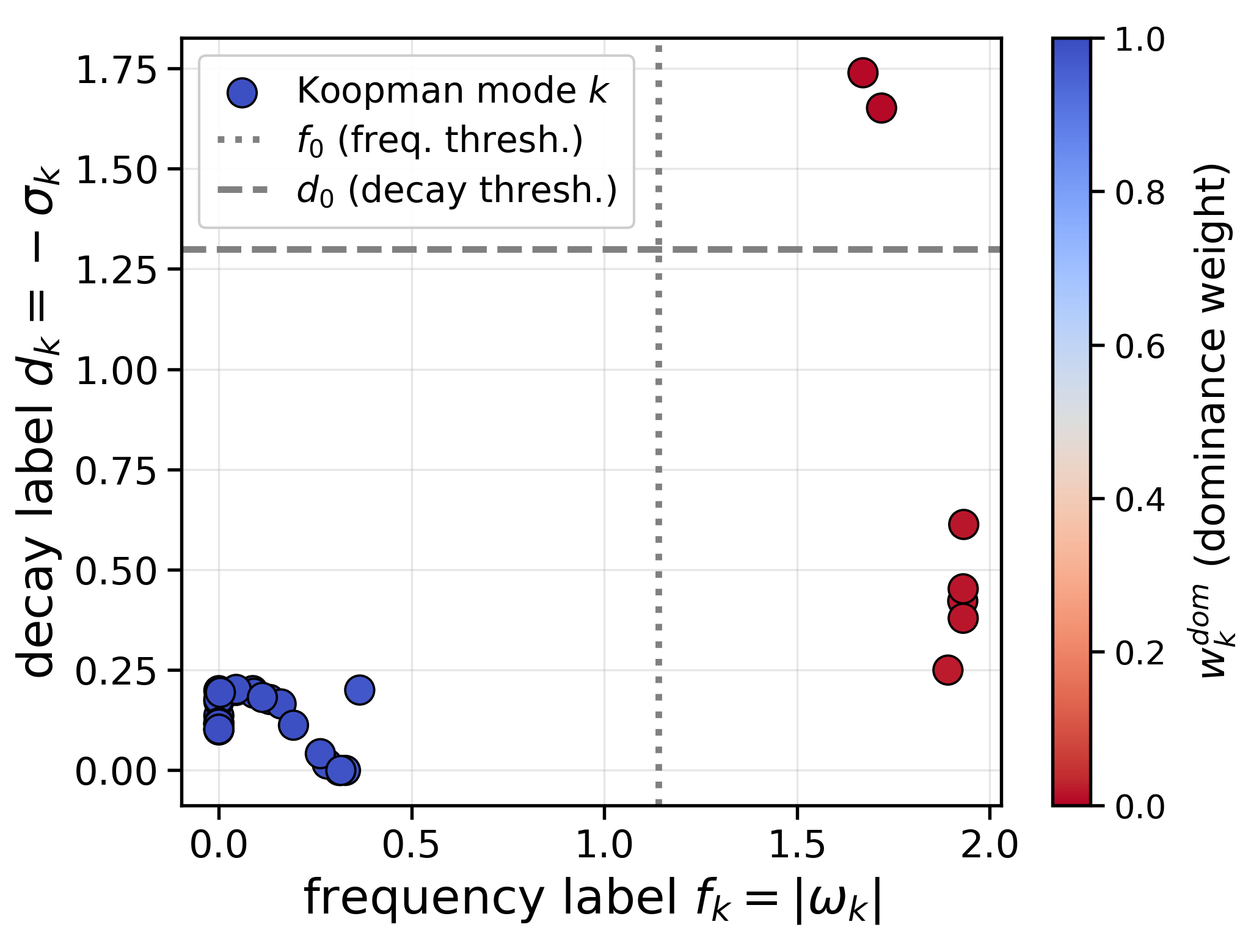}
    \label{subfig:moons_transient_spectrum}
    }
\subfigure[Future trajectory extrapolation]{
       \includegraphics[width=0.31\linewidth]{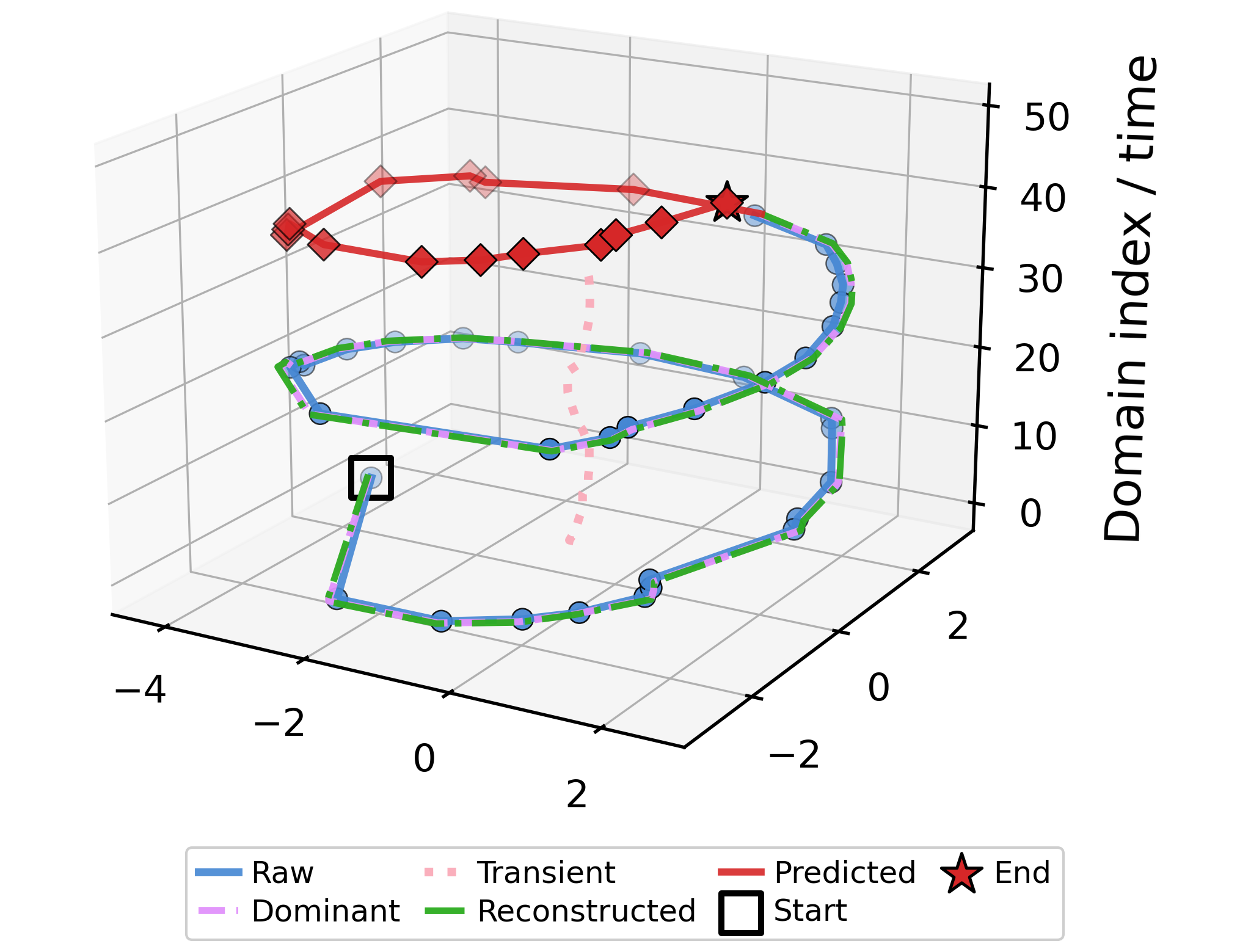}
       \label{subfig:moons_predicted_trajectory}
}
\caption{Qualitative spectral analysis on 2-Moons-C. (a) The parameter trajectory is decomposed into a smooth dominant component and a short-lived transient residual. (b) Learned modes are organized in the frequency--decay plane, where adaptive gates separate persistent modes from transient ones without fixed DFT bins or manual frequency thresholds. (c) Continuous-time inference extrapolates the dominant dynamics while freezing the transient component, producing a coherent future trajectory.}
\label{fig:visualization}
\end{figure*}
\subsection{Qualitative Spectral Analysis}
\label{sec:exp_spectral}

To provide qualitative intuition for our continuous spectral mechanism, Figure~\ref{fig:visualization} visualizes the learned modal decomposition on 2-Moons-C across three complementary perspectives. Notably, all panels are generated from a single trained model to ensure consistency.

\subsubsection{Parameter-space trajectory decomposition}
Figure~\ref{subfig:moons_param_trajectory} illustrates how raw parameter evolution is cleanly disentangled into distinct spectral branches. The smooth dominant component tracks the persistent, recurring backbone of the parameter trajectory, whereas the transient component isolates localized, short-lived residual fluctuations. The high reconstruction fidelity confirms that our continuous modal parameterization faithfully preserves predictable long-term dynamics without overfitting domain-specific noise.

\subsubsection{Spectral distribution and adaptive gating}
Figure~\ref{subfig:moons_transient_spectrum} plots the learned modal coordinates in the frequency--decay plane $(f_k, d_k)$, where $f_k = |\omega_k|$ and $d_k = -\sigma_k$. Bypassing rigid DFT grids and manual frequency cutoffs, the learnable soft gates carve out a smooth, data-driven dominance boundary in the continuous spectral space. Persistent modes with low frequency and minimal decay ($\sigma_k \approx 0$) receive high dominance weights ($w_k^{\mathrm{dom}} \approx 1$), whereas high-frequency or rapidly decaying modes are automatically routed to the transient branch.

\subsubsection{Extrapolated space-time trajectory}
Figure~\ref{subfig:moons_predicted_trajectory} projects continuous parameter evolution into a 3D space-time representation, highlighting FreKoo++'s exceptional long-horizon extrapolation capability governed by Eqs.~\eqref{eq:infer-z}--\eqref{eq:infer-theta}. Far beyond the historical training window ($t_s \gg t_T$), the extrapolated dominant segment seamlessly extends the smooth cyclical trajectory into distant future domains without trajectory collapse or numerical divergence. Meanwhile, the transient residual is held constant at $t_T$ as stationary local context. This visually validates both our long-term extrapolation stability and asymmetric inference protocol, as propagating persistent macro-dynamics secures reliable far-future forecasting while strictly preventing volatile local noise from corrupting long-term domain generalization.

\begin{figure*}[h]
\centering
\subfigure[Koodos]{
    \includegraphics[width=1\linewidth]{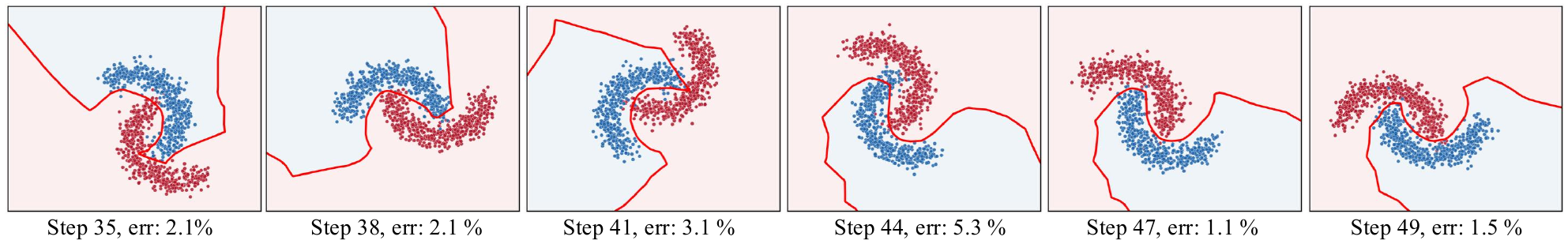}
    \label{fig:koodos_cb}
    }
\subfigure[FreKoo++]{
    \includegraphics[width=1\linewidth]{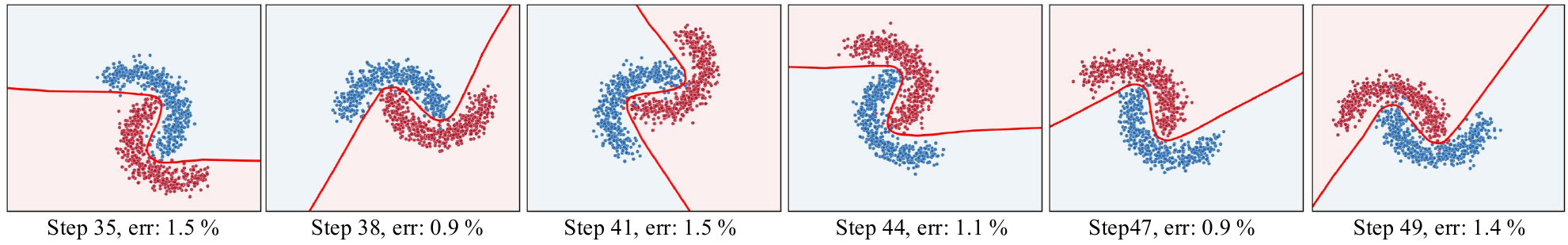}
    \label{fig:frekoo++_cb}
    }
\caption{Qualitative comparison of decision boundaries learned on 2-Moons-C across six unseen future domains. Red curves denote learned decision boundaries alongside the corresponding classification error rates. 
}
\label{fig:boundary_visualization}
\end{figure*}

\subsection{Qualitative Analysis of Decision Boundaries}
\label{sec:decision_boundary}

To evaluate how continuous parameter trajectory evolution translates into task-level generalization, Figure~\ref{fig:boundary_visualization} visualizes the learned decision boundaries on 2-Moons-C across six unseen future temporal domains (steps $35$ to $49$). We compare FreKoo++ against the strongest continuous-time baseline Koodos~\cite{cai2024continuous}.
As shown in Figure~\ref{fig:koodos_cb}, Koodos produces convoluted and irregular decision boundaries. As time extends further beyond the historical training window (e.g., at step $44$), Koodos suffers from severe boundary degradation, causing the classification error to surge to $5.3\%$. This instability stems from unconstrained continuous trajectory modeling, which overfits transient local noise and amplifies integration errors over longer extrapolation horizons. 
In contrast, FreKoo++ (Figure~\ref{fig:frekoo++_cb}) yields remarkably smooth, globally consistent, and physically intuitive decision boundaries that naturally rotate in synchronization with the underlying moon-cluster dynamics. Crucially, as the prediction horizon extends into distant future domains (steps $44$ and $49$), FreKoo++ maintains exceptional accuracy without experiencing performance degradation. These results are consistent with the stability predicted by Theorem~\ref{thm:generalization_bound} and Corollary~\ref{cor:bounded_risk}, demonstrating that enforcing dominant mode stability ($\sigma_k^{\mathrm{dom}} \le 0$) effectively prevents boundary distortion and guarantees generalization robustness over continuous time. Additional decision boundary visualizations are provided in Supplementary~\ref{app:tsne_analysis}.

\section{Conclusion}
In this paper, we presented FreKoo++, a continuous-time spectral-dynamical framework for temporal domain generalization. Motivated by the spectral-dynamical relation of continuous Koopman dynamics, FreKoo++ represents the evolution of task parameters through learnable modal components whose eigenvalues jointly encode oscillatory frequency and temporal growth or decay. This formulation removes the dependence on uniformly sampled temporal domains and fixed discrete transitions, allowing parameter evolution to be learned from irregular observations and evaluated at arbitrary future timestamps. In addition, FreKoo++ further introduces an adaptive soft spectral weighting mechanism that separates persistent dominant dynamics from transient domain-specific fluctuations, together with spectral and stability regularization terms that suppress noisy modal components and control unstable long-horizon extrapolation. Theoretical analysis establishes modal approximation and continuous time generalization bounds, clarifying how the learned spectrum and the proposed regularizers affect future-domain risk. Experiments on both discrete and continuous TDG benchmarks demonstrate that FreKoo++ consistently improves temporal generalization. These results suggest that continuous spectral-dynamical modeling provides a principled and effective direction for learning robust models in evolving environments.

\section*{Acknowledgments}
The work was supported by the Australian Research Council (ARC) under Laureate project FL190100149 and Discovery Project DP220102635.

\bibliographystyle{IEEEtran}
\bibliography{ref}

\clearpage
\newpage
\onecolumn
\setcounter{page}{1}
\renewcommand{\thepage}{\arabic{page}}
\renewcommand{\appendixname}{Section}
\appendices
\renewcommand{\thesection}{\Roman{section}}
\renewcommand{\thesectiondis}{\Roman{section}}

\begin{center}
{\LARGE Supplementary Material for FreKoo++: Learning Continuous Spectral Dynamics for Temporal Domain Generalization\par}
\end{center}

\section{Theoretical Analysis Details}
\label{app:theory_details}

\subsection{Proof of Theorem~\ref{thm:approximation_error}}
\label{app:proof_thm1}

We adopt two conventions throughout this supplementary section.
\textbf{(i)}~The omitted transient remainder in Assumption~\ref{assum:true_dynamics} is non-growing ($\operatorname{Re}(\lambda_k)\le 0$ for $k>K$); the learned transient-weighted dynamics in Theorem~\ref{thm:generalization_bound} are evaluated over the historical window $[0,t_T]$, with growth controlled by the window factor $G_{\mathrm{trans}}(t_T)$, which is finite by Assumption~\ref{assum:lipschitz}.
\textbf{(ii)}~Since $\operatorname{Re}(\cdot)$ is $1$-Lipschitz ($\|\operatorname{Re}(u)-\operatorname{Re}(v)\|_2\le\|u-v\|_2$ for $u,v\in\mathbb{C}^m$), every bound derived for the complex expansion $\sum_k a_k e^{\lambda_k t}$ applies directly to the real-valued parameterization $\tilde{z}(t)=\operatorname{Re}\big(\sum_k a_k e^{\lambda_k t}\big)$; we therefore suppress $\operatorname{Re}(\cdot)$ in intermediate steps.

\medskip
\noindent\textbf{Restatement of Theorem~\ref{thm:approximation_error} (Mode Approximation Bound).}
Under Assumption~\ref{assum:true_dynamics}, let $\hat{z}(t)=\sum_{k=1}^{K}\hat{a}_k e^{\hat{\lambda}_k t}$. Then for any horizon $t_s>0$:
\begin{equation*}
    \sup_{t\in[0,t_s]}\|z(t)-\hat{z}(t)\|_2
    \le
    \sum_{k=1}^{K}\|a_k-\hat{a}_k\|_2\,e^{\hat{\rho}_k t_s}
    +C_a\sum_{k=1}^{K}|\lambda_k-\hat{\lambda}_k|\,t_s\,e^{\varrho_k t_s}
    +\sum_{k=K+1}^{K^*}\|a_k\|_2,
\end{equation*}
where $\hat{\rho}_k=\max(\operatorname{Re}(\hat{\lambda}_k),0)$, $\varrho_k=\max(\operatorname{Re}(\lambda_k),\operatorname{Re}(\hat{\lambda}_k),0)$, and $C_a=\max_{1\le k\le K}\|a_k\|_2$.

\begin{proof}
Adding and subtracting $\sum_{k=1}^K a_k e^{\hat{\lambda}_k t}$ decomposes the error as:
\begin{equation}
    z(t)-\hat{z}(t)
    =\underbrace{\sum_{k=1}^{K}(a_k-\hat{a}_k)e^{\hat{\lambda}_k t}}_{E_{\mathrm{amp}}(t)}
    +\underbrace{\sum_{k=1}^{K}a_k\big(e^{\lambda_k t}-e^{\hat{\lambda}_k t}\big)}_{E_{\mathrm{est}}(t)}
    +\underbrace{\sum_{k=K+1}^{K^*}a_k\,e^{\lambda_k t}}_{E_{\mathrm{trunc}}(t)}.
    \label{eq:app_thm1_decomp}
\end{equation}
The triangle inequality gives $\|z(t)-\hat{z}(t)\|_2\le\|E_{\mathrm{amp}}(t)\|_2+\|E_{\mathrm{est}}(t)\|_2+\|E_{\mathrm{trunc}}(t)\|_2$.

\medskip
\noindent\emph{Step 1: Amplitude estimation error.}
For $t\in[0,t_s]$:
\begin{equation}
    \|E_{\mathrm{amp}}(t)\|_2
    \le\sum_{k=1}^{K}\|a_k-\hat{a}_k\|_2\,\big|e^{\hat{\lambda}_k t}\big|
    =\sum_{k=1}^{K}\|a_k-\hat{a}_k\|_2\,e^{\operatorname{Re}(\hat{\lambda}_k)t}
    \le\sum_{k=1}^{K}\|a_k-\hat{a}_k\|_2\,e^{\hat{\rho}_k t_s},
    \label{eq:app_thm1_amp}
\end{equation}
where the last inequality uses $\hat{\rho}_k\ge\operatorname{Re}(\hat{\lambda}_k)$ and $t\le t_s$.

\medskip
\noindent\emph{Step 2: Eigenvalue estimation error.}
Fix $t\in[0,t_s]$ and $k\in\{1,\dots,K\}$. By the fundamental theorem of calculus applied to $s\mapsto e^{(\hat{\lambda}_k+s(\lambda_k-\hat{\lambda}_k))t}$ on $[0,1]$:
\begin{equation}
    e^{\lambda_k t}-e^{\hat{\lambda}_k t}
    =(\lambda_k-\hat{\lambda}_k)\,t\int_0^1
    e^{(\hat{\lambda}_k+s(\lambda_k-\hat{\lambda}_k))t}\,ds.
    \label{eq:app_thm1_segment}
\end{equation}
Since $|e^{\mu t}|=e^{\operatorname{Re}(\mu)t}$ and $\operatorname{Re}(\hat{\lambda}_k+s(\lambda_k-\hat{\lambda}_k))$ is linear in $s$, its maximum over $[0,1]$ is attained at an endpoint:
\begin{equation}
    \big|e^{\lambda_k t}-e^{\hat{\lambda}_k t}\big|
    \le|\lambda_k-\hat{\lambda}_k|\,t\,e^{\max(\operatorname{Re}(\lambda_k),\operatorname{Re}(\hat{\lambda}_k))t}
    \le|\lambda_k-\hat{\lambda}_k|\,t\,e^{\varrho_k t},
    \label{eq:app_thm1_mvt}
\end{equation}
where $\varrho_k=\max(\operatorname{Re}(\lambda_k),\operatorname{Re}(\hat{\lambda}_k),0)\ge 0$. Because $\varrho_k\ge 0$, the map $t\mapsto t\,e^{\varrho_k t}$ is nondecreasing on $[0,t_s]$, so the supremum is attained at $t=t_s$. Combining with $C_a=\max_{k\le K}\|a_k\|_2$:
\begin{equation}
    \sup_{t\in[0,t_s]}\|E_{\mathrm{est}}(t)\|_2
    \le\sum_{k=1}^{K}\|a_k\|_2\,\big|e^{\lambda_k t}-e^{\hat{\lambda}_k t}\big|
    \le C_a\sum_{k=1}^{K}|\lambda_k-\hat{\lambda}_k|\,t_s\,e^{\varrho_k t_s}.
    \label{eq:app_thm1_est}
\end{equation}

\medskip
\noindent\emph{Step 3: Modal truncation error.}
By Assumption~\ref{assum:true_dynamics}, $\operatorname{Re}(\lambda_k)\le 0$ for $k>K$, so $|e^{\lambda_k t}|=e^{\operatorname{Re}(\lambda_k)t}\le 1$ for all $t\ge 0$. Hence:
\begin{equation}
    \sup_{t\in[0,t_s]}\|E_{\mathrm{trunc}}(t)\|_2
    \le\sum_{k=K+1}^{K^*}\|a_k\|_2\,\big|e^{\lambda_k t}\big|
    \le\sum_{k=K+1}^{K^*}\|a_k\|_2.
    \label{eq:app_thm1_trunc}
\end{equation}
Summing Eqs.~\eqref{eq:app_thm1_amp}, \eqref{eq:app_thm1_est}, and~\eqref{eq:app_thm1_trunc} completes the proof.
\end{proof}

\subsection{Proof of Theorem~\ref{thm:generalization_bound}}
\label{app:proof_thm2}

\begin{theorem}[Restatement of Theorem~\ref{thm:generalization_bound}]
Under Assumptions~\ref{assum:true_dynamics}--\ref{assum:lipschitz}, the future-domain excess risk at $t_s>t_T$ satisfies:
\begin{equation}
\mathcal{E}_{t_s}\le L_\ell L_g L_{\mathrm{dec}}\big(E_{\mathrm{dom}}(t_s)+E_{\mathrm{trans}}(t_s)\big),
\label{eq:app_thm2_risk}
\end{equation}
where $E_{\mathrm{dom}}(t_s)=\|\hat{z}_{\mathrm{dom}}(t_s)-z^\star_{\mathrm{dom}}(t_s)\|_2$ and $E_{\mathrm{trans}}(t_s)=\|\hat{z}_{\mathrm{trans}}(t_T)-z^\star_{\mathrm{trans}}(t_s)\|_2$. With $\Delta_s=t_s-t_T$, $\hat{\rho}_k=\max(\operatorname{Re}(\hat{\lambda}_k),0)$, $\varrho_k=\max(\operatorname{Re}(\lambda_k^\star),\operatorname{Re}(\hat{\lambda}_k),0)$:
\begin{equation}
E_{\mathrm{dom}}(t_s)\le\sum_{k=1}^{K}\|\hat{b}_k-b_k^\star\|_2\,e^{\hat{\rho}_k\Delta_s}
+B_{\mathrm{dom}}\sum_{k=1}^{K}|\hat{\lambda}_k-\lambda_k^\star|\,\Delta_s\,e^{\varrho_k\Delta_s},
\label{eq:app_thm2_dom_bound}
\end{equation}
and the transient holding error satisfies:
\begin{equation}
E_{\mathrm{trans}}(t_s)\le B_{\mathrm{trans}}+G_{\mathrm{trans}}(t_T)\sqrt{\mathcal{R}_{\mathrm{spec}}},
\label{eq:app_thm2_trans_bound}
\end{equation}
where $G_{\mathrm{trans}}(t_T)=\big(\sum_{k=1}^K w_k^{\mathrm{trans}}e^{2\hat{\rho}_k t_T}\big)^{1/2}$.
\end{theorem}

\begin{proof}
By Assumption~\ref{assum:lipschitz}, $\theta^\star(t_s)=\phi^{-1}(z^\star(t_s))$ and $\hat{\theta}(t_s)=\phi^{-1}(\hat{z}(t_s))$.

\medskip
\noindent\textbf{Step 1: Excess risk to latent error.}
By $L_\ell$-Lipschitz continuity of $\ell$ and $L_g$-Lipschitz continuity of $g(X;\cdot)$:
\begin{equation}
\mathcal{E}_{t_s}\le L_\ell L_g\|\hat{\theta}(t_s)-\theta^\star(t_s)\|_2.
\label{eq:app_thm2_parameter}
\end{equation}
By $L_{\mathrm{dec}}$-Lipschitz continuity of $\phi^{-1}$:
\begin{equation}
\|\hat{\theta}(t_s)-\theta^\star(t_s)\|_2
=\|\phi^{-1}(\hat{z}(t_s))-\phi^{-1}(z^\star(t_s))\|_2
\le L_{\mathrm{dec}}\|\hat{z}(t_s)-z^\star(t_s)\|_2.
\label{eq:app_thm2_decoder}
\end{equation}
Combining Eqs.~\eqref{eq:app_thm2_parameter}--\eqref{eq:app_thm2_decoder}:
\begin{equation}
\mathcal{E}_{t_s}\le L_\ell L_g L_{\mathrm{dec}}\|\hat{z}(t_s)-z^\star(t_s)\|_2.
\label{eq:app_thm2_lip}
\end{equation}

\medskip
\noindent\textbf{Step 2: Latent error decomposition.}
By the inference rule $\hat{z}(t_s)=\hat{z}_{\mathrm{dom}}(t_s)+\hat{z}_{\mathrm{trans}}(t_T)$ (Eq.~\eqref{eq:infer-z}) and the decomposition $z^\star(t_s)=z^\star_{\mathrm{dom}}(t_s)+z^\star_{\mathrm{trans}}(t_s)$:
\begin{equation}
\|\hat{z}(t_s)-z^\star(t_s)\|_2
\le\underbrace{\|\hat{z}_{\mathrm{dom}}(t_s)-z^\star_{\mathrm{dom}}(t_s)\|_2}_{E_{\mathrm{dom}}(t_s)}
+\underbrace{\|\hat{z}_{\mathrm{trans}}(t_T)-z^\star_{\mathrm{trans}}(t_s)\|_2}_{E_{\mathrm{trans}}(t_s)}.
\label{eq:app_thm2_split}
\end{equation}
Substituting into Eq.~\eqref{eq:app_thm2_lip} proves Eq.~\eqref{eq:app_thm2_risk}.

\medskip
\noindent\textbf{Step 3: Dominant extrapolation error.}
Setting $\tau=\Delta_s$ in the anchored representations of Assumption~\ref{assum:lipschitz}:
\begin{equation}
z^\star_{\mathrm{dom}}(t_s)=\operatorname{Re}\!\Big(\sum_{k=1}^K b_k^\star e^{\lambda_k^\star\Delta_s}\Big),\qquad
\hat{z}_{\mathrm{dom}}(t_s)=\operatorname{Re}\!\Big(\sum_{k=1}^K\hat{b}_k e^{\hat{\lambda}_k\Delta_s}\Big).
\label{eq:app_learned_dom}
\end{equation}
By the non-expansive property of $\operatorname{Re}(\cdot)$:
\begin{equation}
E_{\mathrm{dom}}(t_s)\le\Big\|\sum_{k=1}^K\big(\hat{b}_k e^{\hat{\lambda}_k\Delta_s}-b_k^\star e^{\lambda_k^\star\Delta_s}\big)\Big\|_2.
\label{eq:app_dom_complex}
\end{equation}
Adding and subtracting $b_k^\star e^{\hat{\lambda}_k\Delta_s}$:
\begin{equation}
\hat{b}_k e^{\hat{\lambda}_k\Delta_s}-b_k^\star e^{\lambda_k^\star\Delta_s}
=(\hat{b}_k-b_k^\star)e^{\hat{\lambda}_k\Delta_s}+b_k^\star\big(e^{\hat{\lambda}_k\Delta_s}-e^{\lambda_k^\star\Delta_s}\big).
\label{eq:app_dom_decomposition}
\end{equation}
The triangle inequality yields:
\begin{equation}
E_{\mathrm{dom}}(t_s)\le
\sum_{k=1}^K\|\hat{b}_k-b_k^\star\|_2\,|e^{\hat{\lambda}_k\Delta_s}|
+\sum_{k=1}^K\|b_k^\star\|_2\,|e^{\hat{\lambda}_k\Delta_s}-e^{\lambda_k^\star\Delta_s}|.
\label{eq:app_dom_split}
\end{equation}
For the first sum, $|e^{\hat{\lambda}_k\Delta_s}|=e^{\operatorname{Re}(\hat{\lambda}_k)\Delta_s}\le e^{\hat{\rho}_k\Delta_s}$. For the second sum, applying the integral identity to $u\mapsto e^{[\lambda_k^\star+u(\hat{\lambda}_k-\lambda_k^\star)]\Delta_s}$ on $[0,1]$:
\begin{equation}
e^{\hat{\lambda}_k\Delta_s}-e^{\lambda_k^\star\Delta_s}
=(\hat{\lambda}_k-\lambda_k^\star)\Delta_s\int_0^1 e^{[\lambda_k^\star+u(\hat{\lambda}_k-\lambda_k^\star)]\Delta_s}\,du.
\label{eq:app_exp_integral}
\end{equation}
Since $\operatorname{Re}(\lambda_k^\star+u(\hat{\lambda}_k-\lambda_k^\star))$ is linear in $u$, its maximum over $[0,1]$ is attained at an endpoint. Therefore:
\begin{equation}
|e^{\hat{\lambda}_k\Delta_s}-e^{\lambda_k^\star\Delta_s}|
\le|\hat{\lambda}_k-\lambda_k^\star|\,\Delta_s\,e^{\max(\operatorname{Re}(\lambda_k^\star),\operatorname{Re}(\hat{\lambda}_k))\Delta_s}
\le|\hat{\lambda}_k-\lambda_k^\star|\,\Delta_s\,e^{\varrho_k\Delta_s}.
\label{eq:app_exp_bound}
\end{equation}
Substituting $\max_k\|b_k^\star\|_2\le B_{\mathrm{dom}}$ and Eq.~\eqref{eq:app_exp_bound} into Eq.~\eqref{eq:app_dom_split} establishes Eq.~\eqref{eq:app_thm2_dom_bound}.

\medskip
\noindent\textbf{Step 4: Transient holding error.}
The triangle inequality gives:
\begin{equation}
E_{\mathrm{trans}}(t_s)\le\|\hat{z}_{\mathrm{trans}}(t_T)\|_2+\|z^\star_{\mathrm{trans}}(t_s)\|_2.
\label{eq:app_thm2_trans_split}
\end{equation}
By Assumption~\ref{assum:lipschitz}, $\|z^\star_{\mathrm{trans}}(t_s)\|_2\le B_{\mathrm{trans}}$.

For the learned transient contribution, the non-expansive property and the triangle inequality give:
\begin{equation}
\|\hat{z}_{\mathrm{trans}}(t_T)\|_2
=\Big\|\operatorname{Re}\!\Big(\sum_{k=1}^K w_k^{\mathrm{trans}}\hat{a}_k e^{\hat{\lambda}_k t_T}\Big)\Big\|_2
\le\sum_{k=1}^K w_k^{\mathrm{trans}}\|\hat{a}_k\|_2\,e^{\hat{\rho}_k t_T},
\label{eq:app_trans_triangle}
\end{equation}
where $|e^{\hat{\lambda}_k t_T}|=e^{\operatorname{Re}(\hat{\lambda}_k)t_T}\le e^{\hat{\rho}_k t_T}$. Writing $w_k^{\mathrm{trans}}=\sqrt{w_k^{\mathrm{trans}}}\cdot\sqrt{w_k^{\mathrm{trans}}}$ and applying Cauchy--Schwarz:
\begin{equation}
\|\hat{z}_{\mathrm{trans}}(t_T)\|_2
\le\Big(\sum_{k=1}^K w_k^{\mathrm{trans}}e^{2\hat{\rho}_k t_T}\Big)^{1/2}
\Big(\sum_{k=1}^K w_k^{\mathrm{trans}}\|\hat{a}_k\|_2^2\Big)^{1/2}.
\label{eq:app_trans_cs}
\end{equation}
The first factor is $G_{\mathrm{trans}}(t_T)$, which is finite by Assumption~\ref{assum:lipschitz}. For the second factor, since $1-\operatorname{Var}_{k=1}^K(w_k^{\mathrm{dom}})\ge 0$, the definition $\mathcal{R}_{\mathrm{spec}}=\sum_k w_k^{\mathrm{trans}}\|\hat{a}_k\|_2^2+(1-\operatorname{Var}(w_k^{\mathrm{dom}}))$ (Eq.~\eqref{eq:rspec}) implies:
\begin{equation}
\sum_{k=1}^K w_k^{\mathrm{trans}}\|\hat{a}_k\|_2^2\le\mathcal{R}_{\mathrm{spec}}.
\label{eq:app_rspec_surrogate}
\end{equation}
Substituting into Eq.~\eqref{eq:app_trans_cs}:
\begin{equation}
\|\hat{z}_{\mathrm{trans}}(t_T)\|_2\le G_{\mathrm{trans}}(t_T)\sqrt{\mathcal{R}_{\mathrm{spec}}}.
\label{eq:app_thm2_trans}
\end{equation}
Combining Eqs.~\eqref{eq:app_thm2_trans_split} and~\eqref{eq:app_thm2_trans} proves Eq.~\eqref{eq:app_thm2_trans_bound}. Steps 1--4 together complete the proof.
\end{proof}

\subsection{Proof of Proposition~\ref{prop:energy}}
\label{app:proof_prop2}

\begin{proof}
Define the cumulative temporal energy of mode $k$ over $[0,T_{\max}]$:
\begin{equation}
E_k=\int_0^{T_{\max}}\|a_k e^{\lambda_k t}\|_2^2\,dt.
\end{equation}
Since $\|\operatorname{Re}(u)\|_2\le\|u\|_2$, this upper-bounds the energy of the real-valued trajectory $\operatorname{Re}(a_k e^{\lambda_k t})$.

Given $\lambda_k=\sigma_k+j\omega_k$ with $\sigma_k<0$:
\begin{equation}
E_k=\|a_k\|_2^2\int_0^{T_{\max}}|e^{(\sigma_k+j\omega_k)t}|^2\,dt
=\|a_k\|_2^2\int_0^{T_{\max}}e^{2\sigma_k t}\,dt
=\|a_k\|_2^2\cdot\frac{1-e^{2\sigma_k T_{\max}}}{-2\sigma_k}.
\label{eq:app_prop2}
\end{equation}
Because $\sigma_k<0$ and $T_{\max}>0$, we have $0<e^{2\sigma_k T_{\max}}<1$, so $1-e^{2\sigma_k T_{\max}}\le 1$. Therefore:
\begin{equation}
E_k\le\frac{\|a_k\|_2^2}{-2\sigma_k}=\frac{\|a_k\|_2^2}{2|\sigma_k|},
\end{equation}
which establishes Eq.~\eqref{eq:energy-bound}.
\end{proof}

\section{Experimental Details}
\label{app:experiments}

\subsection{Datasets}
\label{app:datasets}
\subsubsection{CTDG Benchmarks}
We evaluate FreKoo++ on six continuous temporal domain generalization (CTDG) benchmarks characterized by asynchronous arrivals and irregularly sampled temporal gaps. For all CTDG datasets, timestamps are normalized to $[0,1]$ while preserving their relative temporal intervals. Following the standard CTDG setup, the chronologically earliest 70\% of domains serve as historical source training data, while the remaining latest 30\% are reserved as unseen target domains for future extrapolation. The '-C' suffix denotes the continuous, irregularly sampled variant to distinguish it from its discrete counterpart.

\textbf{Rotated 2-Moons-C.} This synthetic benchmark constructs a continuous variant of the two-moons problem to model rotational concept drift. It contains 1,000 two-dimensional samples per domain across two classes formed by upper and lower moon-shaped clusters. We sample 50 timestamps from $[0,50]$ to construct 50 sequential domains. Domains 1--35 (35 source domains) are used for training, while domains 36--50 (15 target domains) serve as unseen test sets. Concept drift is induced by continuously rotating the dataset counter-clockwise at a constant rate of $18^\circ$ per unit time.

\textbf{Rotated MNIST-C.} We extend the MNIST dataset~\cite{deng2012mnist} to continuous temporal shifts under irregular sampling. Fifty timestamps are sampled from $[0,50]$, where each timestamp defines a domain containing 1,000 randomly selected digit images, yielding a total of 50 temporal domains. The first 35 rotated domains serve as training source data, while the remaining 15 domains are reserved for testing. Analogous to Rotated 2-Moons-C, progressive counter-clockwise rotation at $18^\circ$ per unit time induces continuous visual domain drift.

\textbf{Twitter.} This benchmark evaluates influenza-risk prediction using streaming social media signals~\cite{zhao2017feature,bai2023sign}. Domains are constructed from tweet streams collected at 50 arbitrary starting timestamps across the 2010--2014 flu seasons, with each domain spanning a 7-day temporal window. We partition the data into 50 chronological domains, using the first 35 as training sources and the final 15 for testing. Features consist of disease-related term frequency statistics validated against CDC Influenza-Like Illness reports, where concept drift reflects both seasonal flu cycles and evolving user tweeting patterns.

\textbf{Yearbook.} Derived from the Wild-Time benchmark~\cite{yao2022wild}, this dataset contains frontal yearbook portraits collected from 1930 to 2013 across 128 high schools for gender classification. We sample 40 non-consecutive years across the 84-year span to construct 40 irregularly spaced domains, training on the first 28 source domains and evaluating on the final 12 future target domains. Long-term visual drift naturally emerges from non-uniform temporal intervals, reflecting shifts in fashion trends, photographic styles, and social context.

\textbf{Cyclone.} This benchmark addresses a regression task mapping satellite imagery of tropical cyclones to wind intensity~\cite{chen2018rotation}. Each cyclone event represents a distinct domain indexed by its occurrence date. Following Koodos~\cite{cai2024continuous}, we extract West Pacific cyclone records from 2014 to 2016 to form 72 continuous temporal domains. We use the first 50 domains as training sources and the remaining 22 domains for future wind intensity prediction. Concept drift naturally arises from seasonal meteorological variations and event-driven irregular arrivals.

\textbf{House-C.} This real-estate regression benchmark predicts housing prices from property features based on transaction records from 2013 to 2019. Following the CTDG protocol, we sample 40 non-overlapping one-month temporal windows to create 40 chronological domains, using the first 28 domains for training and the remaining 12 domains for testing. Concept drift stems naturally from evolving local real-estate market conditions and non-uniform sampling intervals.

\subsubsection{DTDG benchmarks}
To verify backward compatibility under discrete settings, we evaluate FreKoo++ across seven temporal benchmarks following the experimental setup and evaluation protocol established in FreKoo~\cite{yu2025learning}. Following standard DTDG practice, historical chronological domains serve as source training data, while the immediately succeeding unseen future domain is reserved as the target domain. This benchmark suite encompasses both synthetic and complex real-world temporal data streams with diverse drift patterns.

\textbf{Rotated 2-Moons-D.} This benchmark adapts the 2-Moons dataset to model concept drift via rotation. It contains 1,800  2-dimensional samples across two classes, divided into 10 sequential domains. Each domain is rotated 18° counter-clockwise relative to the previous one. We train on domains 0-8 and test on domain 9, where the drift is caused by the incremental rotation.
    
\textbf{Rotated MNIST-D.} We randomly sampled 1000 instances from the MNIST dataset and constructed a total of five domains by successively rotating them counter-clockwise by 15°, analogous to the Rotated 2-Moons-D setup. The first four rotated domains are used for training, while the fifth domain serves as the test set, creating incremental drift induced by progressive rotation transformations.
    
\textbf{Online News Popularity (ONP)\footnote{https://archive.ics.uci.edu/dataset/332/online+news+popularity}.} This dataset aggregates heterogeneous features of articles published by Mashable over two years, aiming to predict social media shares (popularity). It comprises 39,797 samples with 58 features, where concept drift is characterized by temporal shifts in popularity patterns. We partition the data into 6 time-ordered domains, using the first five for training and the last for testing. The dataset undergoes slight real-world concept drift over the observed time period.

\textbf{Shuttle\footnote{https://archive.ics.uci.edu/dataset/148/statlog+shuttle}.} The Shuttle dataset contains 58,000 instances of multi-class flight status classification under severe class imbalance. It is partitioned into 8 time-stamped domains using a chronological split: domains spanning timestamps 30-70 serve as training data, while the most recent period (70–80) is reserved for testing. The dataset also has real-world concept drifts over the observed time period.

\textbf{Electrical Demand\footnote{https://web.archive.org/web/20191121102533/http://www.inescporto.pt/\~jgama/ales/ales\_5.html}.} This dataset records electricity demand in a province, addressing a binary classification task to predict whether 30-minute demand exceeds or falls below the daily average for that time period. After removing instances with missing values, it contains 28,222 samples with 8 features. It is partitioned into 30 two-week chronological domains, with the first 29 used for training and the 30th for testing. Seasonal variations in demand induce concept drift, making this a real-world benchmark capturing both periodic and incremental drift patterns.
    
\textbf{HousePrices-D\footnote{https://www.kaggle.com/datasets/htagholdings/property-sales}.} This dataset comprises housing price records from 2013 to 2019 for the regression task to predict property prices based on feature values. We treat each calendar year as a distinct domain, using 2013–2018 data for training and the final (2019) domain for testing. Concept drift emerges naturally from temporal economic shifts and market fluctuations over the years.

\textbf{Appliances Energy Prediction\footnote{https://archive.ics.uci.edu/dataset/374/appliances+energy+prediction}.} This dataset addresses regression modeling for predicting appliance energy consumption in a low-energy building. Comprising 10-minute sensor readings over 4.5 months in 2016, it is partitioned into 9 chronological domains. We train on the first eight domains and evaluate on the final (most recent) ninth domain, with concept drift arising from temporal shifts in energy usage patterns across the observation period.

\subsection{Baselines}
\label{app:baselines}
\subsubsection{CTDG Baselines}
Following the continuous temporal domain generalization protocol in Koodos~\cite{cai2024continuous}, we compare FreKoo++ against twelve representative CTDG baselines evaluated in Table~I, categorized into four technical paradigms based on how they handle temporal information:

\textbf{Time-agnostic baselines.} These methods ignore explicit temporal ordering across domains. \emph{Offline} aggregates all source domains into a single pool for empirical risk minimization. \emph{LastDomain} trains exclusively on the chronologically latest source domain, adopting a naive recency-based heuristic. \emph{IncFinetune} sequentially fine-tunes the network across chronological domains using a reduced learning rate.

\textbf{Domain generalization and continuous adaptation baselines.} \emph{IRM}~\cite{arjovsky2019invariant} and \emph{V-REx}~\cite{krueger2021out} penalize invariant predictor mismatches and risk variations across source domains, respectively. \emph{CIDA}~\cite{wang2020continuously} conditions the model on continuous temporal indices, aligning representations via adversarial domain alignment with a domain-aware discriminator.

\textbf{Temporal and parameter forecasting baselines.} \emph{TKNets}~\cite{zeng2024generalizing} leverages Koopman operators for discrete temporal transfer across evolving domains. \emph{DRAIN}~\cite{bai2023temporal} forecasts future neural network parameters via a discrete recurrent network trained on source parameter trajectories. To accommodate irregular sampling intervals, \emph{DRAIN-$\Delta t$}~\cite{cai2024continuous} replaces standard discrete recurrent units with continuous-time recurrent units~\cite{schirmer2022modeling}.

\textbf{Continuous-time dynamics baselines.} \emph{DeepODE}~\cite{cai2024continuous} maps model parameters into a latent space and evolves them via a Neural ODE solver~\cite{chen2018neural}. \emph{NeuralLio}~\cite{cai2025continuous} incorporates continuous-time Liouville operator dynamics to track evolving continuous distributions over time. \emph{Koodos}~\cite{cai2024continuous} models the coupled evolution of data and parameter states using a continuous-time Koopman formulation under structural constraints.

To ensure a fair and rigorous comparison, all baseline methods share identical task backbone architectures with FreKoo++ across each benchmark dataset. Experimental results are either directly cited from their original publications under matched evaluation protocols or reproduced using their official open-source implementations under unified experimental settings. To ensure statistical reliability, all reported metrics represent the mean and standard deviation ($\mathrm{mean} \pm \mathrm{std}$) computed over five independent runs with different random seeds.

\subsubsection{DTDG baselines}
For the DTDG compatibility evaluation, we follow the FreKoo
protocol~\cite{yu2025learning}. The comparison includes time-agnostic baselines
\emph{Offline}, \emph{LastDomain}, and \emph{IncFinetune}; continuous domain
adaptation baselines \emph{CDOT}~\cite{ortiz2019cdot} and
\emph{CIDA}~\cite{wang2020continuously}; discrete TDG methods
\emph{GI}~\cite{nasery2021training},
\emph{LSSAE}~\cite{qin2022generalizing},
\emph{Foresee}~\cite{zeng2023foresee}, and
\emph{DRAIN}~\cite{bai2023temporal}; and spectral-Koopman baselines
\emph{TKNets}~\cite{zeng2024generalizing}, \emph{Koodos}~\cite{cai2024continuous},
and \emph{FreKoo}~\cite{yu2025learning}.

\section{Implementation Details}
\label{app:implementation}
\subsection{Network Architectures}
\label{app:architectures}
To ensure rigorous comparability with state-of-the-art baselines, FreKoo++ integrates three primary modules: a dataset-specific task predictive model ($\theta \in \mathbb{R}^D$), a symmetric four-layer latent autoencoder ($\phi, \phi^{-1}$), and the continuous Koopman spectral-dynamical module operating in $\mathbb{R}^m$. Detailed network architectures for the task backbones and the uniform autoencoder setup across all benchmark datasets are summarized in Table~\ref{tab:model_config}.
\begin{table}[t]
\centering
\caption{Detailed network architectures for the task predictive model and the uniform latent autoencoder ($\phi, \phi^{-1}$) across all benchmark datasets. All autoencoder hidden layers employ ReLU activations.}
\label{tab:model_config}
\begin{tabular}{lm{8cm}m{5cm}}
\toprule
Dataset & Predictive model ($\theta \in \mathbb{R}^D$) & Latent encoder-decoder ($\phi, \phi^{-1}$) \\
\midrule
2-Moons-C 
  & MLP: [2, 50, 50, 1] + ReLU + Sigmoid 
  & \multirow{6}{4.8cm}{\centering 
      $\phi$: [$D \to 1024 \to 512 \to 128 \to m$] \\[5pt]
      $\phi^{-1}$: [$m \to 128 \to 512 \to 1024 \to D$]} \\
Rot-MNIST-C 
  & $3 \times$ [Conv(32/32/64, $k{=}3$) + ReLU + MaxPool($2$)] + MLP: [576, 128, 10] + LogSoftmax 
  & \\
Twitter 
  & Embedding Network + MLP [32, 128, 32, 1] + ReLU + Sigmoid 
  & \\
Yearbook 
  & $3 \times$ [Conv(32/32/64, $k{=}3$) + ReLU + MaxPool($2$)] + MLP: [1024, 128, 32, 1] + ReLU + Sigmoid 
  & \\
Cyclone 
  & $4 \times$ [Conv(32/32/64/64, $k{=}3$) + ReLU + MaxPool($2$)] + MLP: [1024, 128, 32, 1] + ReLU 
  & \\
House 
  & MLP: [30, 400, 400, 1] + ReLU 
  & \\
\bottomrule
\end{tabular}
\end{table}
\subsection{Training Details and Hyperparameters}
We optimize FreKoo++ end-to-end using the Adam optimizer across all benchmark datasets. Optimization is parameterized with three distinct learning rate groups: $\eta_1$ for the predictive task model, $\eta_2$ for the latent encoder and decoder, and $\eta_3$ for the continuous Koopman spectral module. The overall joint loss objective is regulated by four hyperparameter weights: $\alpha$ for parameter reconstruction ($\mathcal{L}_{\mathrm{rec}}^{(C)}$), $\beta$ for modal trajectory fitting ($\mathcal{L}_{\mathrm{fit}}^{(C)}$), $\gamma$ for transient spectral regularization ($\mathcal{R}_{\mathrm{spec}}$), and $\delta$ for stability regularization ($\mathcal{R}_{\mathrm{stab}}$).

For the \textbf{2-Moons-C} dataset, the model is trained for $300$ epochs with learning rates $\eta_1 = 1 \times 10^{-2}$ and $\eta_2 = \eta_3 = 1 \times 10^{-3}$, regulated by loss weights $\alpha = 100$, $\beta = 1$, $\gamma = 1$, and $\delta = 10$. The \textbf{Rot-MNIST-C} configuration employs $400$ training epochs with a uniform learning rate $\eta_1 = \eta_2 = \eta_3 = 1 \times 10^{-3}$ across all modules, governed by $\alpha = 100$, $\beta = 10$, $\gamma = 1$, and $\delta = 10$. For \textbf{Twitter}, training spans $200$ epochs using $\eta_1 = 1 \times 10^{-2}$ for the task predictor and $\eta_2 = \eta_3 = 1 \times 10^{-3}$ for the autoencoder and spectral dynamics modules, combined with loss weights $\alpha = 10$, $\beta = 10$, $\gamma = 1$, and $\delta = 10$. The \textbf{Yearbook} dataset is optimized over $800$ epochs with a uniform learning rate of $1 \times 10^{-3}$ across all components, accompanied by loss coefficients $\alpha = 100$, $\beta = 10$, $\gamma = 1$, and $\delta = 10$. For \textbf{Cyclone}, we train the model for $500$ epochs using a uniform learning rate $\eta_1 = \eta_2 = \eta_3 = 1 \times 10^{-3}$, with objective weights set to $\alpha = 100$, $\beta = 100$, $\gamma = 100$, and $\delta = 10$. Finally, the \textbf{House} dataset is trained over $600$ epochs maintaining a uniform learning rate of $1 \times 10^{-3}$ across all modules, regulated by loss coefficients $\alpha = 10$, $\beta = 10$, $\gamma = 1$, and $\delta = 10$.

\section{Supplementary Experiments}
\label{app:supp_experiments}

\begin{figure*}[t]
\centering
\subfigure[Step 1, err 4.8\%]{
    \includegraphics[width=0.18\linewidth]{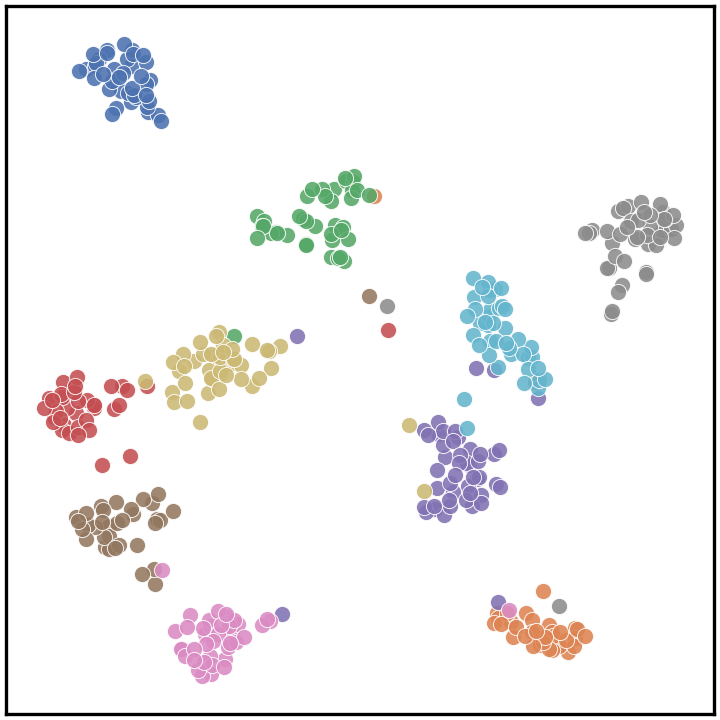}
    }
\subfigure[Step 2, err 3.3\%]{
    \includegraphics[width=0.18\linewidth]{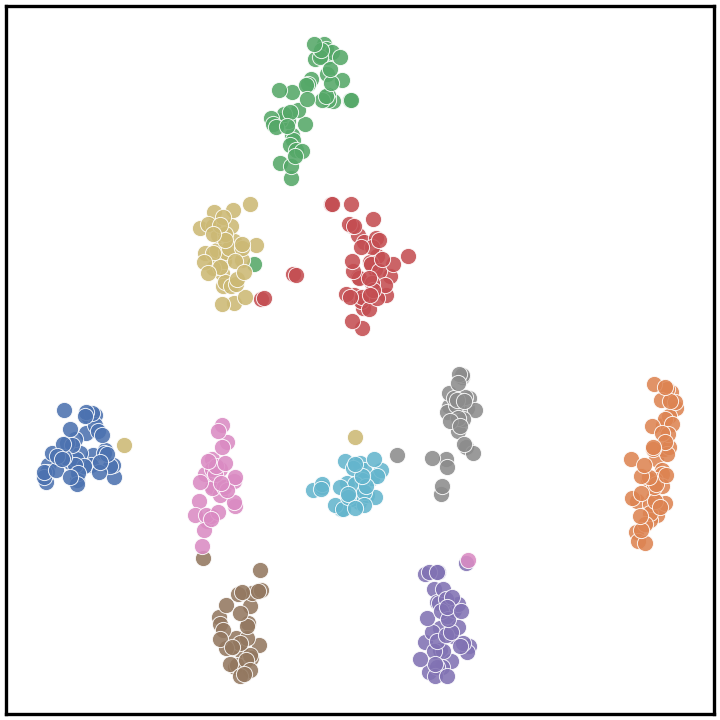}
    }
\subfigure[Step 3, err 3.8\%]{
    \includegraphics[width=0.18\linewidth]{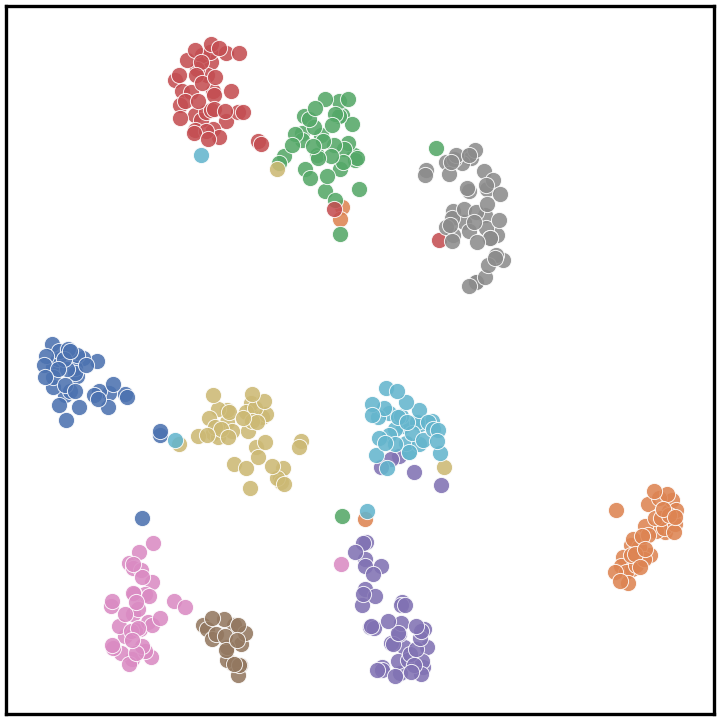}
    }
\subfigure[Step 4, err 2.7\%]{
    \includegraphics[width=0.18\linewidth]{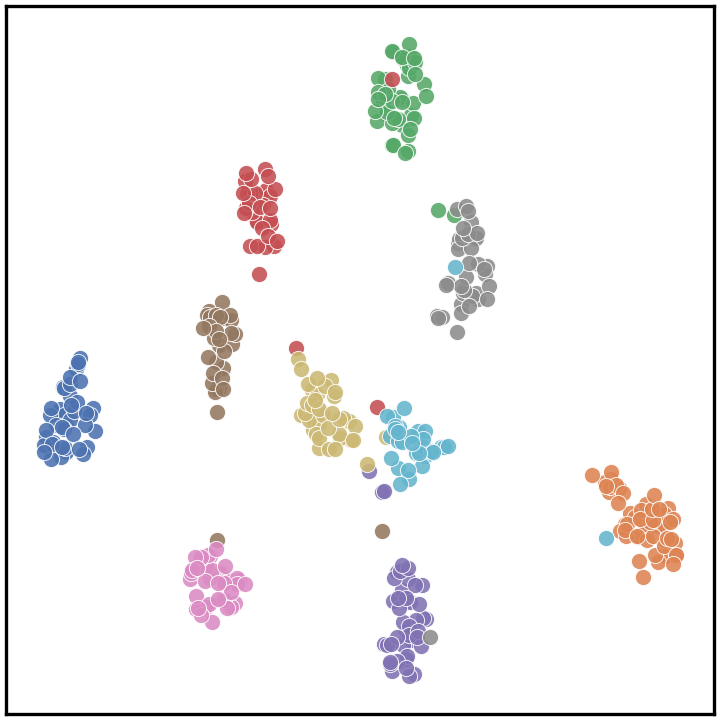}
    }
\subfigure[Step 5, err 2.7\%]{
    \includegraphics[width=0.18\linewidth]{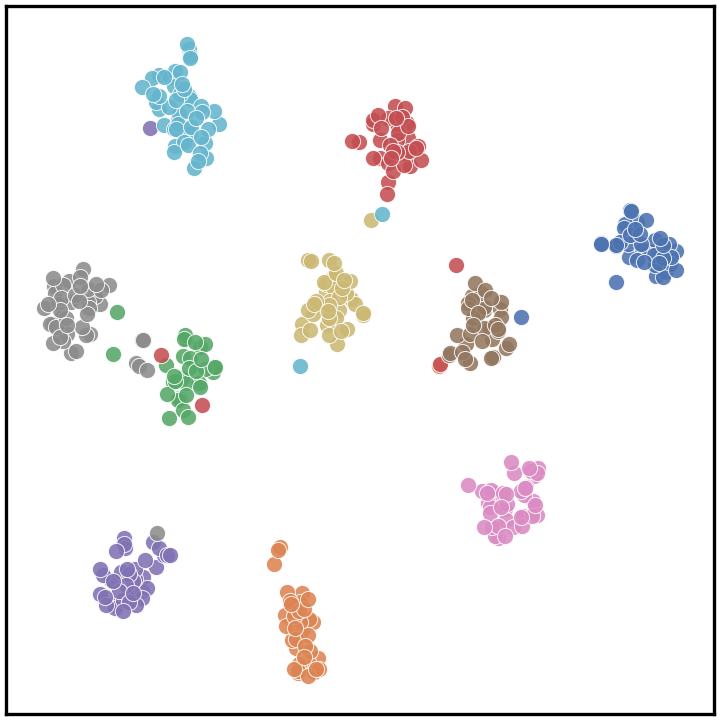}
    }
\\
\subfigure[Step 6, err 2.2\%]{
    \includegraphics[width=0.18\linewidth]{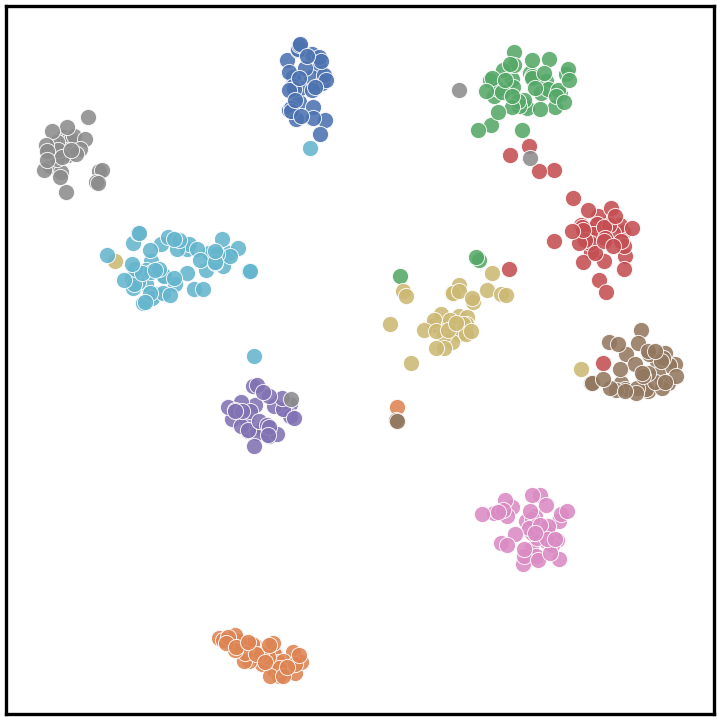}
    }
\subfigure[Step 7, err 1.4\%]{
    \includegraphics[width=0.18\linewidth]{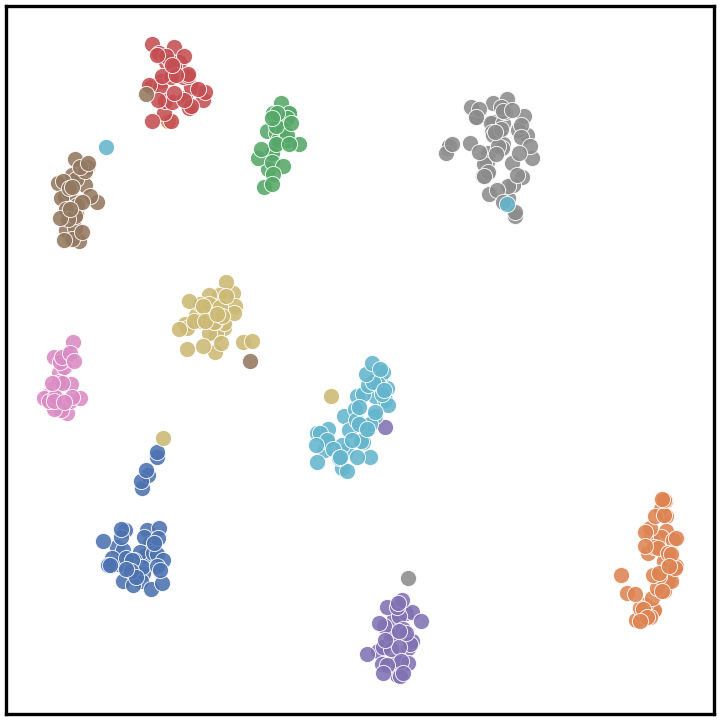}
    }
\subfigure[Step 8, err 2.5\%]{
    \includegraphics[width=0.18\linewidth]{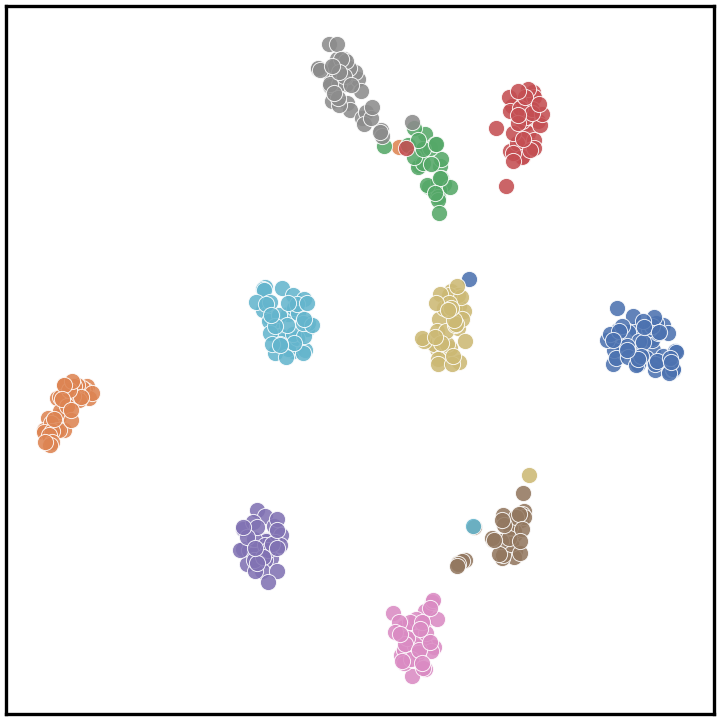}
    }
\subfigure[Step 9, err 1.9\%]{
    \includegraphics[width=0.18\linewidth]{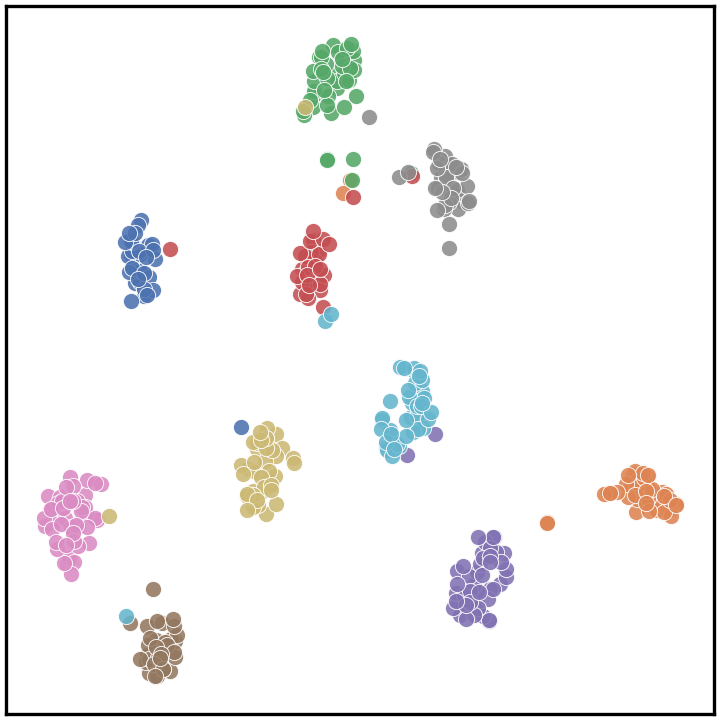}
    }
\subfigure[Step 10, err 1.7\%]{
    \includegraphics[width=0.18\linewidth]{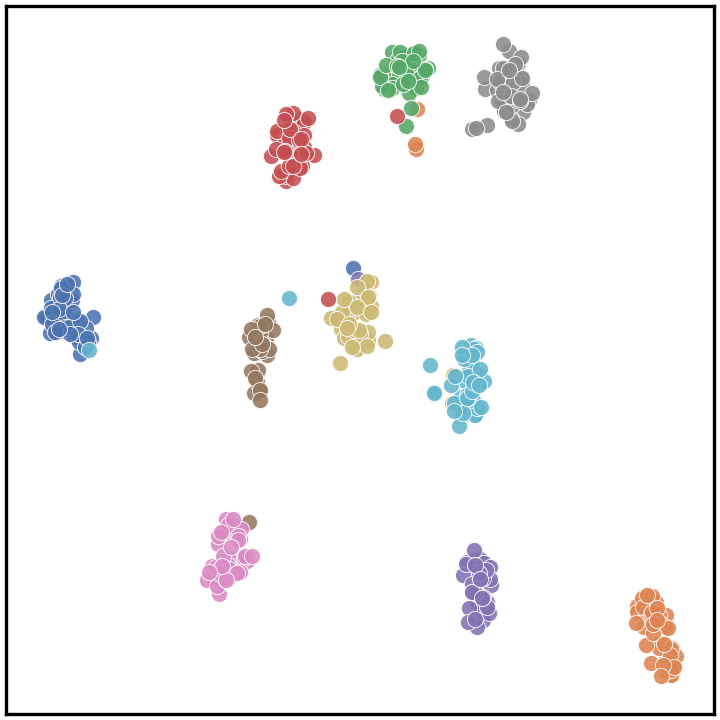}
    }
\\
\subfigure[Step 11, err 4.1\%]{
    \includegraphics[width=0.18\linewidth]{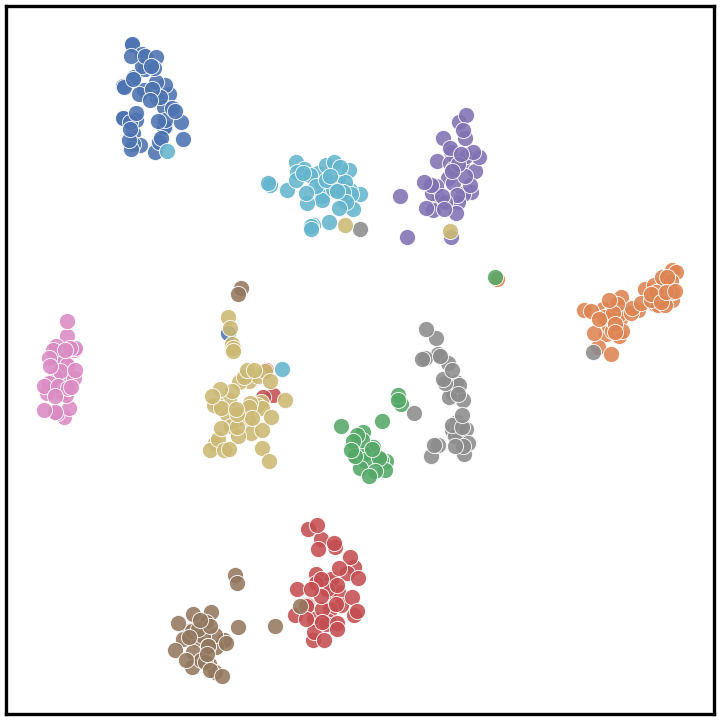}
    }
\subfigure[Step 12, err 2.0\%]{
    \includegraphics[width=0.18\linewidth]{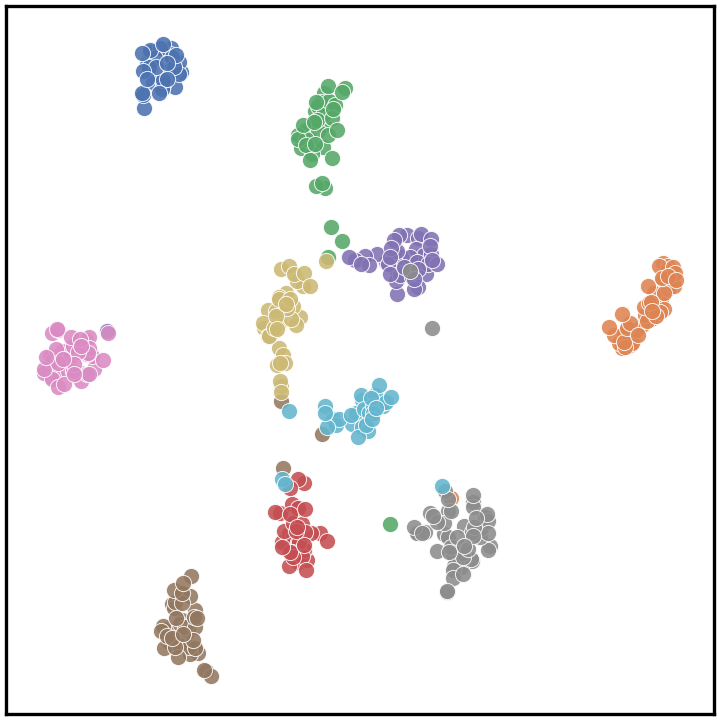}
    }
\subfigure[Step 13, err 2.8\%]{
    \includegraphics[width=0.18\linewidth]{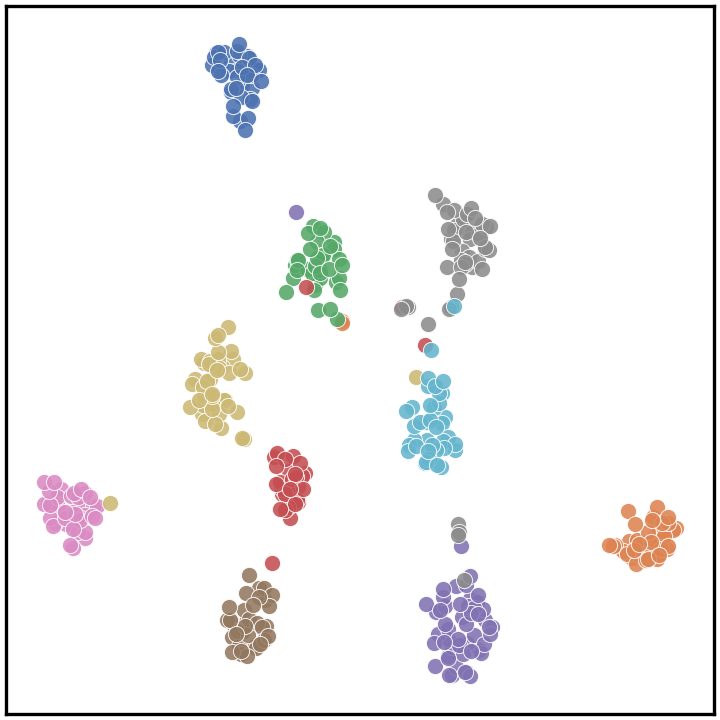}
    }
\subfigure[Step 14, err 4.5\%]{
    \includegraphics[width=0.18\linewidth]{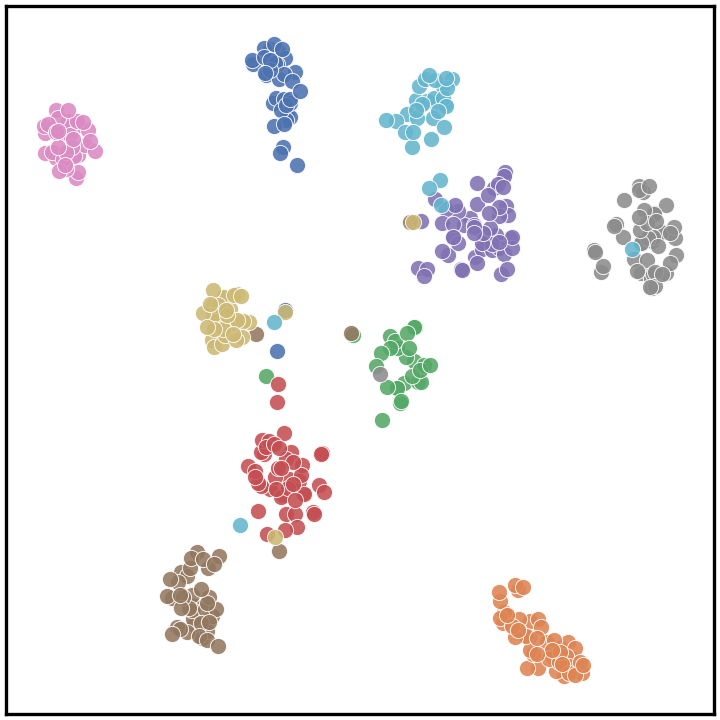}
    }
\subfigure[Step 15, err 4.5\%]{
    \includegraphics[width=0.18\linewidth]{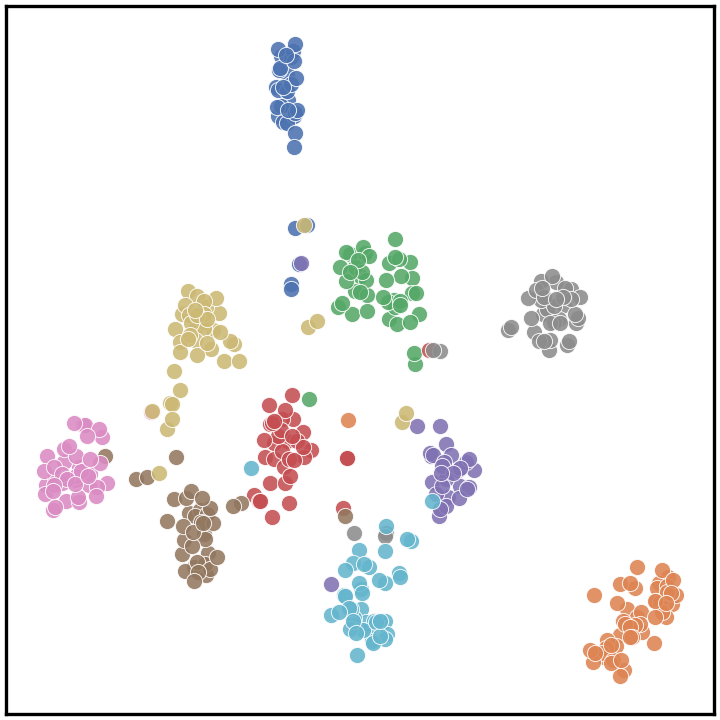}
    }
\caption{t-SNE visualization on all 15 future domains of
Rot-MNIST-C (colored by class). Each subfigure is one future domain with its step index
and classification error.}
\label{fig:tsne_mnist_all}
\end{figure*}

\begin{figure*}[t]
\centering
\subfigure[Step 1, err 6.1\%]{
    \includegraphics[width=0.2\linewidth]{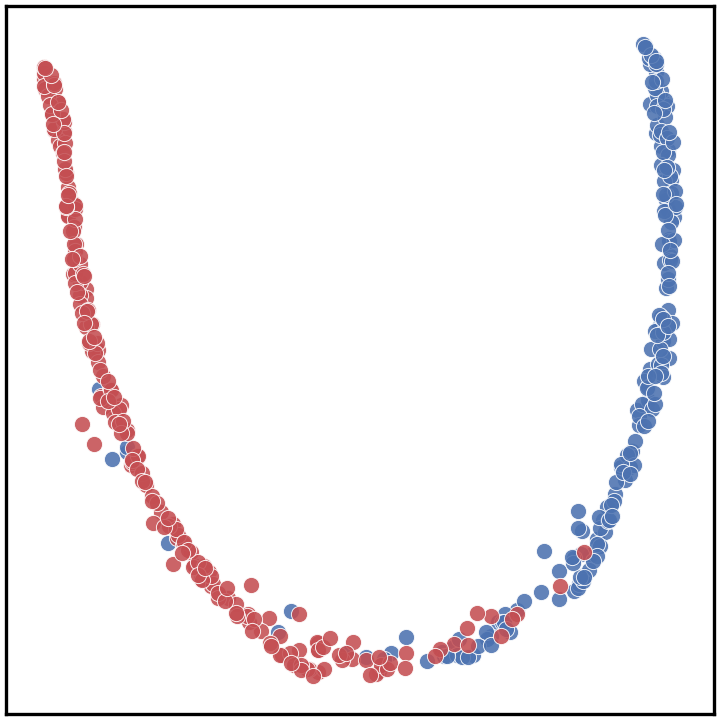}
    }
\subfigure[Step 2, err 5.4\%]{
    \includegraphics[width=0.2\linewidth]{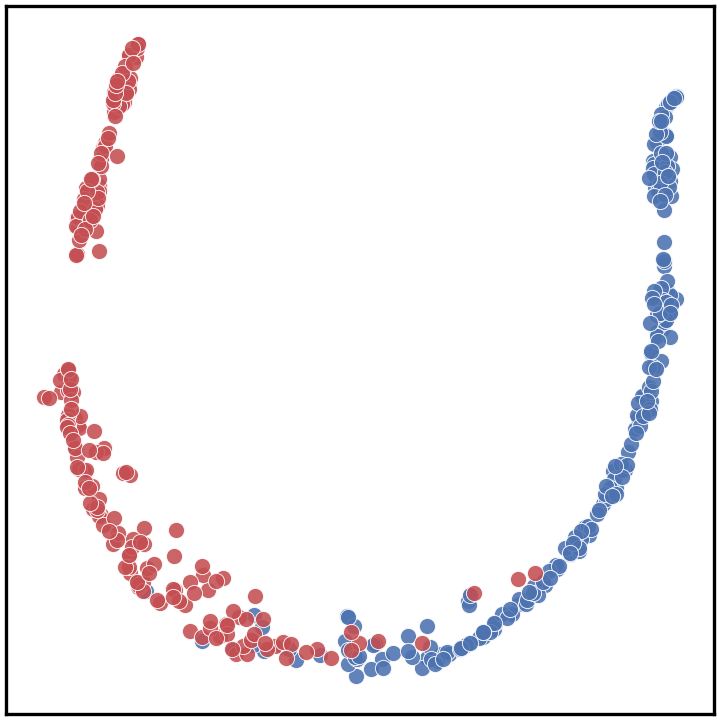}
    }
\subfigure[Step 3, err 4.9\%]{
    \includegraphics[width=0.2\linewidth]{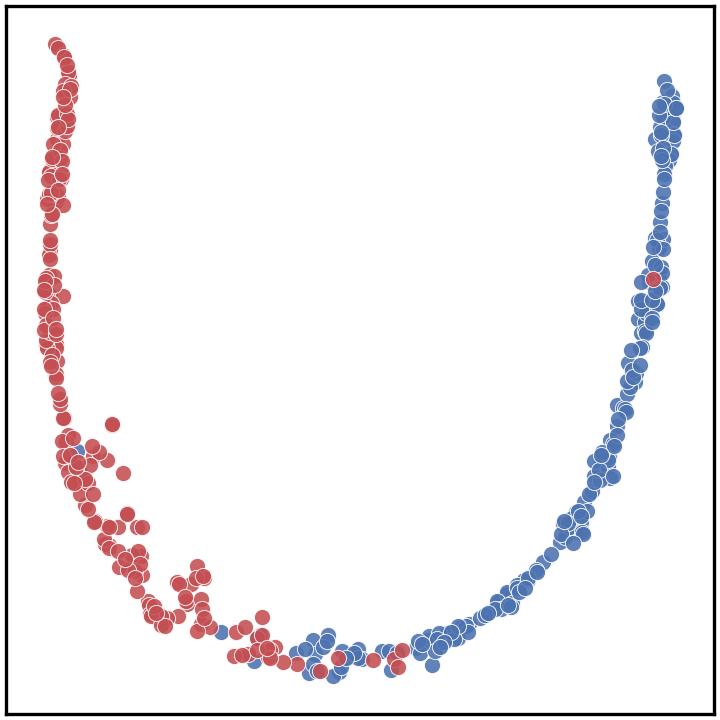}
    }
\subfigure[Step 4, err 4.7\%]{
    \includegraphics[width=0.2\linewidth]{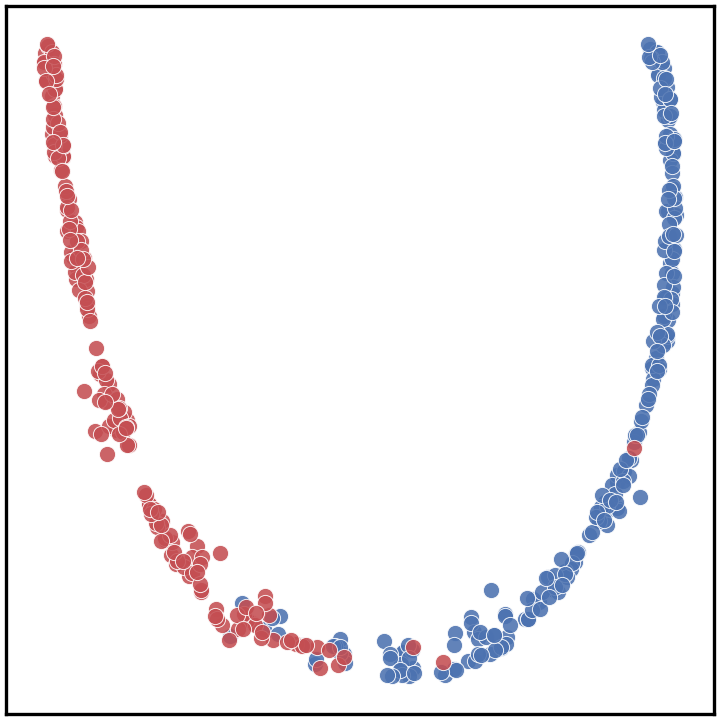}
    }
\\
\subfigure[Step 5, err 2.6\%]{
    \includegraphics[width=0.2\linewidth]{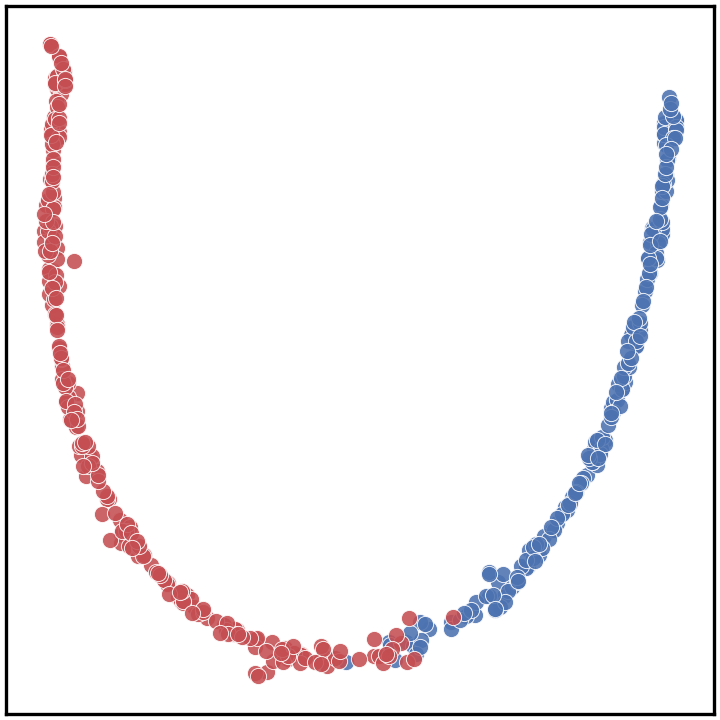}
    }
\subfigure[Step 6, err 2.7\%]{
    \includegraphics[width=0.2\linewidth]{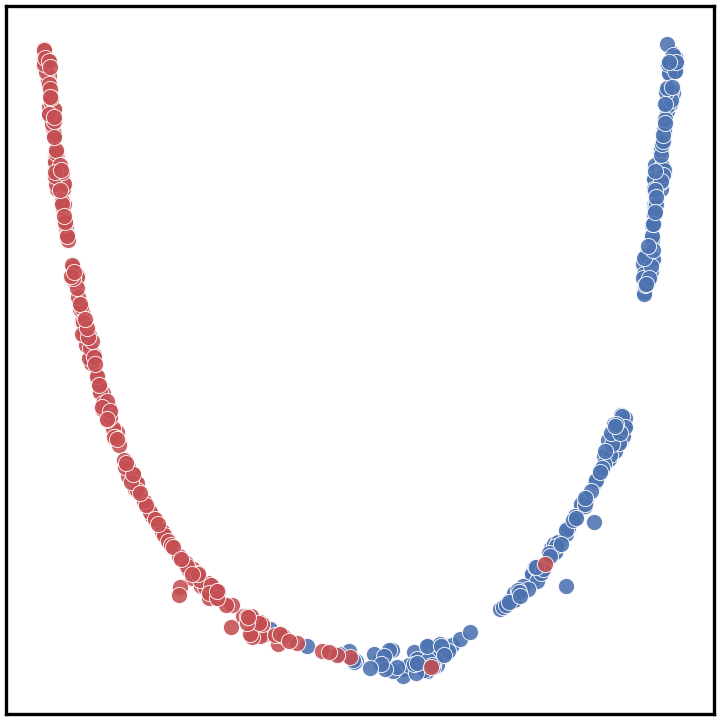}
    }
\subfigure[Step 7, err 2.8\%]{
    \includegraphics[width=0.2\linewidth]{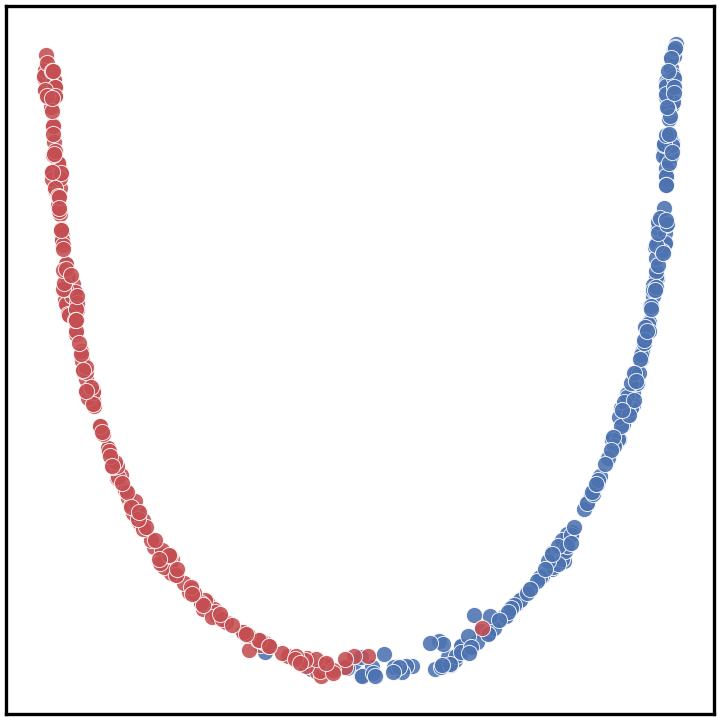}
    }
\subfigure[Step 8, err 5.9\%]{
    \includegraphics[width=0.2\linewidth]{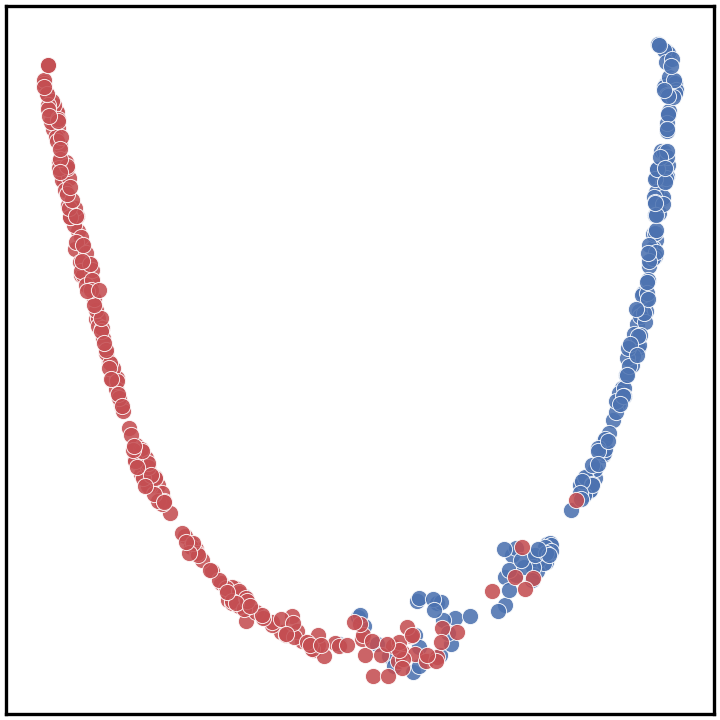}
    }
\\
\subfigure[Step 9, err 6.3\%]{
    \includegraphics[width=0.2\linewidth]{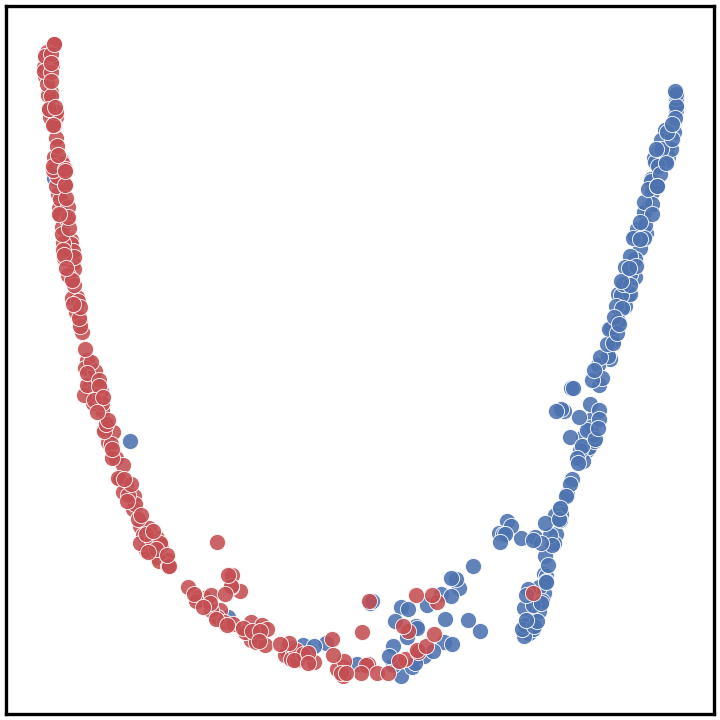}
    }
\subfigure[Step 10, err 5.6\%]{
    \includegraphics[width=0.2\linewidth]{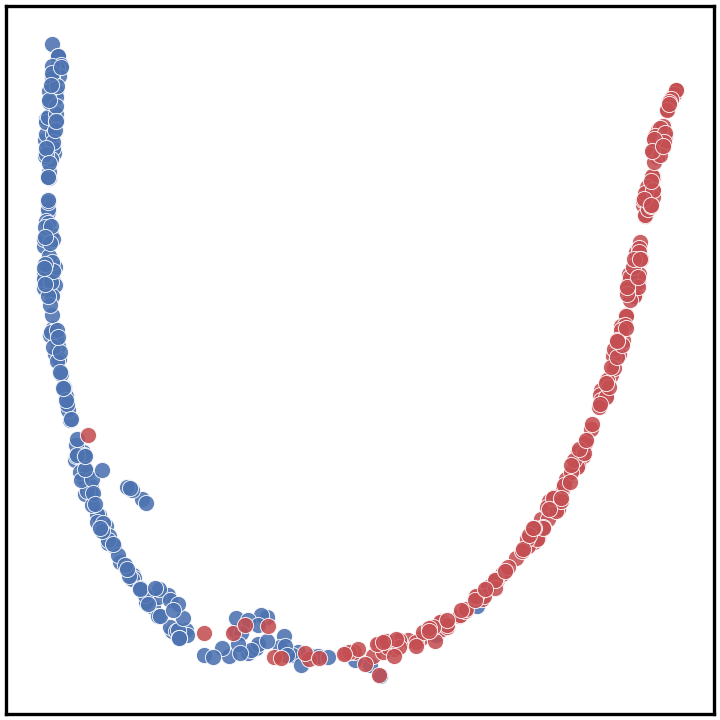}
    }
\subfigure[Step 11, err 6.7\%]{
    \includegraphics[width=0.2\linewidth]{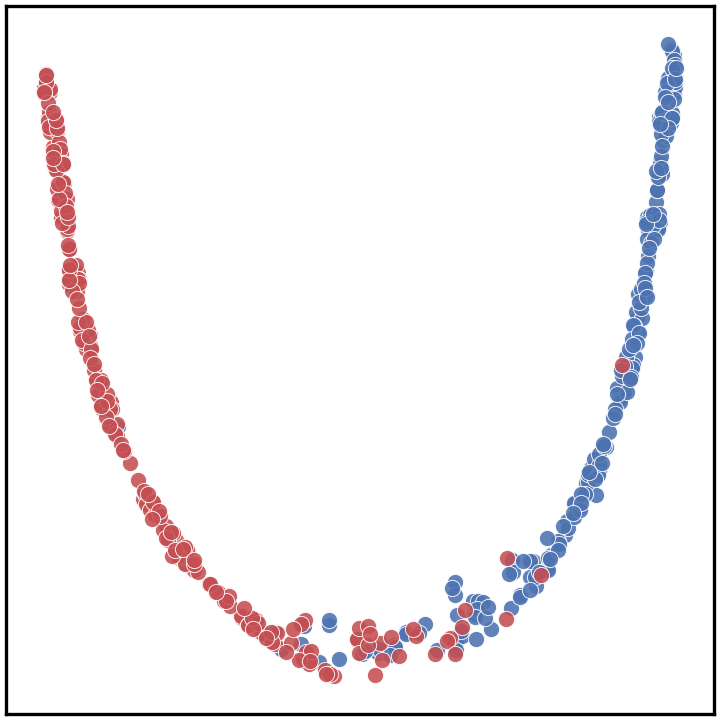}
    }
\subfigure[Step 12, err 5.2\%]{
    \includegraphics[width=0.2\linewidth]{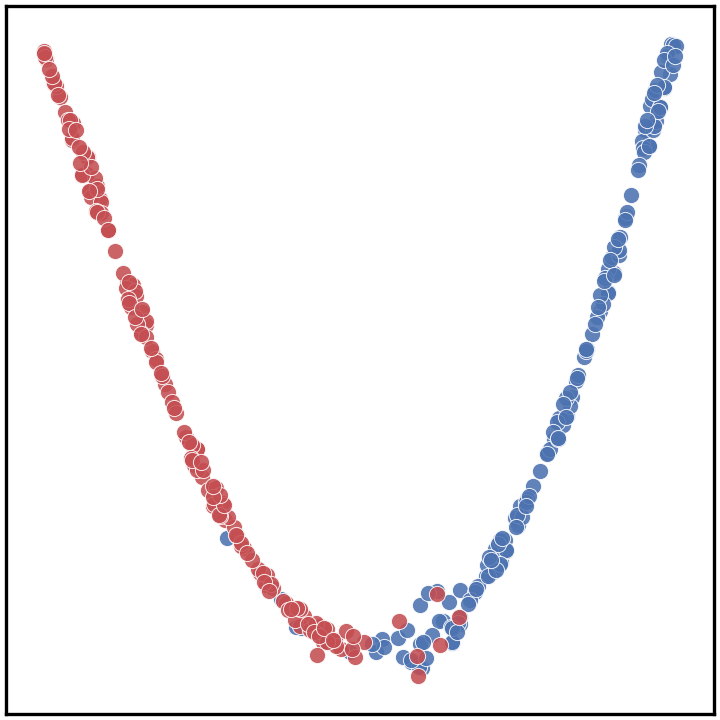}
    }
\caption{t-SNE visualization on all 12 future domains of
Yearbook (colored by class). Each subfigure is one future domain with its step index
and classification error.}
\label{fig:tsne_yearbook_all}
\end{figure*}

\subsection{Qualitative Analysis of Extrapolated Test-Domain Representations}
\label{app:tsne_analysis}

To inspect the geometric evolution of feature representations synthesized via continuous parameter extrapolation, we visualize the domain-adapted penultimate features across all unseen future target domains ($t_s > t_T$) using t-SNE. The resulting manifolds are colored by semantic class and annotated with per-domain error rates in Fig.~\ref{fig:tsne_mnist_all} (Rot-MNIST-C) and Fig.~\ref{fig:tsne_yearbook_all} (Yearbook). Two key geometric properties validate the core design of FreKoo++:

\paragraph{Sustained long-horizon class separability} 
Feature discriminability remains well-preserved across the entire extrapolation window. Across all $15$ future steps of Rot-MNIST-C and $12$ future steps of Yearbook, the feature representations consistently form tight, well-isolated clusters corresponding to the respective semantic categories. This confirms that the target model configurations $\hat{\theta}(t_s)$ generated via closed-form Koopman modal superposition synthesize valid, discriminative decision boundaries on the underlying model manifold, demonstrating that FreKoo++ captures true continuous distributional drift rather than overfitting to source-domain boundary states.

\paragraph{Horizon-invariant stability and error boundedness} 
The feature geometry exhibits stability over extended temporal horizons. The per-domain classification errors remain consistently low and tightly bounded ($1.4\%$--$4.8\%$ on Rot-MNIST-C, $2.6\%$--$6.7\%$ on Yearbook), with no evidence of representation collapse or dispersion at distant time steps. This empirical trend provides direct qualitative evidence for Theorem~\ref{thm:generalization_bound} and Corollary~\ref{cor:bounded_risk}: by analytically extrapolating non-growing dominant spectral dynamics while locking volatile transient components at $t_T$, FreKoo++ effectively eliminates the compounding numerical integration errors and boundary distortions that typically plague discrete or ODE-based autoregressive baselines. 

Together, these visualizations demonstrate that FreKoo++ preserves a temporally coherent yet class-discriminative feature topology across arbitrary prediction horizons, providing qualitative corroboration for the quantitative superiority.

\subsection{Hyperparameter Sensitivity Analysis}
\label{sec:sensitivity}

To evaluate the stability of FreKoo++ with respect to objective loss weights, we analyze performance sensitivity over a wide grid search range $\{0.01, 0.1, 1.0, 10, 100\}$ for $\alpha$ ($\mathcal{L}_{\mathrm{rec}}^{(C)}$), $\beta$ ($\mathcal{L}_{\mathrm{fit}}^{(C)}$), $\gamma$ ($\mathcal{R}_{\mathrm{spec}}$), and $\delta$ ($\mathcal{R}_{\mathrm{stab}}$) across six CTDG benchmark datasets. As illustrated in Fig.~\ref{fig:parameter_sensetivity}, our framework demonstrates remarkable robustness across several orders of magnitude.
Across most real-world datasets (e.g., Cyclone, House, Twitter, Yearbook, and 2-Moons-C), the error curves remain remarkably flat across the entire hyperparameter spectrum, indicating that FreKoo++ does not require delicate parameter tuning to achieve superior generalization. Optimal performance is generally attained when loss weights are selected within the moderate range $[0.1, 10]$. 
Specifically, setting $\beta$ ($\mathcal{L}_{\mathrm{fit}}^{(C)}$) or $\alpha$ ($\mathcal{L}_{\mathrm{rec}}^{(C)}$) excessively large ($100$) leads to an error increase on Rot-MNIST-C, as over-constraining parameter trajectory alignment forces the model to fit transient noise. Conversely, moderate values of stability ($\delta$) and spectral regularization ($\gamma$) effectively constrain positive modal growth rates and filter transient noise without sacrificing model capacity. Overall, these results confirm that FreKoo++ is numerically stable and consistently effective under standard loss weight configurations.

Fig.~\ref{fig:mode_sensitivity} illustrates the impact of modal dimension $K \in \{2, \dots, 64\}$ across all six CTDG datasets. When $K \le 4$, insufficient spectral capacity causes severe underfitting and high variance (e.g., on 2-Moons-C and Rot-MNIST-C). Expanding $K$ to the range of $16$--$32$ consistently yields the best generalization, with $K=32$ achieving peak AUC on Twitter and lowest MAE on House-C and 2-Moons-C. Beyond this point ($K=64$), performance on noisy real-world datasets (e.g., House-C and Twitter) slightly deteriorates due to the over-parameterization of redundant modes. Thus, $K=32$ serves as the global sweet spot balancing spectral expressiveness and noise suppression. For architectural regularity and implementation simplicity, we standardize $K = m = 32$ across all experiments.
\begin{figure*}[h]
\centering
\subfigure[$\alpha$]{
    \includegraphics[width=0.23\linewidth]{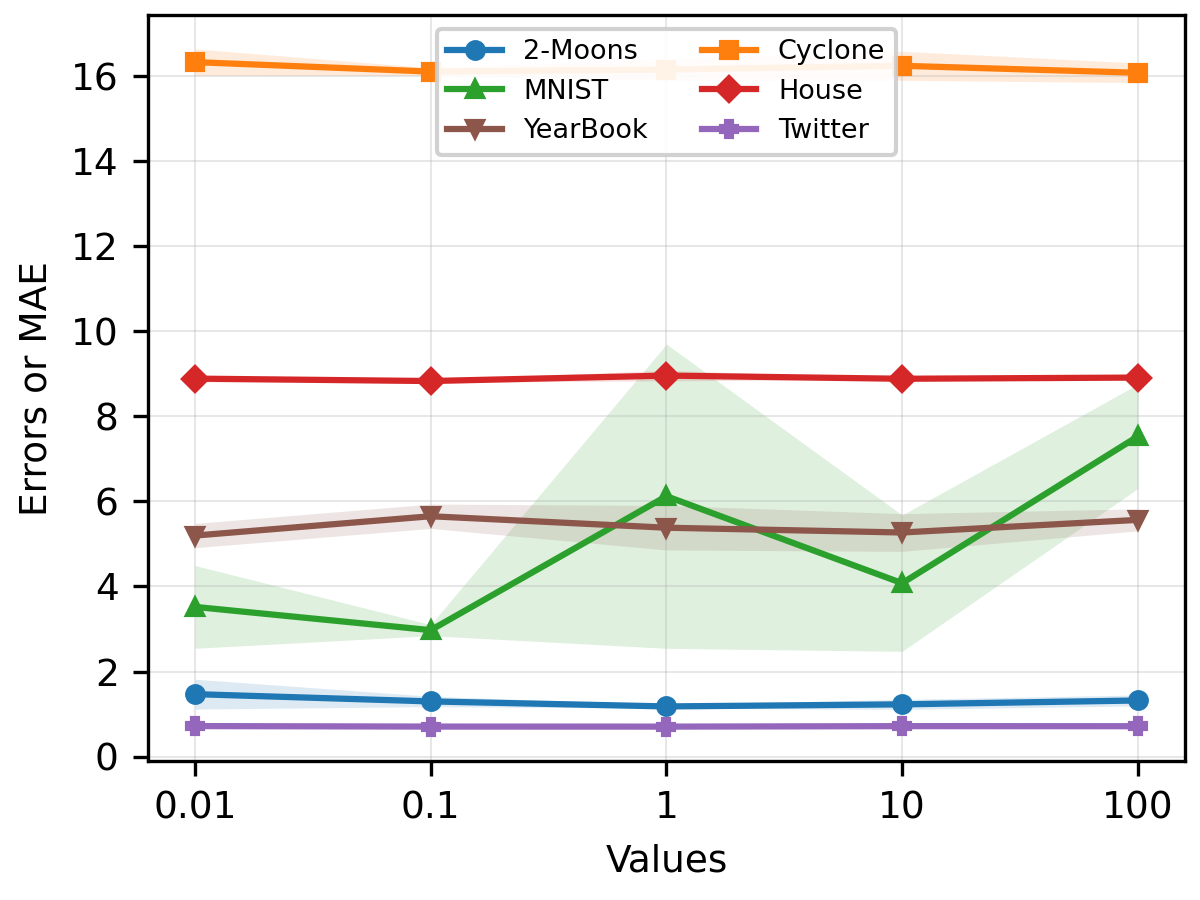}
    }
\subfigure[$\beta$]{
    \includegraphics[width=0.23\linewidth]{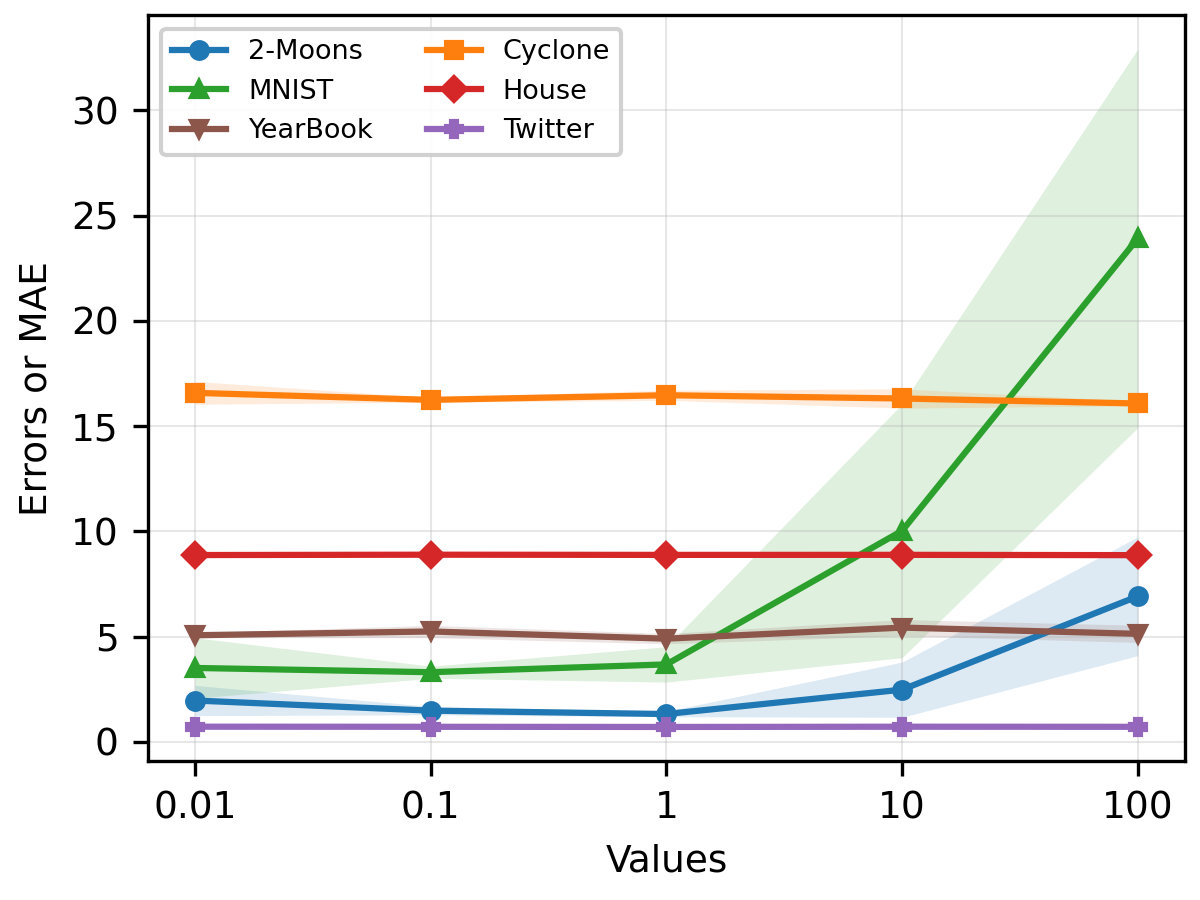}
    }
\subfigure[$\gamma$]{
       \includegraphics[width=0.23\linewidth]{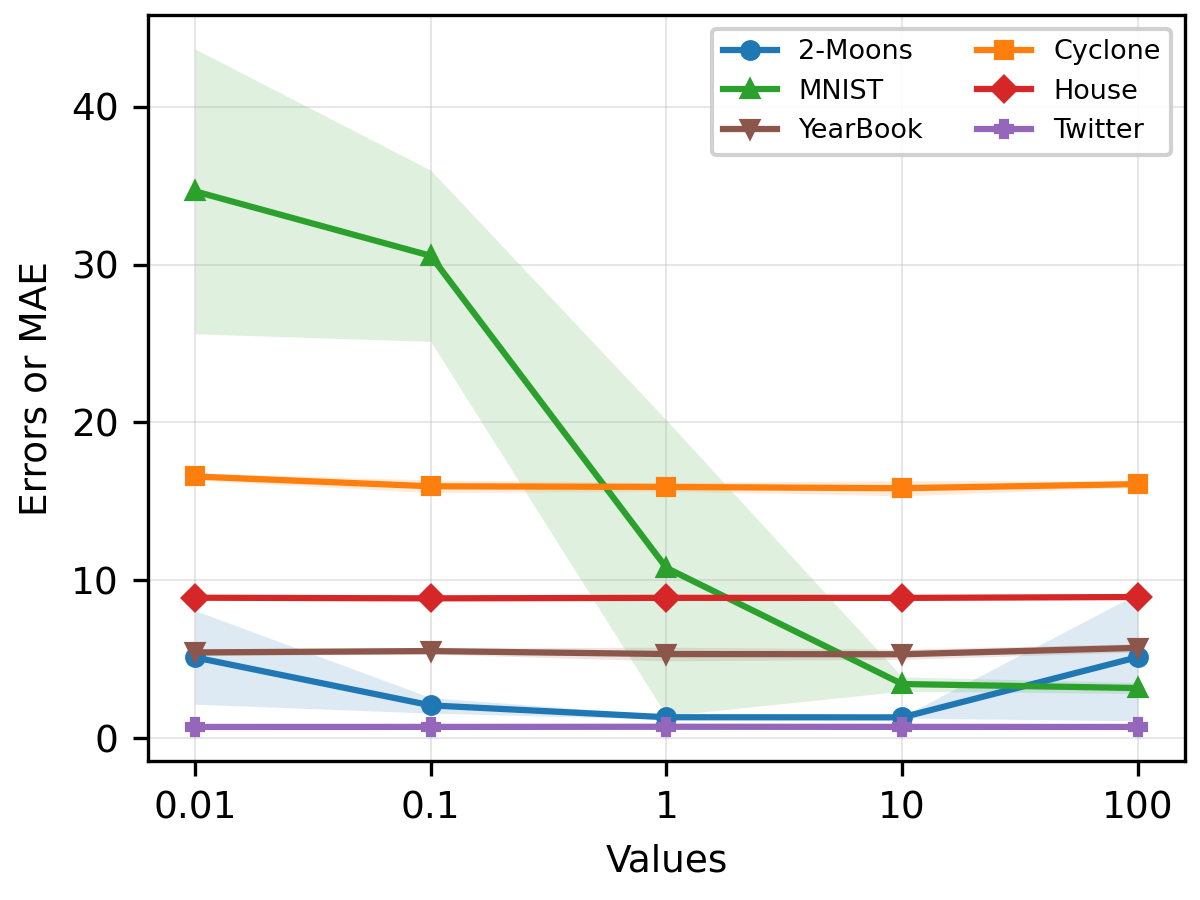} 
}
\subfigure[$\delta$]{
       \includegraphics[width=0.23\linewidth]{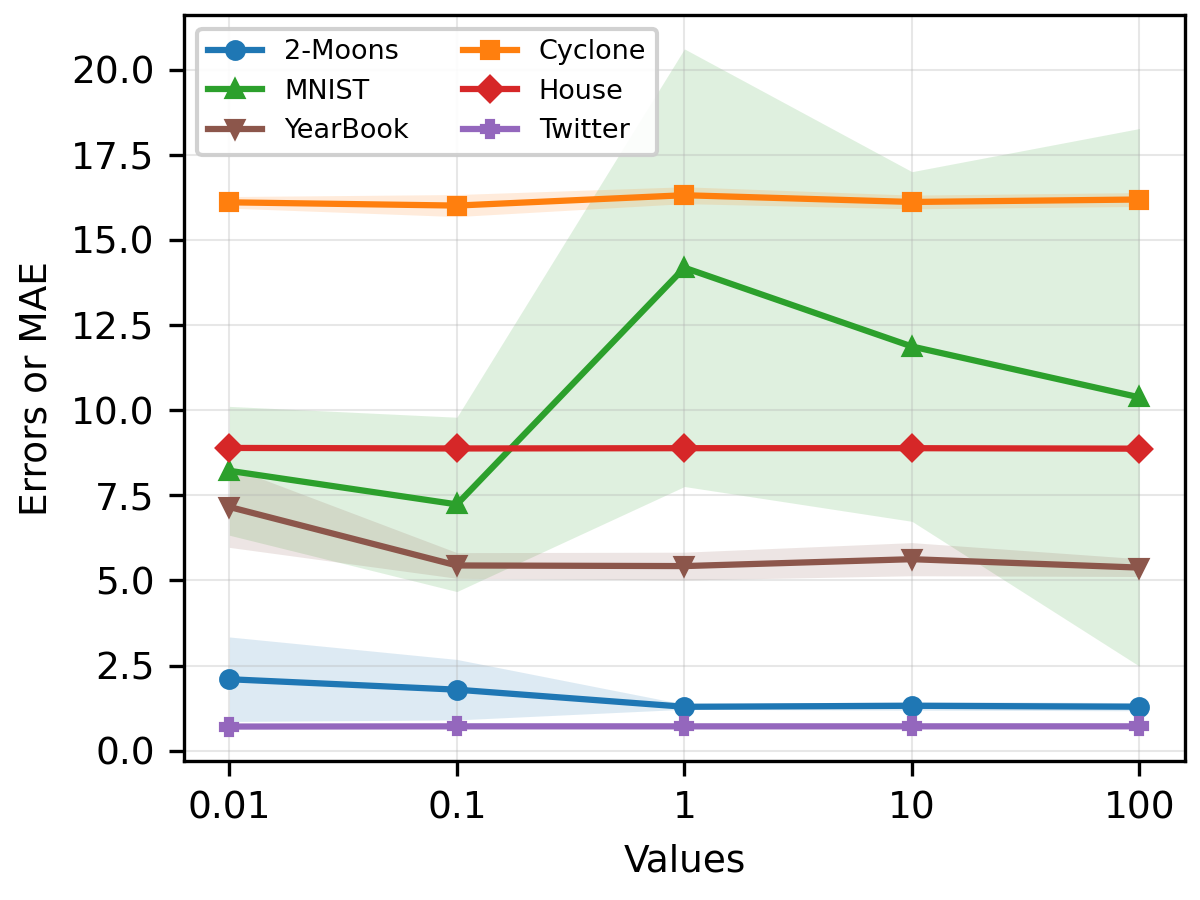} 
}
\caption{Hyperparameter sensitivity analysis of FreKoo++ with respect to objective loss weights $\alpha$ ($\mathcal{L}_{\mathrm{rec}}^{(C)}$), $\beta$ ($\mathcal{L}_{\mathrm{fit}}^{(C)}$), $\gamma$ ($\mathcal{R}_{\mathrm{spec}}$), and $\delta$ ($\mathcal{R}_{\mathrm{stab}}$) swept over $\{0.01, 0.1, 1.0, 10, 100\}$ across six benchmark datasets.}
\label{fig:parameter_sensetivity}
\end{figure*}

\begin{figure*}[h]
\centering
\subfigure[2-Moons-C]{
    \includegraphics[width=0.3\linewidth]{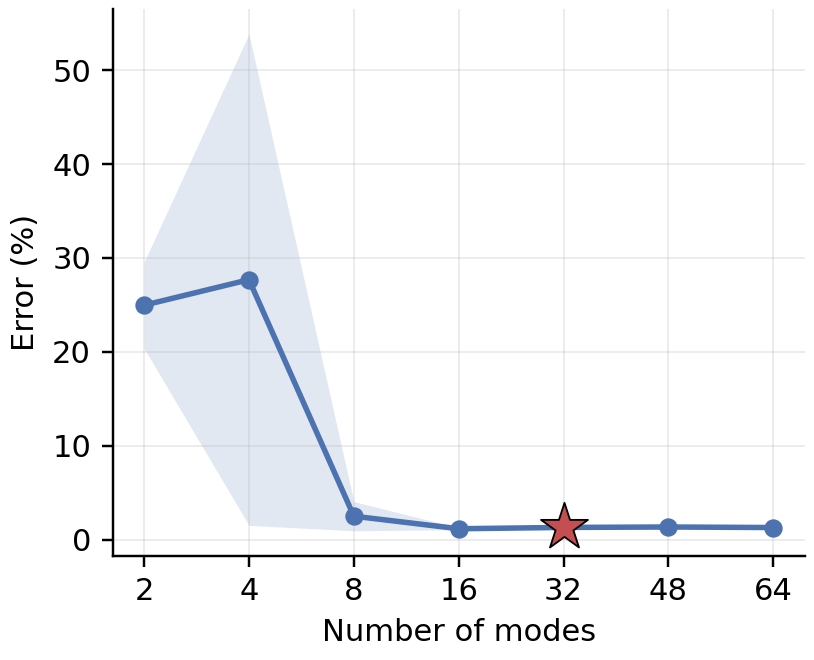}
    }
\subfigure[Rot-MNIST-C]{
    \includegraphics[width=0.3\linewidth]{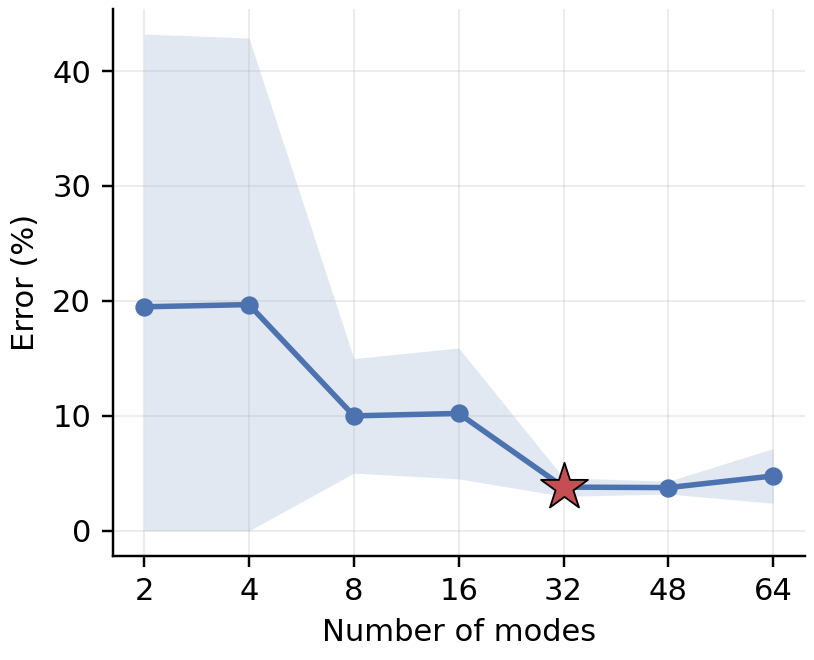}
    }
\subfigure[Twitter]{
       \includegraphics[width=0.3\linewidth]{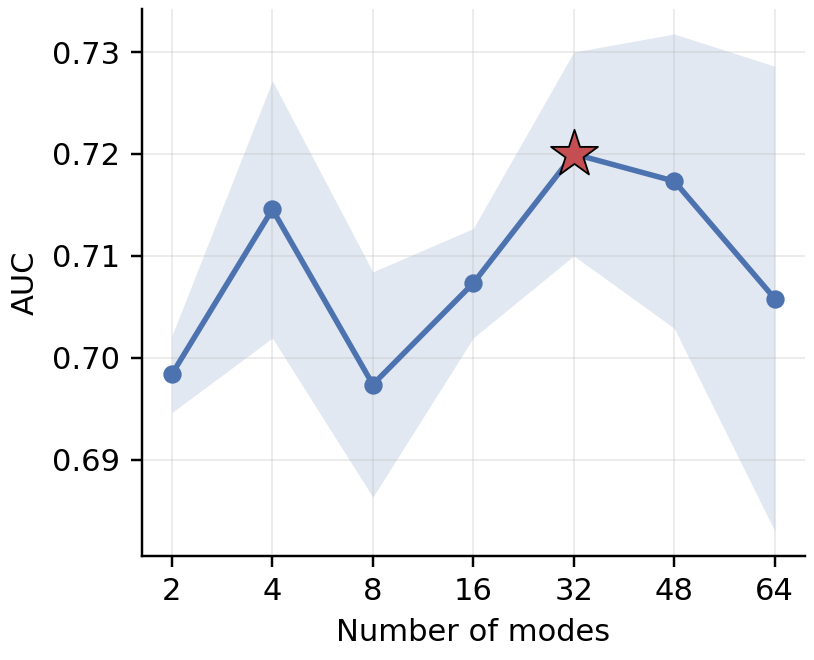} 
}
\subfigure[Yearbook]{
       \includegraphics[width=0.3\linewidth]{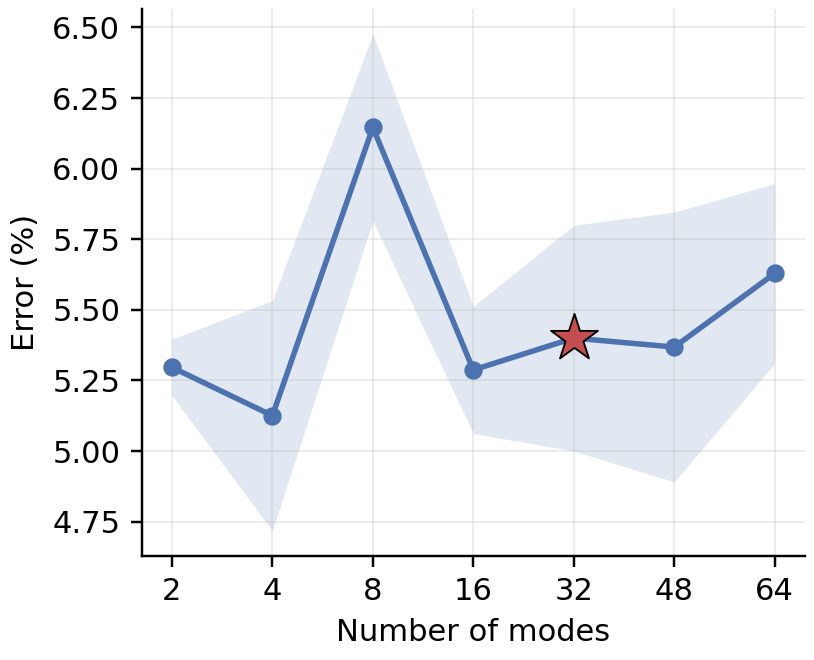} 
}
\subfigure[Cyclone]{
       \includegraphics[width=0.3\linewidth]{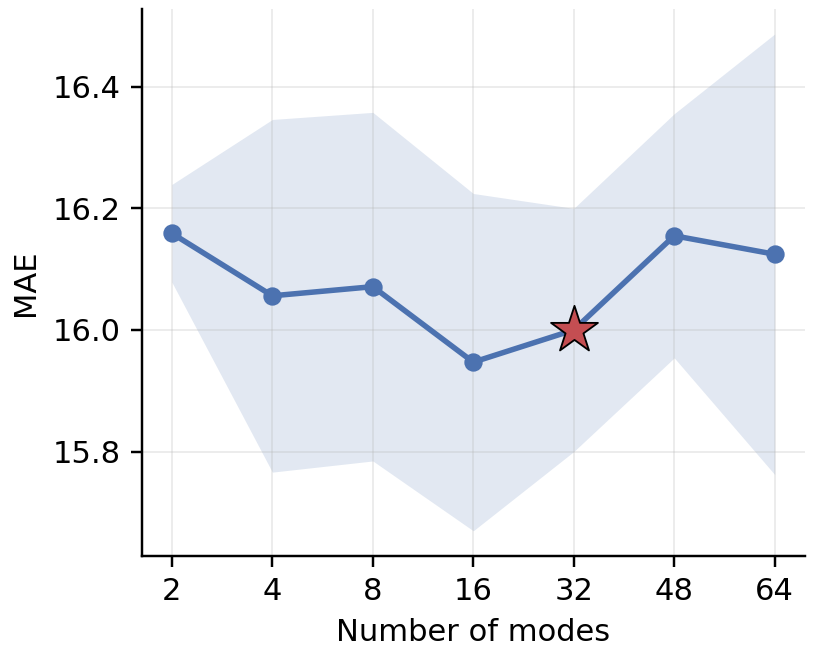} 
}
\subfigure[House-C]{
       \includegraphics[width=0.3\linewidth]{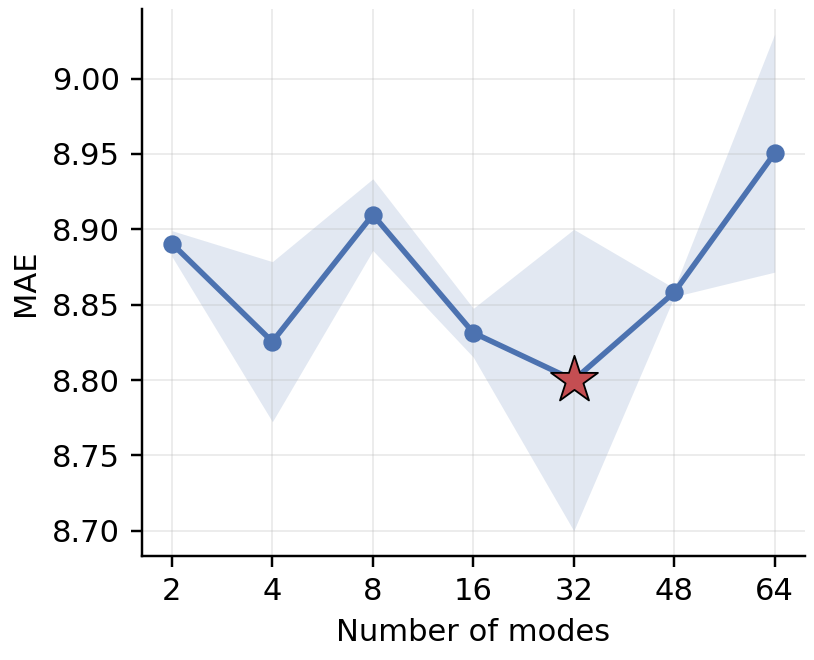} 
}
\caption{Sensitivity of the number of Koopman modes $K$. The red star marks the adopted $K=32$.}
\label{fig:mode_sensitivity}
\end{figure*}

\end{document}